%% file: example.tex
\documentclass{article}

\usepackage{amsmath,amsthm}
\usepackage{amsfonts}
\usepackage{wrapfig}
\usepackage{caption}
\usepackage{subcaption}
\usepackage{fleqn, tabularx}
\usepackage{multirow}
\usepackage[export]{adjustbox}
\usepackage{colortbl}
\usepackage{multirow,tabularx,booktabs}
\usepackage[skins,theorems]{tcolorbox}
\usepackage[english]{babel}
\usepackage{xcolor}

\usepackage[final]{corl_2021} 

\newtheorem{theorem}{Theorem}[section]
\newtheorem{protocol}{Guideline}[section]
\newtheorem{guideline}{Metric}[section]

\title{A Workflow for Offline Model-Free Robotic Reinforcement Learning}

\author{
  Aviral Kumar$^{\star, 1}$, Anikait Singh$^{\star, 1}$, Stephen Tian$^{1}$, Chelsea Finn$^{2}$, Sergey Levine$^{1}$\\
  $^1$ UC Berkeley, $^2$ Stanford University ~~~~~~~ ($^*$ Equal Contribution)\\
  \texttt{aviralk@berkeley.edu, asap7772@berkeley.edu}
}

\include{defs}

\begin{document}
\maketitle


\begin{abstract}
    Offline reinforcement learning (RL) enables learning control policies by utilizing only prior experience, without any online interaction. This can allow robots to acquire generalizable skills from large and diverse datasets, without any costly or unsafe online data collection. Despite recent algorithmic advances in offline RL, applying these methods to real-world problems has proven challenging. Although offline RL methods can learn from prior data, there is no clear and well-understood process for making various design choices, from model architecture to algorithm hyperparameters, without actually evaluating the learned policies online. In this paper, our aim is to develop a practical workflow for using offline RL analogous to the relatively well-understood workflows for supervised learning problems. To this end, we devise a set of metrics and conditions that can be tracked over the course of offline training, and can inform the practitioner about how the algorithm and model architecture should be adjusted to improve final performance. Our workflow is derived from a conceptual understanding of the behavior of conservative offline RL algorithms and cross-validation in supervised learning. We demonstrate the efficacy of this workflow in producing effective policies without any online tuning, both in several simulated robotic learning scenarios and for three tasks on two distinct real robots, focusing on learning manipulation skills with raw image observations with sparse binary rewards. Explanatory video and additional results can be found at \href{https://sites.google.com/view/offline-rl-workflow}{sites.google.com/view/offline-rl-workflow}. 
\end{abstract}


\input{intro.tex}


\input{related.tex}


\input{background.tex}


\input{method.tex}

\input{sim_exps}

\input{real_experiments.tex}

\vspace{-5pt}
\section{Discussion}
\vspace{-7pt}
\label{sec:conclusion}
While offline RL algorithms have improved significantly over the past years, applying these methods to real-world robotic domains is still challenging due to little guidance on selecting policies, adding regularization, or modifying model capacity. In this paper, we devise a \emph{workflow} for conservative offline RL algorithms such as CQL, which consists of a set of metrics and conditions that can be tracked by a practitioner over the course of offline training as well as a set of recommendations to addresses the resulting overfitting and underfitting challenges. We use our workflow to tune CQL on a number of robotic manipulation problems with diverse robot types, multiple objects, and long horizons, all while learning without online interaction from image observations and sparse rewards. Both in simulation and the real world, we observe strong performance benefits.

While our proposed workflow is an initial step towards practical robotic offline RL and is based on our best conceptual understanding of certain offline RL algorithms, these guidelines are heuristic. The validation of our workflow is largely empirical in nature. However, we believe to some extent this is unavoidable, since a workflow is a set of guidelines and recommendations, rather than a rigid algorithm. Regardless of how theoretically justified it is, in the end, its value is determined by its ability to produce good results in practice. Our validation consists of a series of case studies, and although the results (as with any empirical finding) are specific to the used domains, we believe the breadth of tasks considered, which consist of two different real robots and multiple simulated tasks, indicates the broad applicability of our approach. Our workflow is also specific to conservative offline RL methods, and particularly CQL and BRAC. Extending our technique to devise similar workflows for other algorithms is an exciting direction for future work. Deriving theoretical guarantees regarding workflows of this type is also an important direction for future research. We hope that this work allows future development into theoretically justified workflows for offline RL algorithms.

\vspace{-5pt}
\section*{Acknowledgements}
\vspace{-5pt}
We thank Ilya Kostrikov, Avi Singh, Ashvin Nair, Alexander Khazatsky, Albert Yu, Jedrzej Orbik, and Jonathan Yang for their help with setting up and debugging various aspects of the experimental setup as well as for providing us with offline datasets we could test our workflow on. We thank Dibya Ghosh, anonymous reviewers, and the area chair from CoRL for constructive feedback on an earlier version of this paper. AK thanks George Tucker and Rishabh Agarwal for valuable discussions. This research was funded by the DARPA Assued Autonomy Program and compute support from Google and Microsoft Azure.

\bibliography{example} 

\newpage
\appendix
\input{appendix}

\end{document}

%% file: defs.tex
\newcommand{\figleft}{{\em (Left)}}

\newcommand{\figright}{{\em (Right)}}
\newcommand{\figtop}{{\em (Top)}}
\newcommand{\figbottom}{{\em (Bottom)}}

\newcommand{\x}{\mathbf{x}}

\newcommand{\data}{\mathcal{D}}

\newcommand{\policy}{\pi}

\newcommand{\behavior}{{\pi_\beta}}
\newcommand{\bellman}{\mathcal{B}}

\newcommand{\bs}{\mathbf{s}}
\newcommand{\ba}{\mathbf{a}}

\newcommand{\E}{\mathbb{E}}

%% file: intro.tex
\begin{wrapfigure}{r}{0.63\textwidth}
\vspace{-1.25cm}
\begin{center}
\includegraphics[width=0.99\linewidth]{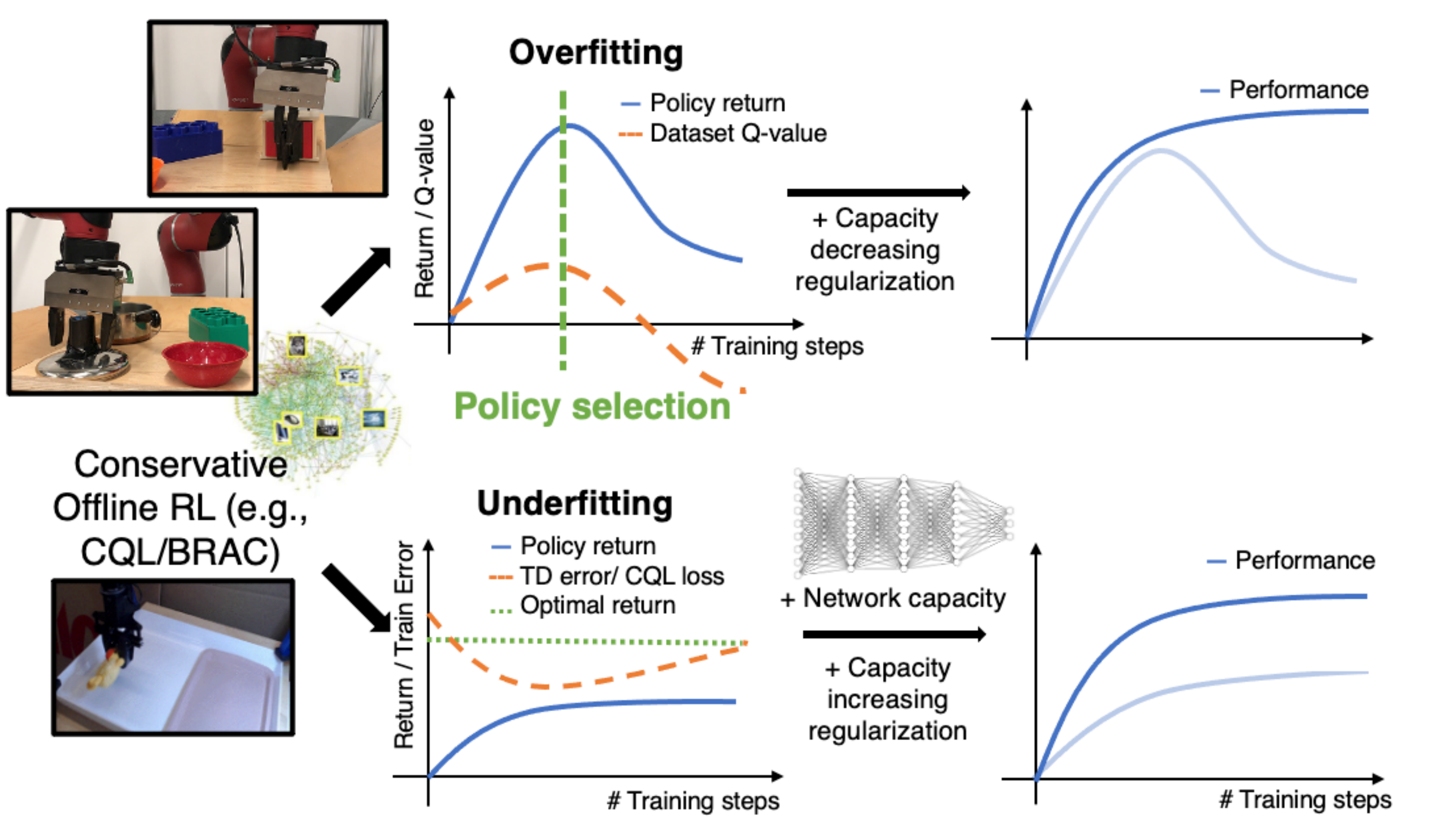}
\vspace{-0.6cm}
\end{center}
\caption{\small{\label{fig:teaser}\textbf{Our proposed workflow}  aims to detect overfitting and underfitting, and provides guidelines for addressing these issues via policy selection, regularization, and architecture design. We evaluate this workflow on \textbf{two} real-world robotic systems and simulation domains, and we find it to be effective.}}
\vspace{-0.45cm}
\end{wrapfigure}

\vspace{-5pt}
\section{Introduction}
\vspace{-8pt}
Offline reinforcement learning (RL) can in principle make it possible to convert existing large datasets of robotic experience into effective policies, without the need for costly or dangerous online interaction for each training run. While offline RL algorithms have improved significantly~\citep{kalashnikov2021mt,kumar2020conservative,singh2020cog,chebotar2021actionable,kalashnikov2018scalable}, applying such methods to real-world robotic control problems presents a number of major challenges. In standard online RL, any intermediate policy found during training is executed in the environment to collect more experience, which naturally allows for an evaluation of the policy performance. This ability to evaluate intermediate policies lets practitioners use ``brute-force'' to evaluate the effects of various design factors, such as model capacity and expressivity, the number of training steps, and so forth, and facilitates comparatively straightforward tuning. In contrast, offline RL methods do not have access to real-world on-policy rollouts for evaluating the learned policy. Thus, in order for these methods to be truly practical for real-world applications, we not only require effective algorithms, but also an effective \emph{workflow}: a set of protocols and metrics that can be used to reliably and consistently adjust model capacity, regularization, etc in offline RL to obtain policies with good performance, without requiring real-world rollouts for tuning.   

A number of prior works have studied \emph{model selection} in offline RL by utilizing off-policy evaluation (OPE) methods~\cite{precup2000eligibility} to estimate policy performance. These methods can be based either on model or value learning~\citep{kostrikov2020statistical,paduraru2012off,paine2020hyperparameter,nachum2020reinforcement}
or importance sampling~\cite{precup2000eligibility,thomas2015high,thomas2015higheval,jiang2015doubly}. However, developing reliable OPE methods is itself an open problem, and modern OPE methods themselves suffer from hyperparameter selection challenges (see \citet{fu2021benchmarks} for an empirical study). Moreover, accurate off-policy evaluation is likely not necessary to simply tune algorithms for best performance -- we do not need a precise estimate of \emph{how} good our policy is, but rather a workflow that enables us to best improve it by adjusting various algorithm hyperparameters.

In this paper, we devise a practical workflow for selecting regularizers, model architectures, and policy checkpoints for offline RL methods in robotic learning settings. We focus on a specific class of conservative offline RL algorithms~\citep{levine2020offline,kumar2020conservative} that regularize the Q-function, but also show that our workflow can be effectively applied to policy constraint methods~\citep{wu2019behavior}. Our aim is not to focus on complete off-policy evaluation or to devise a new approach for off-policy evaluation, but rather to adopt a strategy similar to the one in supervised learning. Analogously to how supervised learning practitioners can detect overfitting and underfitting by tracking training and validation losses, and then adjust hyperparameters based on these metrics, our workflow (see Figure~\ref{fig:teaser} for a schematic) first defines and characterizes overfitting and underfitting, proposes metrics and conditions that users can track to determine if an offline RL exhibits overfitting or underfitting, and then utilizes these metrics to inform design decisions pertaining to neural net architectures, regularization, and early stopping. This protocol is intended to act as a ``user's manual'' for a practitioner, with guidelines for how to modify algorithm parameters for best results without real-world evaluation rollouts.    

The primary contribution of this paper is a simple yet effective workflow for robotic offline RL. We propose metrics and protocols to assist practitioners in selecting policy checkpoints, regularization parameters, and model architectures for conservative offline RL algorithms such as CQL~\citep{kumar2020conservative} and BRAC~\citep{wu2019behavior}. We empirically verify the efficacy of our proposed workflow on simulated robotic manipulation problems as well as three real-world robotic manipulation problems on two different robots, with diverse objects, pixel observations, and sparse binary reward supervision. 
Experimentally, we evaluate our method on two real-world robots (the Sawyer and WidowX robots), and one realistic simulated tasks. Our approach is effective in all of these cases, and on two tasks with the Sawyer robot that initially fail completely, our workflow improves the success rate to \textbf{70\%}.

%% file: related.tex
\vspace{-5pt}
\section{Related Work}
\label{sec:related}
\vspace{-5pt}

\textbf{Robotic RL with offline datasets.} Learning-based methods have been applied to a number of robotics problems, such as grasping objects~\cite{kalashnikov2018qtopt,zeng2018learning}, in-hand object manipulation~\cite{dexterity, VanHoofInHandManipulation, dexterous1, dexterous2}, pouring fluids~\cite{schenck2017visual}, door opening~\cite{yahya2017collective}, and manipulating cloth~\cite{matas18}. While the majority of these works use standard online RL, a number of prior works have used also leveraged robotic datasets to train skills in addition to active environment rollouts. \citet{kalashnikov2018qtopt},  \citet{julian2020efficient}, and \citet{cabi2019framework} use offline pre-training followed by a finetuning phase to improve the policy. Visual foresight~\citep{finn2017deep,ebert2018visual,xie2019improvisation,hristov2018interpretable,tian2020model} train a video-predictive dynamics model for offline planning. \citet{young2020visual,johns2021coarse} lean skills in an offline manner and use it for imitation. \citet{mandlekar2020iris,m2020learning} learn hierarchical skills and combine them via imitation learning. Our work is complementary to these prior works, in that our workflow can be applied to any robotic offline RL system.

\textbf{Offline deep RL algorithms.} Algorithms for offline deep RL~\citep{lange2012batch,levine2020offline} can be divided into three categories: those that constrain the policy to the dataset~\citep{fujimoto2018off,kumar2019stabilizing,wu2019behavior,peng2019awr,jaques2019way,nair2020accelerating}, those that prevent overestimation via critic regularization~\citep{kumar2020conservative,kostrikov2021offline,fakoor2021continuous} and those that train dynamics models and apply a reward penalty~\citep{yu2020mopo,kidambi2020morel}. These algorithms have been applied in robotics, for example when learning from unlabeled data~\citep{singh2020cog}, robotic manipulation~\citep{Rafailov2020LOMPO,kalashnikov2018qtopt}, goal-conditioned RL~\citep{chebotar2021actionable,khazatsky2021can} and multi-task RL problems~\citep{kalashnikov2021mt}.
Rather than developing a new offline RL method, our work develops criteria and workflow rules that simplify the application of these methods to new robotics tasks.

\textbf{Off-policy evaluation for model selection in offline RL.} To the best of our knowledge, prior work that attempts to tackle model-selection in offline RL has focused
exclusively on devising off-policy evaluation (OPE) methods. These methods utilize importance sampling~\citep{precup2001off,voloshin2019empirical,thomas2015higheval} or learn a dynamics model or a value function~\citep{jiang2015doubly,paine2020hyperparameter,kostrikov2020statistical,nachum2019dualdice} to estimate the policy return. However, empirical studies by \citet{fu2021benchmarks} and \citet{qin2021neorl} show that none of these OPE methods actually perform reliably and consistently across tasks and offline datasets of the kind we are likely to find in the real-world, and present tuning challenges of their own. Our workflow does not perform direct off-policy evaluation, and instead utilizes comparative metrics across checkpoints and training runs based on observations about the behavior of specific types of offline RL algorithms.

%% file: background.tex
\vspace{-5pt}
\section{Preliminaries, Background, and Definitions}
\label{sec:background}
\vspace{-5pt}
The goal in RL is to optimize the infinite horizon discounted return  $R = \sum_{t = 0}^{\infty}\gamma^t r(\bs_t, \ba_t)$, where $r(s,a)$ represents the reward function evaluated at a state-action pair $(\bs, \ba)$. We operate in the \emph{offline} RL setting and are provided with a fixed dataset $\mathcal{D} = \{(\bs, \ba, r(\bs, \ba), \bs')\}$, consisting of transition tuples obtained from rollouts under a behavior policy $\behavior(\ba|\bs)$. Our goal is to obtain the best possible policy by only training on this fixed offline dataset $\data$, with no access to online rollouts. We focus on \emph{conservative} offline RL algorithms that modify the Q-function to penalize distributional shift, with most experiments on CQL~\citep{kumar2020conservative}, though we also adapt our workflow to BRAC~\citep{wu2019behavior} in Appendix~\ref{app:which_algos}.

\textbf{Conservative Q-learning (CQL).} The actor-critic formulation of CQL trains a Q-function $Q_\theta(\bs, \ba)$ with a separate policy $\pi_\phi(\ba|\bs)$, which maximizes the expected Q-value $\E_{\bs \sim \mathcal{D}, \ba \sim \pi_\phi}[Q_\theta(\bs, \ba)]$ like other standard actor-critic deep RL methods~\citep{lillicrap2015continuous,fujimoto2018addressing,haarnoja2018soft}. However, in addition to the standard TD error $\mathcal{L}_\text{TD}(\theta)$ (in blue below), CQL applies a regularizer $\mathcal{R}(\theta)$ (in red below) to prevent overestimation of Q-values for out-of-distribution (OOD) actions. This term minimizes the Q-values under a distribution $\mu(\ba|\bs)$, which is automatically chosen to pick actions
$\ba$ with high Q-values $Q_\theta(\bs, \ba)$, and counterbalances this term by maximizing the values of the actions in the dataset: \vspace{-0.3cm}
\begin{equation}
\label{eqn:cql_training}
\hspace{-1cm} \small{\min_{\theta}~ \textcolor{red}{\alpha\left(\mathbb{E}_{\bs \sim \mathcal{D}, \ba \sim \mu(\cdot|\bs)} \left[Q_\theta(\bs,\ba)\right] - \mathbb{E}_{\bs, \ba \sim \mathcal{D}}\left[Q_\theta(\bs,\ba)\right]\right)}+\frac{1}{2} \textcolor{blue}{\mathbb{E}_{\bs, \ba, \bs' \sim \mathcal{D}}\left[\left(Q_\theta(\bs, \ba) - \bellman^\policy\bar{Q}(\bs, \ba)\right)^2 \right]}},
\vspace{-0.1cm}
\end{equation}
where $\bellman^\policy \bar{Q} (\bs, \ba)$ is the Bellman backup operator with a delayed target Q-function, $\bar{Q}$: $\bellman^\policy \bar{Q}(\bs, \ba) := r(\bs, \ba) + \gamma \E_{\ba' \sim \pi(\ba'|\bs')}[\bar{Q}(\bs', \ba')]$. In practice, CQL computes $\mu(\ba|\bs)$ using actions sampled from the policy $\pi_\phi(\ba|\bs)$. More discussion of CQL is in Appendix~\ref{app:additional_background}. In this paper, we will utilize CQL as a base algorithm that our workflow intends to tune, but we also extend it to BRAC.

\begin{wraptable}{r}{0.7\linewidth}
 \vspace{-0.4cm}
  \centering
  \scriptsize
  \def\arraystretch{0.9}
  \setlength{\tabcolsep}{0.42em}
\begin{tabularx}{0.95\linewidth}{c X X}
\toprule
\textbf{Quantity} & \textbf{Supervised Learning} & \textbf{Conservative Offline RL}\\
\midrule
Test error & Loss $\mathcal{L}$ evaluated on test data, $\mathcal{D}_{\text{test}}$ & Performance of policy, $J(\pi)$ \\
Train error & Loss $\mathcal{L}$ evaluated on train data, $\mathcal{D}_\text{train}$ & Objective in Equations~\ref{eqn:generic_offline_rl},~\ref{eqn:cql_training}\\
\midrule
Overfitting & $\mathcal{L}(\mathcal{D}_\text{train})$ low, $\mathcal{L}(\mathcal{D}_\text{val})$ high, $\mathcal{D}_\text{val}$ is a validation set drawn i.i.d. as $\mathcal{D}_\text{train}$ & Training objective in Equation~\ref{eqn:cql_training} is extremely low, low value of $J(\pi)$ \\
Underfitting & high value of train error $\mathcal{L}(\mathcal{D}_\text{train})$ & Training objective in  Equation~\ref{eqn:cql_training} is extremely high, low value of $J(\pi)$\\
\bottomrule
    \end{tabularx}
    \vspace{-0.1cm}
         \caption{\label{tab:summary} \footnotesize{Summary of train error, test error and our definitions of overfitting and underfitting in supervised learning and conservative offline RL methods. We will propose metrics to measure these phenomena in a purely offline manner and recommend how to tune the underlying method accordingly.  
     }}
\vspace{-0.4cm}
\end{wraptable}
\textbf{Overfitting and underfitting in CQL.} Conservative offline RL algorithms~\citep{kumar2020conservative,kostrikov2021offline} like CQL can be sensitive to design choices, including number of gradient steps for training~\citep{kumar2021implicit,anonymous2021value} and network capacity. These challenges are also present in supervised learning, but supervised learning methods benefit from a simple and powerful workflow that involves using training error and validation error to characterize overfitting and underfitting. A practitioner can then make tuning choices based on these characterizations. To derive an analogous workflow for offline RL, we first ask:\textbf{ what do overfitting and underfitting actually mean for the case of conservative offline RL?}

To define overfitting and underfitting generically for any conservative offline RL method, we consider an abstract optimization formulation for such methods~\citep{kumar2020conservative}: \vspace{-0.3cm}

\begin{equation}
\label{eqn:generic_offline_rl}
    \pi^* := \arg \max_{\pi}~~ J_{\mathcal{D}}(\pi) - \alpha D(\pi, \pi_\beta)~~~~~~~~~~~ \text{(Conservative offline RL)}.
\end{equation}
\vspace{-0.3cm}
$J_{\mathcal{D}}(\pi)$ denotes the average return of policy $\pi$ in the empirical MDP induced by the transitions in the offline dataset $\mathcal{D}$, and $D(\pi, \pi_\beta)$ denotes a closeness constraint to the behavior policy, effectively applied by the offline RL method. Our definition of conservative offline RL requires that this divergence be computed in expectation over the state visitation distribution of the learned policy $\pi$ in the empirical MDP as discussed in Appendix~\ref{app:which_algos}. For example, Equation~\ref{eqn:cql_training} translates to utilizing $D_\text{CQL}(p, q) := \sum_{\x} p(\x) (p(\x)/q(\x) - 1)$ in Equation~\ref{eqn:generic_offline_rl} (see Theorem 3.5 in \citet{kumar2020conservative} for a proof). The training loss is discussed in Equations~\ref{eqn:cql_training} and \ref{eqn:generic_offline_rl} and the test loss is equal to the negative of the actual return $J(\pi)$ of the learned policy. Analogously to supervised learning, we can use the notion of train and test error to define overfitting and underfitting in offline RL, as discussed in Table~\ref{tab:summary}. However, note that the conditions summarized in Table~\ref{tab:summary} are not measurable completely offline. Precisely estimating if a run of an offline RL method overfits or underfits requires evaluating the learned policy via interaction with the real-world environment. In Section~\ref{sec:workflow_metrics}, our goal will be to devise offline metrics for characterizing overfitting that do not have this requirement. We will tailor our study specifically towards CQL, though we extend it to BRAC in Appendix~\ref{app:which_algos}. A similar procedure could be devised for other offline RL methods, but we leave this for future work.

%% file: method.tex
\vspace{-8pt}
\section{Detecting Overfitting and Underfitting in Conservative Offline RL}
\label{sec:workflow_metrics}
\vspace{-8pt}

In standard supervised learning, we can determine if a method overfits or underfits by comparing the training loss to the same loss function evaluated on a held-out validation dataset, which serves as a ``proxy'' test dataset. In contrast, the return of the learned policy $J(\pi)$ in RL does not have a direct proxy that can be computed offline. Thus, our goal is to identify offline metrics and conditions that allow us to measure overfitting and underfitting in conservative offline RL, with a focus on CQL. {We also adapt these conditions to BRAC~\citep{wu2019behavior}, a policy-constraint method in Appendix~\ref{app:brac_example}.}

\begin{wrapfigure}{r}{0.3\textwidth}
\vspace{-0.3cm}
\includegraphics[width=\linewidth]{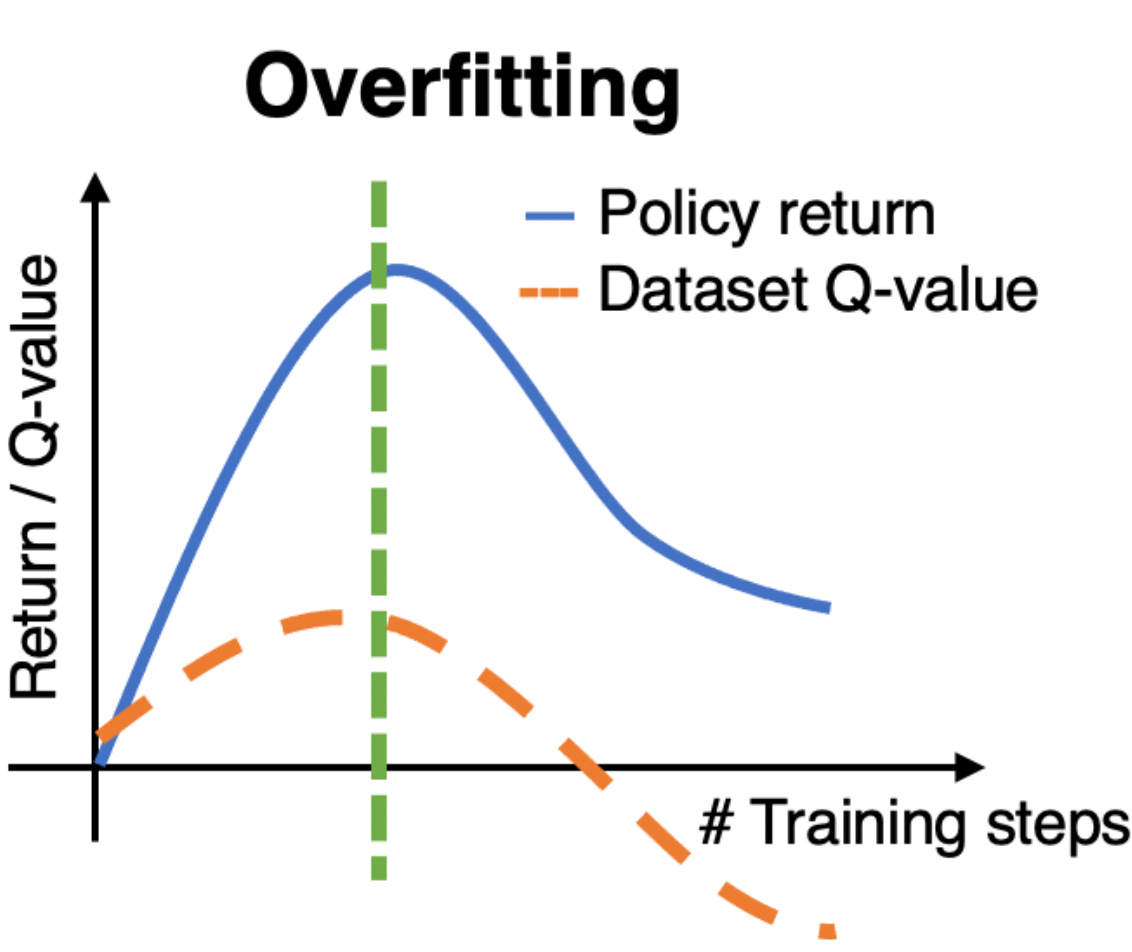}
\vspace{-0.9cm}
\end{wrapfigure}
\textbf{\underline{Detecting overfitting in CQL.}} Our definition of overfitting (Table~\ref{tab:summary}) corresponds to a low value for the training loss (Equation~\ref{eqn:cql_training}), but poor actual policy performance $J(\pi)$. To detect this, we analyze the time series of the estimated Q-values averaged over the dataset \emph{samples} $(\bs, \ba, r, \bs') \in \mathcal{D}$ over the course of training with a large number of gradient steps.
A run is labeled as overfitting if we
see that the expected dataset Q-value exhibits a non-monotonic trend: if the average Q-values first increase and then decrease as shown in the figure on the right. 
Additionally, we would see that training loss in Equation~\ref{eqn:cql_training} eventually becomes very low. Why do we see such a trend in the average dataset Q-value? Since CQL selectively penalizes the average Q-value under the distribution $\mu(\ba|\bs)$ supported on actions with large Q-values,
we would expect the Q-values on states from the dataset $\bs \sim \mathcal{D}$ and the learned $\ba \sim \pi(\cdot|\bs)$ to be small since the policy is trained to maximize the Q-function as well. This in turn would lead to an eventual reduction in the average Q-value on dataset actions, $\E_{\bs, \ba \sim \mathcal{D}}[Q_\theta(\bs, \ba)]$. This would be visible after sufficiently many steps of training, when values have propagated via Bellman backups in Equation~\ref{eqn:cql_training} giving rise to the non-monotonic trend. If such a trend is observed, this raises two questions, as we discuss next.

\textbf{\textit{What does a low average Q-value $\E_{\bs, \ba \sim \data}[Q_\theta(\bs, \ba)]$ imply about $J(\pi)$?}} We show in Appendix~\ref{app:overfitting_and_underfitting} that, in principle, CQL training (Equation~\ref{eqn:cql_training}) should never learn Q-values smaller than the dataset Monte-Carlo return, and the Q-values should increase unless the learned policy $\pi$ is better than $\pi_\beta$. Intuitively, this is because the objective in Equation~\ref{eqn:cql_training} aims to also \emph{maximize} the average dataset Q-value and thus the Q-values for the behavior policy are not underestimated in expectation. Now, if the policy optimizer finds a policy that attains a smaller learned Q-value than the dataset return, the policy can always be updated further towards the behavior policy so as to raise the Q-value. Therefore, Q-values can only decrease when the policy found by CQL is better than the behavior policy. We formalize this intuition in Appendix~\ref{app:overfitting_and_underfitting} in Theorem~\ref{thm:cql_no_reduce}. Of course, these insights only apply to runs where the value of the training CQL regularizer is small, otherwise out-of-distribution Q-values may be overestimated. Thus, a low Q-value on $(\bs, \ba) \in \mathcal{D}$ indicates that the Q-function predicts extremely small Q-values on actions sampled from $\mu(\ba|\bs)$. {Typically this would mean the highest Q-value actions $\ba$ at a state $\bs \in \mathcal{D}$ are those sampled from the offline dataset, drawn from the behavior policy. Thus, policy optimization, which aims to maximize the Q-value, would make $\pi(\ba|\bs)$ closer to the behavior policy $\pi_\beta(\ba|\bs)$ on $\bs \in \mathcal{D}$, implying that the resulting policy would have poor performance $J(\pi)$, that matches or is worse than $\pi_\beta$.}

\textbf{\textit{Which training checkpoint is likely to attain the best policy performance?}} Tracking overfitting in supervised learning is important for selecting the best-performing checkpoint, before overfitting becomes severe. Analogously, rather than quantifying what a ``low'' Q-value means, we can compare the average dataset Q-value across different checkpoints within the same run, using a relative comparison to pick the best policy. 
Since CQL aims to increase the average dataset Q-value (Equation~\ref{eqn:cql_training}), we would expect Q-values to initially increase, until learning starts to overfit and the average dataset Q-value starts decreasing. We should therefore select the latest checkpoint that corresponds to a peak in the estimated dataset Q-value. A visual illustration of this idea is shown in the figure above, where the checkpoint marked by the green line is recommended to be chosen. \textbf{In summary}, \textbf{(a)} to detect overfitting we can track:
\begin{tcolorbox}[colback=blue!6!white,colframe=black,boxsep=0pt,top=3pt,bottom=5pt]
\begin{guideline}[Overfitting]
\label{guideline:overfitting}
A low average data Q-value $\E_{\bs, \ba \sim \mathcal{D}}[Q_\theta(\bs, \ba)]$ that decreases with more gradient steps on Equation~\ref{eqn:cql_training} indicates that the offline RL algorithm is overfitting.
\end{guideline}
\end{tcolorbox} 
and \textbf{(b)} further, given a run that exhibits overfitting, our principle for policy selection is given by:
\begin{tcolorbox}[colback=blue!6!white,colframe=black,boxsep=0pt,top=3pt,bottom=5pt]
\begin{protocol}[Policy selection]
\label{guideline:policy checkpoint}
If a run overfits (per Metric~\ref{guideline:overfitting}), select the checkpoint that attains the highest average dataset Q-value before overfitting for deployment.
\end{protocol}
\end{tcolorbox} 
Finally, for actor-critic algorithms~\citep{fujimoto2018addressing} that update the actor slower than the critic, the next policy checkpoint after the peak in the average dataset Q-value appears must be selected. In most of our experiments, we find that simply utilizing the policy checkpoint at the point of the peak in the Q-value also leads to good results making this a rare concern, but in some cases, utilizing the next checkpoint after the Q-value peak performs better empirically.  

\begin{wrapfigure}{r}{0.25\textwidth}
\vspace{-0.6cm}
\includegraphics[width=\linewidth]{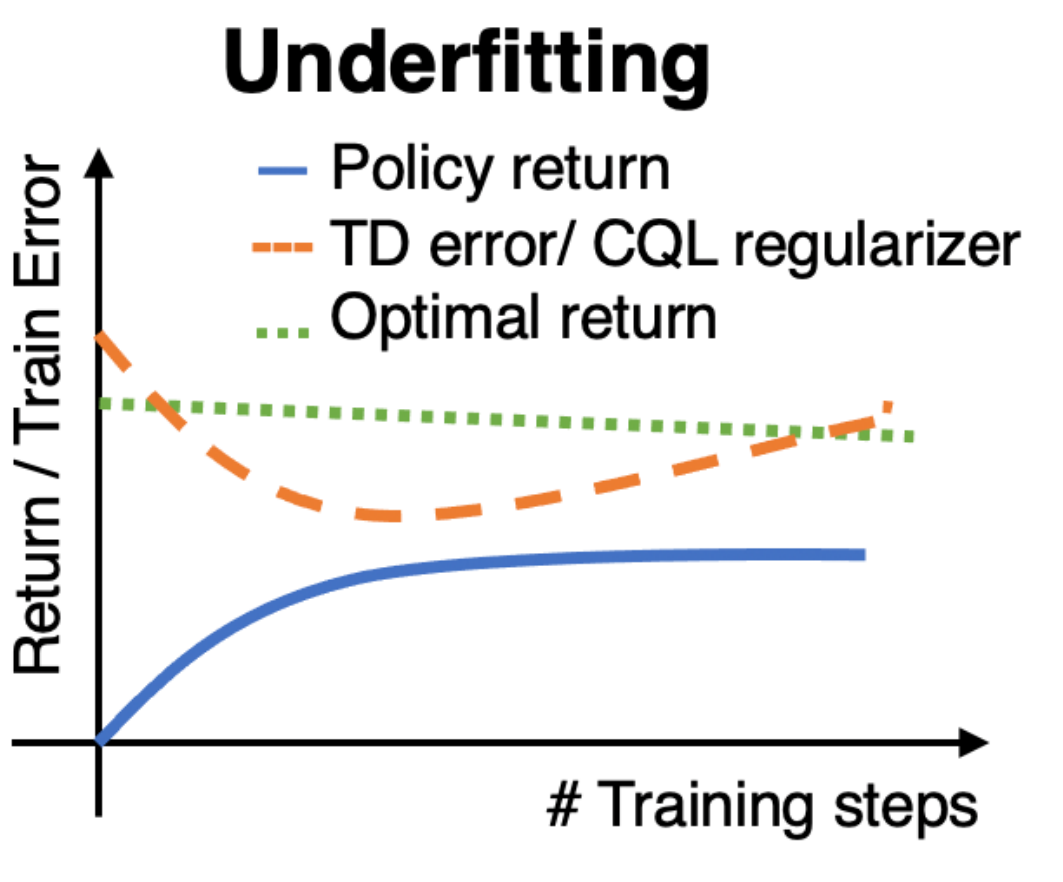}
\vspace{-1cm}
\end{wrapfigure}
\textbf{\underline{Detecting underfitting in CQL.}} Next, we turn to devising a procedure to detect underfitting. As summarized in Table~\ref{tab:summary}, underfitting occurs when the RL algorithm is unable to minimize the training objective in Equation~\ref{eqn:cql_training} effectively. Therefore, large values for the TD error, the CQL regularizer, or both imply underfitting. A large value for the CQL regularizer, $\mathcal{R}(\theta)$, indicates an overestimation of Q-values relative to their true value~\citep{kumar2020conservative} and thus, unlike the overfitting regime, we would \emph{not} expect the average learned Q-value to
decrease with more training.
Thus, one approach to predict underfitting is to track both the TD error, $\mathcal{L}_\mathrm{TD}(\theta)$, and the CQL regularizer, $\mathcal{R}(\theta)$, and check if the  value of even one of these quantities is large. More discussion is provided in Appendix~\ref{app:overfitting_and_underfitting}.   

\textit{\textbf{How do we determine if the TD error and the CQL regularizer are ``large''?}} In order to determine if the error of a particular run is large, we can rerun the base CQL algorithm but with models of higher capacity, {which does not necessarily correspond to the function approximator size, as we will discuss in Section~\ref{sec:addressing_workflow}.} For each model, we record the corresponding training errors and check if the training TD error and CQL regularizer value are reduced with capacity increase. If increasing capacity leads to a reduction in the loss without exhibiting the overfitting signs described previously, then we are in an underfitting regime. Another approach to answer the question is to utilize the value of the TD error ($\mathcal{L}_\text{TD}(\theta)$) and the task horizon ($1/(1 - \gamma)$) to estimate the overall error in the learned Q-values against the actual Q-value, which is equal to $\mathcal{L}_\text{TD}(\theta) / (1 - \gamma)$~\citep{munos2003api} (see Appendix~\ref{app:overfitting_and_underfitting}). If this overall error spans the range of allowed Q-values on the task -- which could be inferred based on the structure of the reward function in the task -- then we can say that the algorithm is underfitting. 
\begin{tcolorbox}[colback=blue!6!white,colframe=black,boxsep=0pt,top=3pt,bottom=5pt]
\begin{guideline}[Underfitting]
\label{guideline:underfitting}
Compute the values of the training TD error, $\mathcal{L}_\mathrm{TD}(\theta)$ and CQL regularizer, $\mathcal{R}(\theta)$ for the current run and another identical run with increased model capacity. If the training errors reduce with increasing model capacity, the original run was underfitting.
\end{guideline}
\end{tcolorbox}
\vspace{-15pt}
\section{Addressing Overfitting and Underfitting in Conservative Offline RL}
\label{sec:addressing_workflow}
\vspace{-10pt}
The typical workflow for supervised learning not only identifies overfitting and underfitting, but also guides the practitioner how to adjust their method so as to alleviate it (e.g., by modifying regularization or model capacity), thus improving performance.
Can we devise similar guidelines to address overfitting and underfitting with conservative offline RL? Here, we discuss some ways to adjust regularization and model capacity to alleviate these phenomena.

\textbf{Capacity-decreasing regularization for overfitting.} As we observed in Section~\ref{sec:workflow_metrics}, the mechanism behind extremely low Q-values on the dataset is that CQL training minimizes Q-values on actions sampled from $\mu(\ba|\bs)$. {Two possible approaches to preventing over-minimization of these values are \textbf{(1)} applying regularization such as dropout~\citep{srivastava2014dropout} on Q-function layers, similar to supervised learning, and \textbf{(2)} enforcing that representations of the learned Q-function match a pre-specified target for all state-action tuples.} For \textbf{(2)}, we can apply techniques such as a variational information bottleneck (VIB)~\citep{alemi2016deep,achille2018emergence} 
regularizer on the learned representations, $\phi(\bs)$. Formally, let $(\bs, \ba)$ denote a state-action pair. Instead of predicting a deterministic $\phi(\bs) \in \mathbb{R}^d$ (Figure~\ref{fig:architectures}), we modify the Q-network to predict two distinct vectors, $\phi_m(\bs) \in \mathbb{R}^d$ and $\phi_\Sigma(\bs) \in \mathbb{R}^d$, and sample $\phi(\bs)$ randomly from a Gaussian centered at $\phi_m$ with covariance $\phi_\Sigma$, i.e., $\phi(\bs) \sim \mathcal{N}(\phi_m(\bs), \text{diag}(\phi_\Sigma(\bs))$. VIB then regularizes $\mathcal{N}(\phi_m(\bs), \text{diag}(\phi_\Sigma(\bs))$ to be close to a prior distribution, $\mathcal{N}(0, \mathbb{I})$: \vspace{-0.2cm}

\begin{equation}
\label{eqn:bottleneck}
    \min_\theta~~ \mathcal{L}_\text{CQL}(\theta) + \beta \E_{\bs \sim \mathcal{D}} \left[\mathrm{D_{KL}}\left(\mathcal{N}(\phi_m(\bs), \text{diag}(\phi_\Sigma(\bs)))~\big\vert \big\vert~\mathcal{N}(0, \mathbb{I})\right)\right] ~~~~~~~ \text{(VIB regularizer)},~~~~
\end{equation}
\vspace{-0.2cm}
\begin{tcolorbox}[colback=blue!6!white,colframe=black,boxsep=0pt,top=3pt,bottom=5pt]
\begin{protocol}
\label{guideline:addressing_overfitting}
To address overfitting, we recommend using some form of capacity-decreasing regularization on the Q-function, such as dropout or the VIB regularizer shown in Equation~\ref{eqn:bottleneck}.
\end{protocol}
\end{tcolorbox} 

\textbf{Capacity-increasing techniques for underfitting.} To address underfitting, we need to increase model capacity to improve optimization of the training objective. Analogous to supervised learning, model capacity can be increased by using more expressive neural nets (e.g., ResNets~\citep{he2016deep}, transformers~\citep{vaswani2017attention}) for representing the learned policy. We use ResNets in our experiments (Figure~\ref{fig:architectures}). However, the RL setting presents an additional challenge with capacity: while larger models \emph{in principle} have more capacity, recent work~\citep{ghosh2020representations,kumar2021implicit,anonymous2021value} has shown that utilizing larger networks to represent Q-functions does not always improve its capacity in practice, because TD-based RL methods introduce an ``implicit under-parameterization'' effect that can result in aliased (i.e., similar) internal representations for different state-action inputs, even for very large neural networks that can express the true Q-function effectively. To address this issue, these works apply a ``capacity-increasing'' regularizer to Q-function training. For instance, we can use the DR3 regularizer~\citep{anonymous2021value}, which penalizes the dot product of $\phi(\bs)$ and $\phi(\bs')$ for a transition $(\bs, \ba, \bs') \in \mathcal{D}$, and hence reduces aliasing. This objective is given by:\vspace{-0.6pt}

\begin{equation}
\label{eqn:dr3}
    \min_\theta~~ \mathcal{L}_\text{CQL}(\theta) + \beta \E_{\bs, \ba, \bs' \sim \mathcal{D}} \left[ \big\vert\phi(\bs)^\top \phi(\bs') \big\vert\right] ~~~~~~~~~ \text{(DR3 regularizer~\citep{anonymous2021value})},~~~~~ 
\end{equation}
\vspace{-0.2cm}
\begin{tcolorbox}[colback=blue!6!white,colframe=black,boxsep=0pt,top=3pt,bottom=5pt]
\begin{protocol}
\label{guideline:addressing_underfitting}
To address underfitting, we recommend using some capacity-increasing regularization on the Q-function and the policy either in conjunction or separately. Examples: \textbf{(1)} bigger policy networks (e.g., ResNets), \textbf{(2)} DR3 regularizer on the Q-network. 
\end{protocol}
\end{tcolorbox} 

\vspace{-0.1cm}
\section{What About the Hyperparameter $\alpha$?}
\label{sec:alpha_tuning_main}
\vspace{-0.1cm}
The guidelines in the preceding paragraph suggest how to adjust capacity, but do not tell us how to tune the multiplier on the CQL term, $\alpha$, in Equation~\ref{eqn:cql_training}. This multiplier trades off minimizing TD error with a correction for distributional shift. An inappropriate choice of $\alpha$ will inhibit good policy performance, since CQL would be insufficiently constrained against out-of-distribution actions with excessively low values of $\alpha$, while being too constrained to stay close to the dataset with excessively high values. In our experiments, both in simulation and in the real-world, we found that a default value of $\alpha=1.0$ taken from prior work~\citep{singh2020cog} worked for all scenarios without any tuning; however, we do provide guidelines for tuning $\alpha$ values if required. We expect that tuning $\alpha$ is especially needed when the data is highly diverse or when it is generated from a narrow expert policy. 

\textbf{How can we detect excessively large $\alpha$ values?} Since a larger value of $\alpha$ would correspond to a higher weight on the CQL regularizer $\mathcal{R}(\theta)$, which minimizes Q-values, we would expect that Q-values learned with a large $\alpha$ would exhibit an overfitting trend per Metric~\ref{guideline:overfitting}, where Q-values would decrease with more training steps. Thus, if the Q-values on the dataset exhibit a decreasing (overfitting-like) trend despite applying the mitigation strategies in Section~\ref{sec:addressing_workflow}, it indicates that $\alpha$ may be too large and we need to reduce $\alpha$. This is formalized as:
\begin{tcolorbox}[colback=blue!6!white,colframe=black,boxsep=0pt,top=3pt,bottom=5pt]
\begin{protocol}[Guideline for decreasing $\alpha$]
\label{guideline:decreasing_alpha_main}
If a run of CQL with $\alpha = \alpha_0$ exhibits a trend that resembles overfitting per Metric~\ref{guideline:overfitting} and correcting for overfitting based on Guideline~\ref{guideline:addressing_overfitting} does not address it, then re-run CQL with a smaller value of $\alpha=\alpha_1$. If this new run with $\alpha = \alpha_1$ exhibits overfitting as well, decrease $\alpha$ from $\alpha_0$ and $\alpha_1$. 
\end{protocol}
\end{tcolorbox}

\textbf{How can we detect excessively small $\alpha$ values?} When $\alpha$ is too small, we would expect that the Q-values do not decrease with more training, since the CQL regularizer has minimal effect. Thus a run of CQL with a very small $\alpha$ will resemble underfitting, as identified by Metric~\ref{guideline:underfitting}. Given a run with non-decreasing Q-values and a high value of the training CQL regularizer, our first step is to determine if the run is underfitting due to insufficient capacity or just has a smaller $\alpha$. Thus, we first detect underfitting (Metric~\ref{guideline:underfitting}) and re-run training with a higher-capacity model (e.g., a Resnet policy, DR3~\citep{anonymous2021value} capacity-increasing regularizer). If we find that even higher-capacity models are unable to reduce the value of the CQL regularizer during training, then this indicates that $\alpha$ is too small. This is expected because, no matter what the capacity of the model, a small $\alpha$ would cause the policy to pick unseen, out-of-distribution actions due to erroneous Q-function overestimation. Once such a scenario is detected, we can increase $\alpha$, until the value of the CQL regularizer is sufficiently negative and then utilize the other workflow guidelines.
\begin{tcolorbox}[colback=blue!6!white,colframe=black,boxsep=0pt,top=3pt,bottom=5pt]
\begin{protocol}[Guideline for increasing $\alpha$]
\label{guideline:increasing_alpha_main}
If a run of CQL exhibits a trend that resembles underfitting per Metric~\ref{guideline:underfitting}, and increasing model capacity per recommendations mentioned in Guideline~\ref{guideline:addressing_underfitting} does not reduce the CQL regularizer, then we suggest first increasing the coefficient of the CQL regularizer $\alpha$ until the final value of the CQL regularizer is lower than 0, and then applying the other workflow guidelines with this new $\alpha$ value.
\end{protocol}
\end{tcolorbox}

%% file: sim_exps.tex
 \vspace{-0.3cm}
\section{Evaluation of Our Workflow Metrics and Protocols in Simulation}
\label{sec:workflow_exps}
\vspace{-0.3cm}

\begin{wrapfigure}{r}{0.43\textwidth}
\vspace{-0.8cm}
\begin{center}
\includegraphics[width=\linewidth]{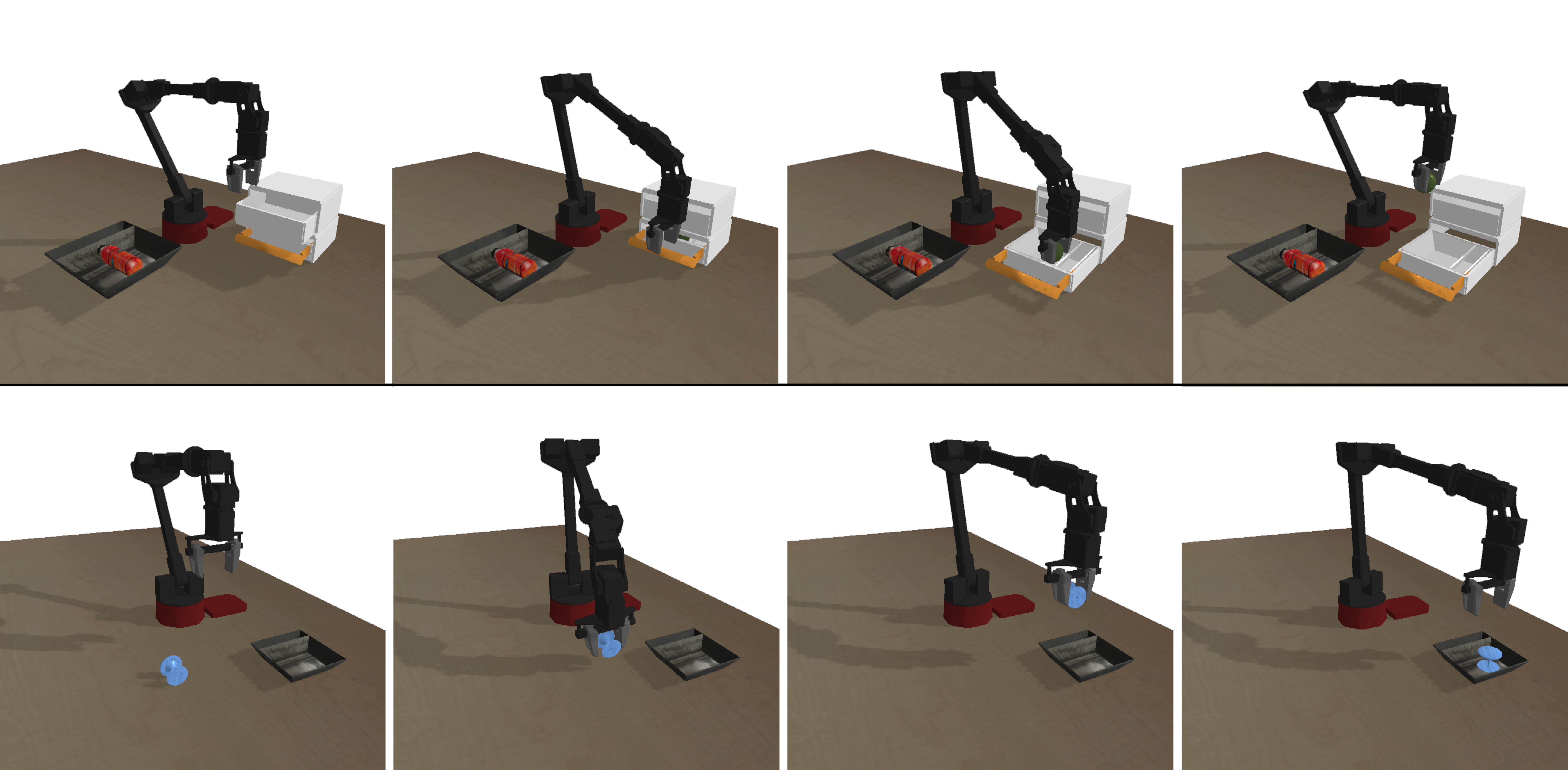}
\vspace{-0.6cm}
\end{center}
\caption{\label{fig:sim_envs} \footnotesize{\textbf{Simulated  domains~\citep{singh2020cog}. } Examples of trajectories for the \figtop \; pick and place task and \figbottom \; grasping from drawer task.}}
\vspace{-0.49cm}
\end{wrapfigure}
Next, we empirically validate the workflow proposed in Sections~\ref{sec:workflow_metrics} and \ref{sec:addressing_workflow} on a suite of simulated robotic manipulation domains that mimic real-robot scenarios, operating directly from image observations with sparse binary rewards. We will examine how applying the workflow in Section~\ref{sec:workflow_metrics} to detect overfitting or underfitting and then utilizing the strategies in Section~\ref{sec:addressing_workflow} affects the performance of offline RL methods. An improved performance would indicate the efficacy of our workflow in guiding a practitioner in making successful design decisions without any online tuning.  

\begin{wrapfigure}{r}{0.55\textwidth}
\vspace{-0.7cm}
\begin{center}
\includegraphics[width=\linewidth]{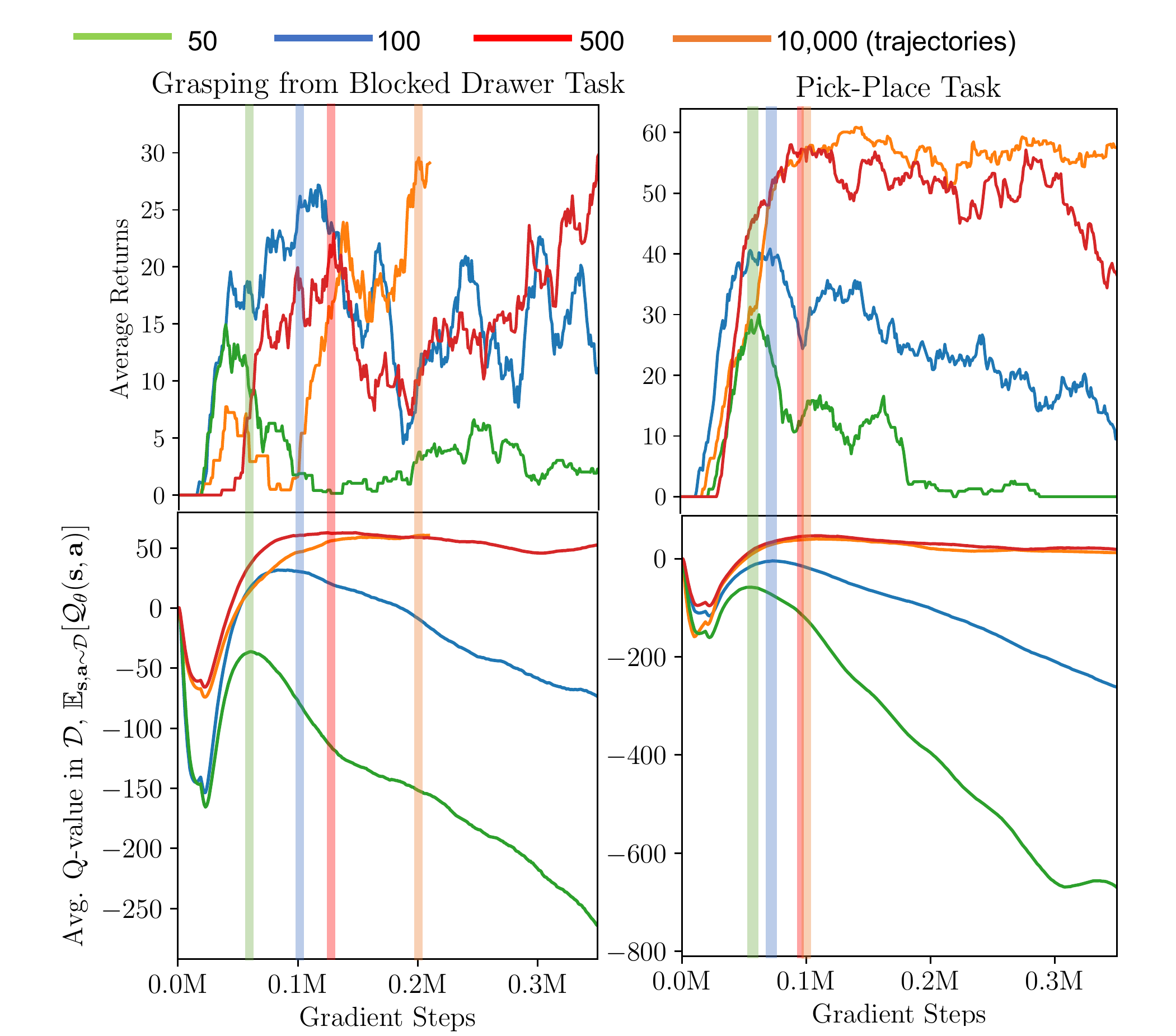}
\vspace{-0.7cm}
\end{center}
\caption{\footnotesize{\label{fig:overfitting_grasping_task} \textbf{Policy performance (Top) and average dataset Q-values of CQL (bottom) with varying number of trajectories.} Vertical bands indicate regions around the peak in average Q-value and observe that these regions correspond to policies with good actual performance.}}
\vspace{-0.4cm}
\end{wrapfigure}
\textbf{Experimental setup.} We use the environments from \citet{singh2020cog} to design offline RL tasks and datasets that we use for our empirical analysis. We consider two tasks: \textbf{(1)} a pick and place task and \textbf{(2)} a grasping object from a drawer task. Examples of trajectories in both of these simulated domains are shown in Figure~\ref{fig:sim_envs} and are detailed in Appendix~\ref{app:setup_details}. Briefly, the \emph{\textbf{pick and place}} task consists of a 6-DoF WidowX robot in front of a tray with an object. The goal is to put the object inside the tray. A non-zero reward of +1 is provided only when the object has been placed in the box. The offline dataset for this task consists of trajectories that grasp an object with a 35\% success and other trajectories that place an object with a 40\% success. Our second task is a \emph{\textbf{grasping from drawer task}} where the WidowX robot is placed in front of a drawer and multiple objects. The robot can open or close the drawer, grasp objects from inside the drawer or on the table, and place them anywhere in the scene. The goal is to close the top drawer, then open the bottom drawer and take the object out. Only if the object has been taken out, a reward of +1 is obtained. The offline dataset consists of trajectories with a 30-40\% success rate for opening and closing a drawer and other trajectories with only 40\% placing success. We use $\alpha=1.0$ for CQL training in all experiments, wich is directly taken from prior work~\citep{singh2020cog}, without any tuning. However, too low or too high $\alpha$ values will inhibit the effectiveness of regular CQL and we first need to tune $\alpha$ using Guidelines~\ref{guideline:decreasing_alpha_main} and \ref{guideline:increasing_alpha_main} in such scenarios, before addressing underfitting and overfitting as discussed in Section~\ref{sec:alpha_tuning_main}. We present results assessing the efficacy of Guidelines~\ref{guideline:increasing_alpha_main} and \ref{guideline:decreasing_alpha_main} in Appendix~\ref{app:alpha_hparam}. We utilize a standard convolutional net for the Q-function (Fig.~\ref{fig:standard}) and the policy. More details on our setup are provided in Appendix~\ref{app:setup_details}.

\textbf{Scenario \#1: Variable amount of training data.} Our first scenario consists of the simulated tasks discussed above with a variable number of trajectories in the training data (50, 100, 500, 10000). We run CQL and track metrics~\ref{guideline:overfitting} and \ref{guideline:underfitting} in each case. Observe in  Figure~\ref{fig:overfitting_grasping_task} (bottom) that with fewer trajectories, the average dataset Q-value $\E_{\bs, \ba \sim \data}[Q_\theta(\bs, \ba)]$ first rises, and then drops. This matches the description of overfitting in Section~\ref{sec:workflow_metrics}. Observe in Figure~\ref{fig:bottleneck_plots} (left) that, at the same time, the value of the CQL regularizer is very low, which is not consistent
\begin{wrapfigure}{r}{0.55\textwidth}
\vspace{-0.5cm}
\begin{center}
\includegraphics[width=\linewidth]{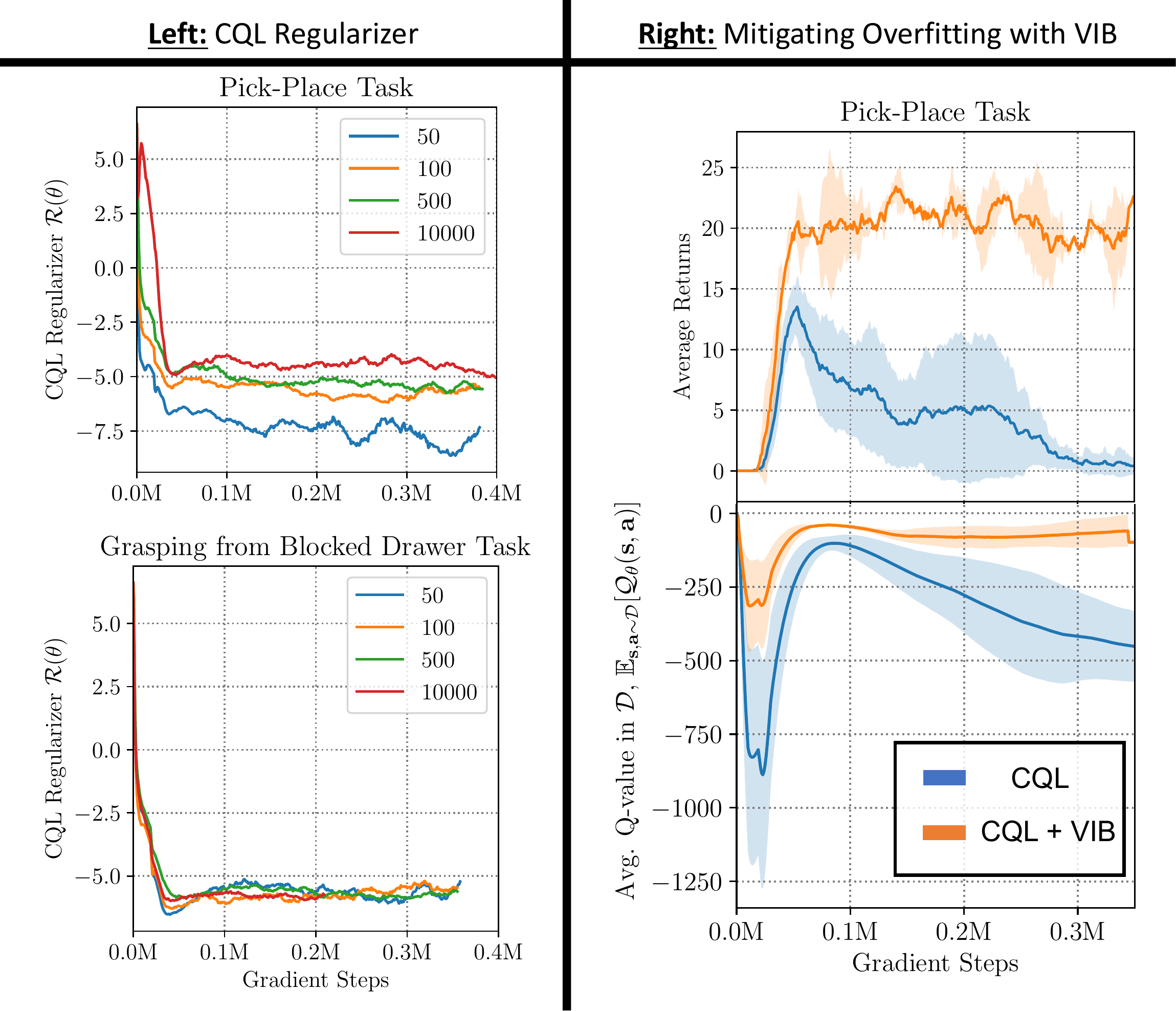}
\vspace{-0.75cm}
\end{center}
\caption{\footnotesize{\label{fig:bottleneck_plots} \textbf{Left:} CQL regularizer attains low values, especially with 50 and 100 trajectories in the pick and place task, \textbf{Right:} Using VIB mitigates overfitting, giving rise to a stable trend in Q-values and better performance which does not degrade with more training steps.}}
\vspace{-0.55cm}
\end{wrapfigure}
with what we expect of underfitting. Thus, we can conclude that these conditions exhibit overfitting, especially with 50 and 100 trajectories. The vertical dashed lines indicate the checkpoints that would be selected for evaluation per Guideline~\ref{guideline:policy checkpoint}. We further visualize the performance of the chosen checkpoints against the actual return of each intermediate policy in Figure~\ref{fig:overfitting_grasping_task} (top). Note that this value is obtained by rolling out the learned policy, and would not be available in a realistic offline RL setting, but is provided only for analysis. Selecting the checkpoint based on Guideline~\ref{guideline:policy checkpoint} leads us to select a model with close to the peak performance over the training process, validating the efficacy of Guideline~\ref{guideline:policy checkpoint}.

Since we detected overfitting by following our workflow, we now aim to address it by using the VIB regularizer in the setting with 100 trajectories. As shown in Figure~\ref{fig:bottleneck_plots} (right), applying this regularizer not only alleviates the drop in Q-values after many training steps, but allows us to pick later checkpoints in training which perform better than base CQL on both the tasks. This validates that overfitting, as detected via our workflow, can be effectively mitigated by decreasing capacity, in this case by using VIB. {We evaluate dropout, $\ell_1$ and $\ell_2$ regularization schemes in Appendix~\ref{app:other_overfitting_corrections}.}

\begin{wrapfigure}{r}{0.72\textwidth}
\vspace{-0.75cm}
\begin{center}
\includegraphics[width=0.32\linewidth]{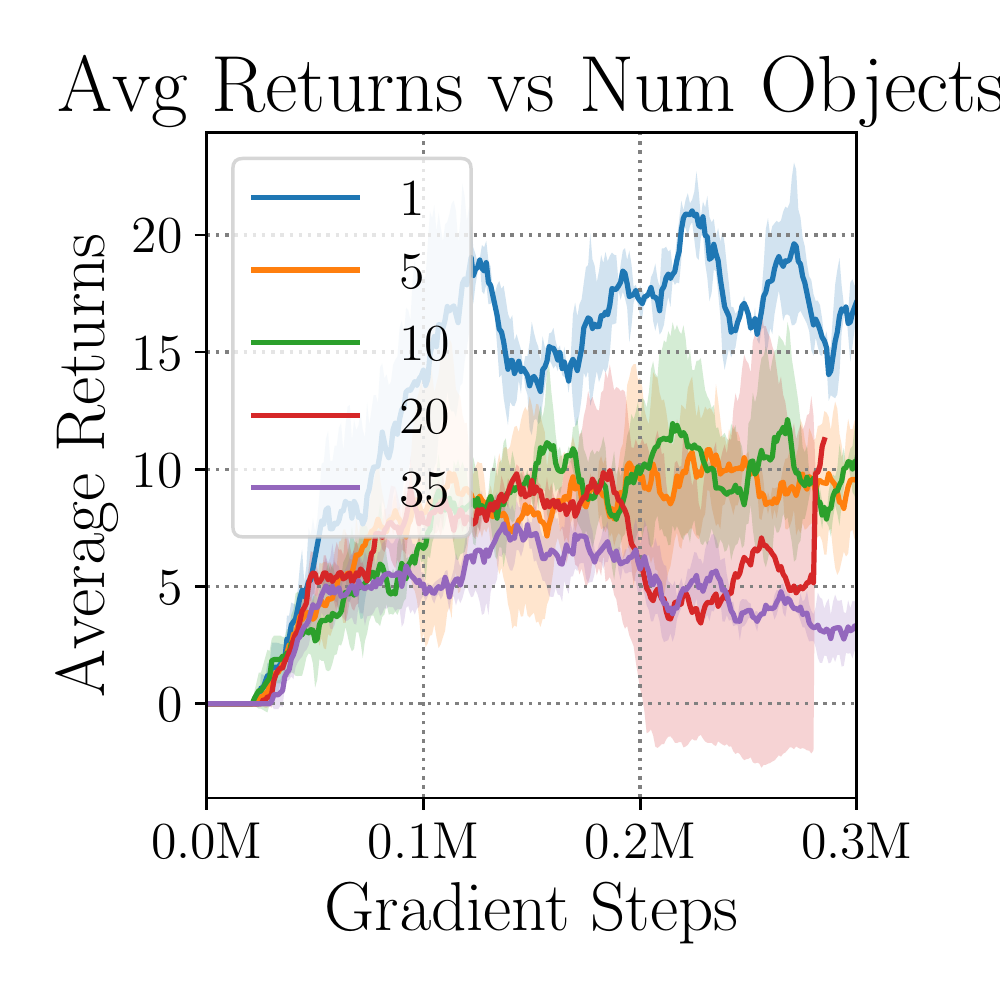}
\includegraphics[width=0.32\linewidth]{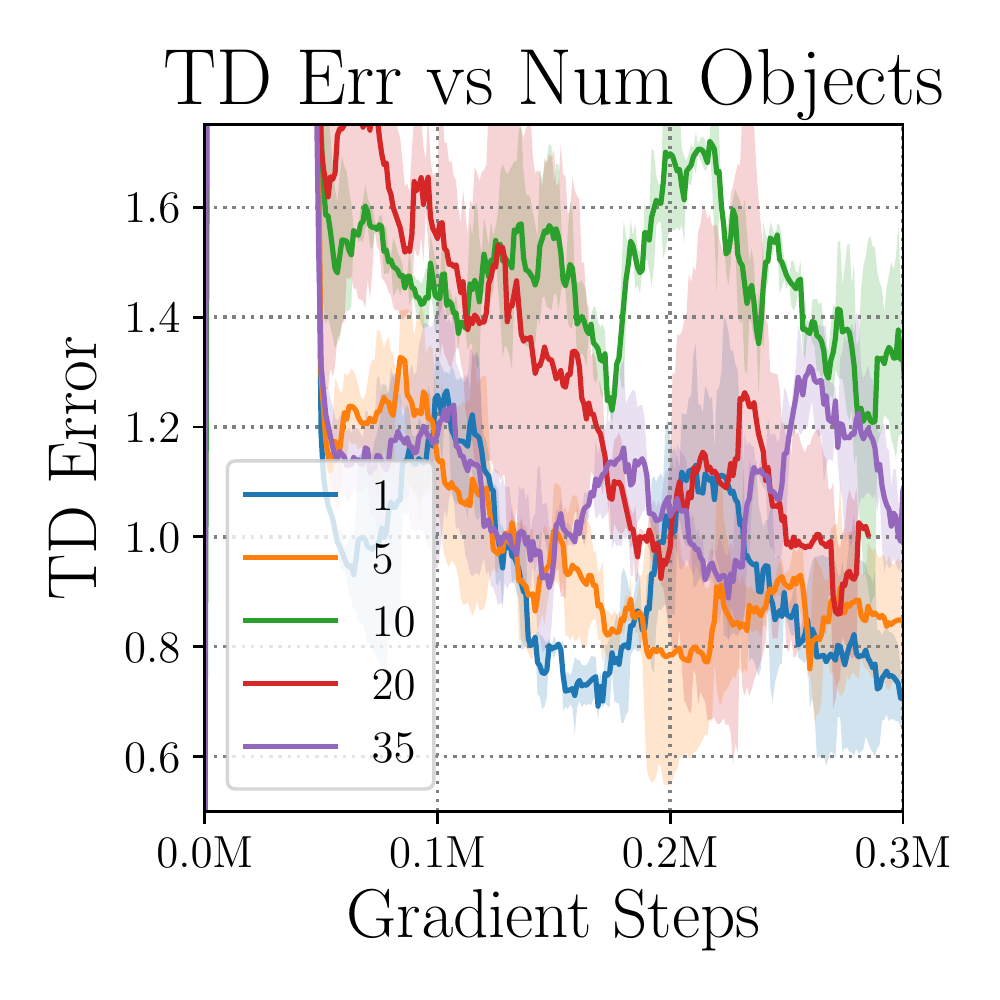}
\includegraphics[width=0.32\linewidth]{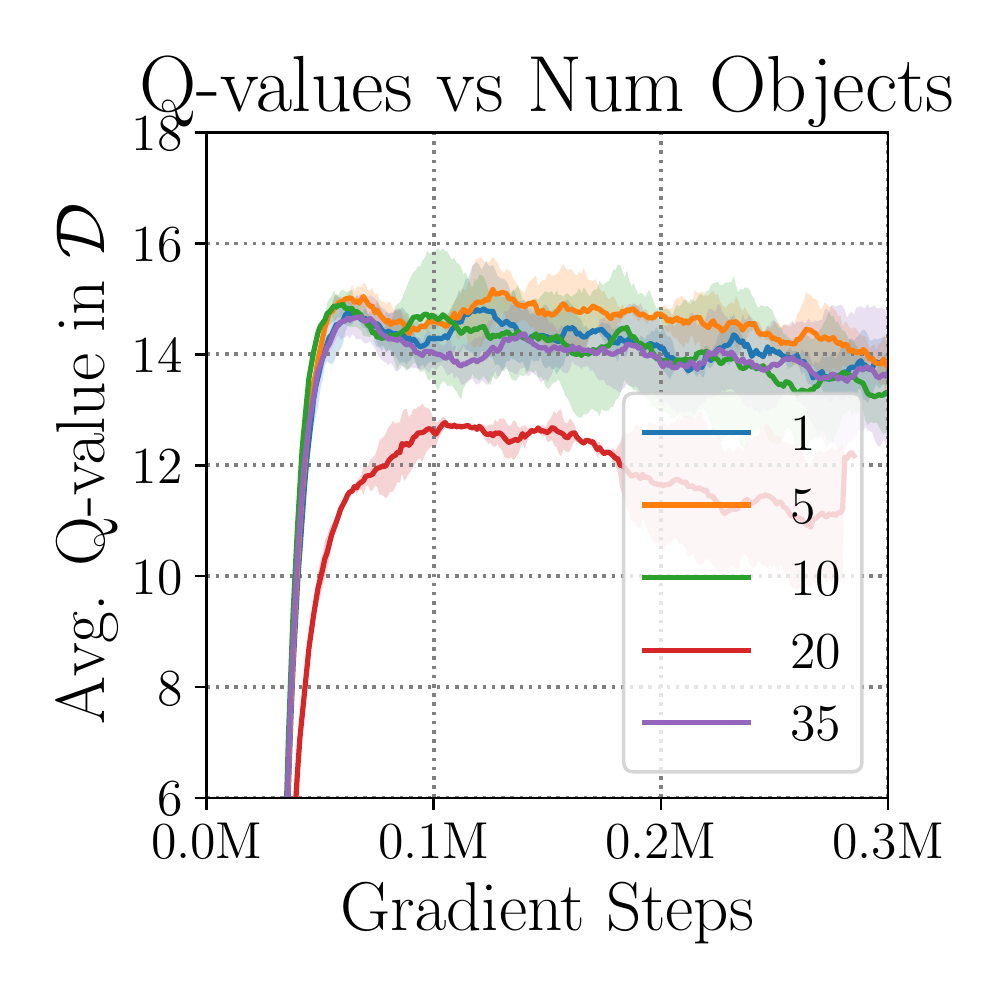}
\vspace{-0.4cm}
\end{center}
\caption{\footnotesize{\label{fig:object_diversity} \textbf{Performance (left), TD error (middle) and average dataset Q-values (right) for the pick and place task with a variable number of objects.} Note that while the learned Q-values increase and stabilize, the TD error values in scenarios with more than 10 objects are large (1.0-2.0). Correspondingly, the performance generally decreases as the number of objects increases.}}
\vspace{-0.35cm}
\end{wrapfigure}
\textbf{Scenario \#2: Multiple training objects.} Our second test scenario consists of the pick and place task, modified to include a variable number of object types (1, 5, 10, 20, 35). Handling more objects requires higher capacity, since each object has a different shape and appearance. In each case, CQL is provided with 5000 trajectories. Following our workflow from Section~\ref{sec:workflow_metrics}, we first compute the average dataset Q-value and the training TD error. We observe in Figure~\ref{fig:object_diversity} that, unlike in Scenario \#1, Q-values do not generally decrease when trained for many steps, suggesting that the Q-function is likely not overfitting. To check for underfitting, we visualize the training TD error and find that, with 10, 20 and 35 objects, TD error magnitudes are in the range of [1.0, 2.0], which suggests a overall Q-value error of [30.0, 60.0] since the task horizon is 30. On an absolute scale, this error magnitude is large: since the rewards are 0/1, the range of difference between actual Q-values for any two policies is at most 30, which suggests that the error magnitude in the runs in Figure~\ref{fig:object_diversity} are high. Hence, we conclude that this scenario generally exhibits underfitting with more objects. Indeed this trend is reflected in the policy performance that we plot for analysis in Figure~\ref{fig:object_diversity}: note that the policy return decreases with an increased number of objects, and the policy performance initially increases and saturates at a suboptimal value in the settings that exhibit underfitting, consistent with Section~\ref{sec:workflow_metrics}. 

\begin{wrapfigure}{r}{0.23\textwidth}
\vspace{-0.3cm}
\begin{center}
\hspace{-10pt} \includegraphics[width=1.1\linewidth]{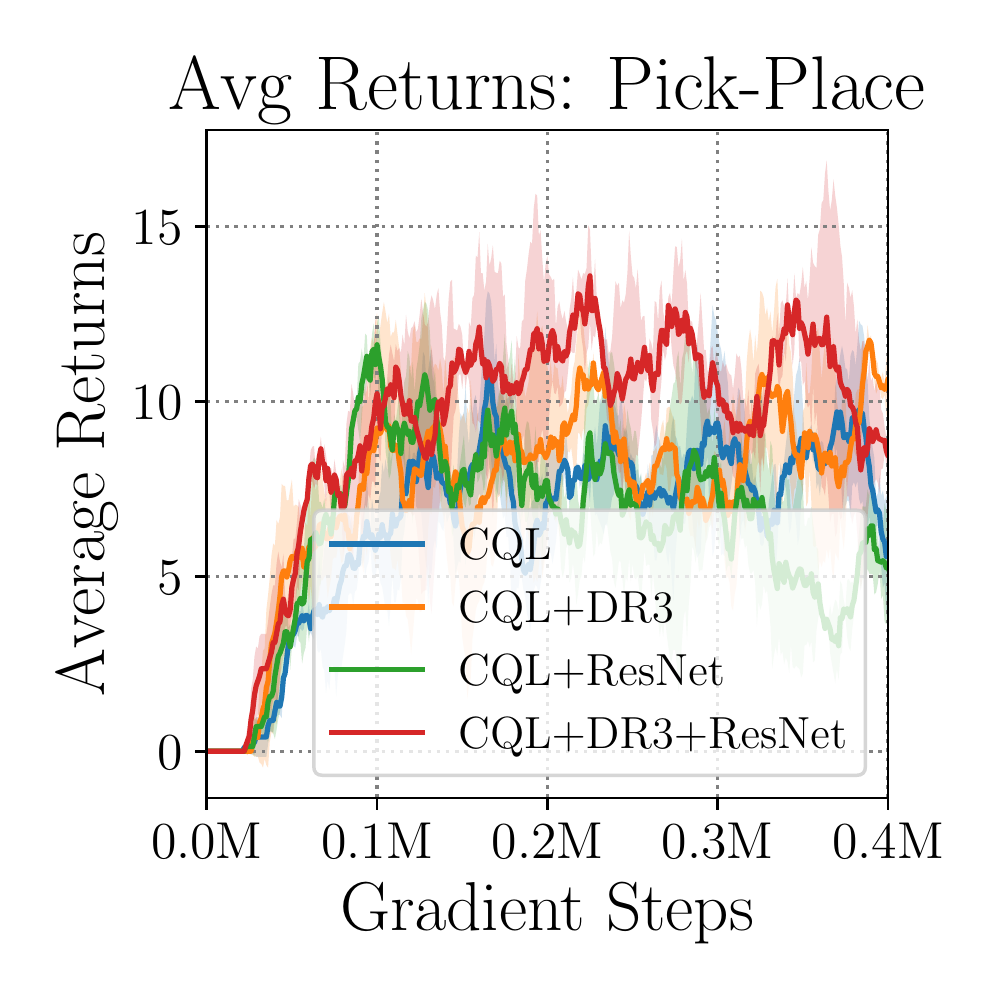}
\vspace{-0.9cm}
\end{center}
\caption{\footnotesize{\label{fig:pickplace35} Our underfitting correction on the 35-object pick-place task.}}
\end{wrapfigure}
To address underfitting detected by our workflow in the multi-object case, we apply the proposed capacity-increasing measures to the 35-object task {(results for 10 and 20 object settings are in Appendix~\ref{app:other_num_objects})}.
We use a more expressive ResNet architecture for the policy and the DR3 regularizer for the Q-function together. Observe in the figure on the right that this combination (shown in red) improves policy performance in this setting (compared to green), which validates our workflow protocol for addressing underfitting. Metrics for this run are provided in Appendix~\ref{app:sim_studies}.

We provide simulated experiments for tuning $\alpha$ in Appendix~\ref{app:alpha_hparam} and also our apply our workflow to effectively tune a different offline RL algorithm, BRAC~\citep{wu2019behavior} in Appendix~\ref{app:brac_example}.

\begin{wrapfigure}{r}{0.6\textwidth}
\vspace{-0.5cm}
\begin{center}
\includegraphics[width=\linewidth]{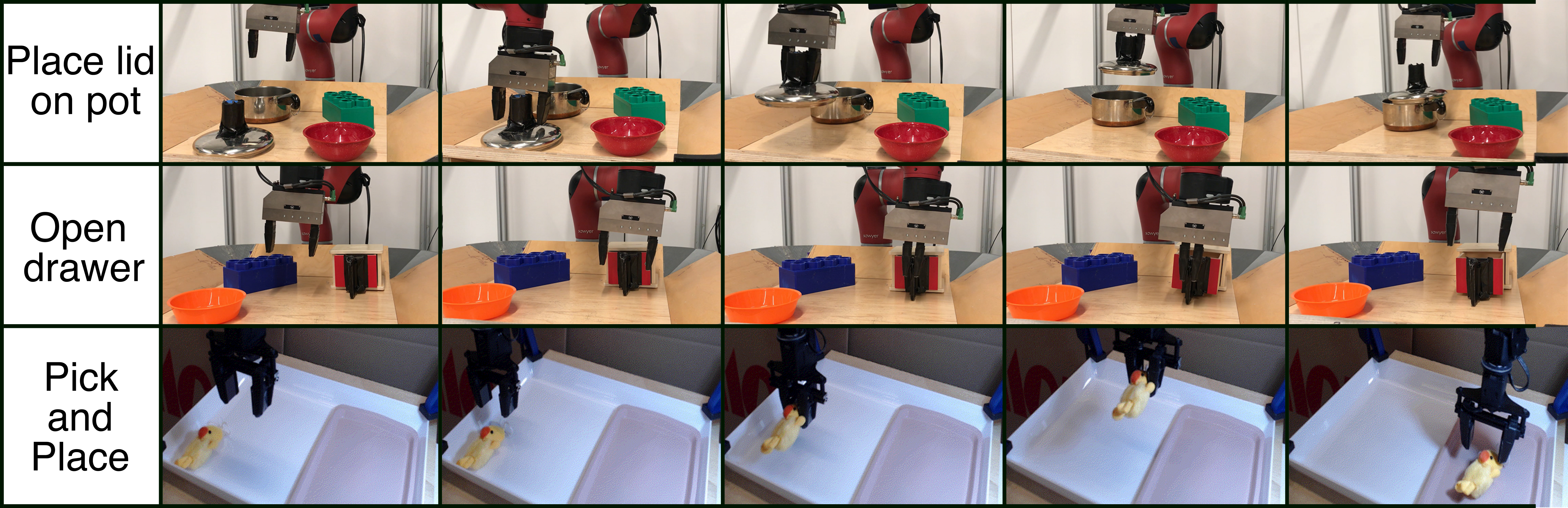}
\vspace{-0.6cm}
\end{center}
\caption{\label{fig:real_tasks} \footnotesize{\textbf{Real-world tasks}. Successful rollouts of CQL tuned with our workflow from Sections~\ref{sec:workflow_metrics} \& \ref{sec:addressing_workflow}. Top to bottom:\; Sawyer lid on pot, Sawyer drawer opening, WidowX pick-place task. }}
\vspace{-0.6cm}
\end{wrapfigure}

%% file: real_experiments.tex
\vspace{-5pt}
\section{Tuning CQL for Real-World Robotic Manipulation}
\label{sec:real_world_case}
\vspace{-5pt}

Having evaluated the efficacy of our proposed workflow in simulation, we now utilize our workflow to tune CQL for real-world robotic manipulation. We test in two setups that require the robot to learn from sparse binary rewards and image observations. The settings differ in robot platform, task specification, and dataset size. Additional results and robot videos are at the following website: \url{https://sites.google.com/view/offline-rl-workflow}

\textbf{Sawyer manipulation tasks~\citep{khazatsky2021can}.} First, we train a Sawyer robot in a tabletop setting to perform two tasks: \textbf{(1)} placing the lid onto a pot and \textbf{(2)} opening a drawer. The robot must perform these tasks in the presence of visual distractor objects, as shown in Figure~\ref{fig:real_tasks}. We directly use the dataset of 100 trajectories for each task collected by \citet{khazatsky2021can} for our experiments so as to mimic the real-world use case of leveraging existing data with offline RL. We use four-dimensional actions with 3D end-effector velocity control in $xyz$-space and 1D gripper open/close action. More details regarding the setup are provided in Appendix \ref{app:setup_details}.

\begin{wrapfigure}{r}{0.5\textwidth}
\vspace{-0.85cm}
\begin{center}
\includegraphics[width=0.49\linewidth]{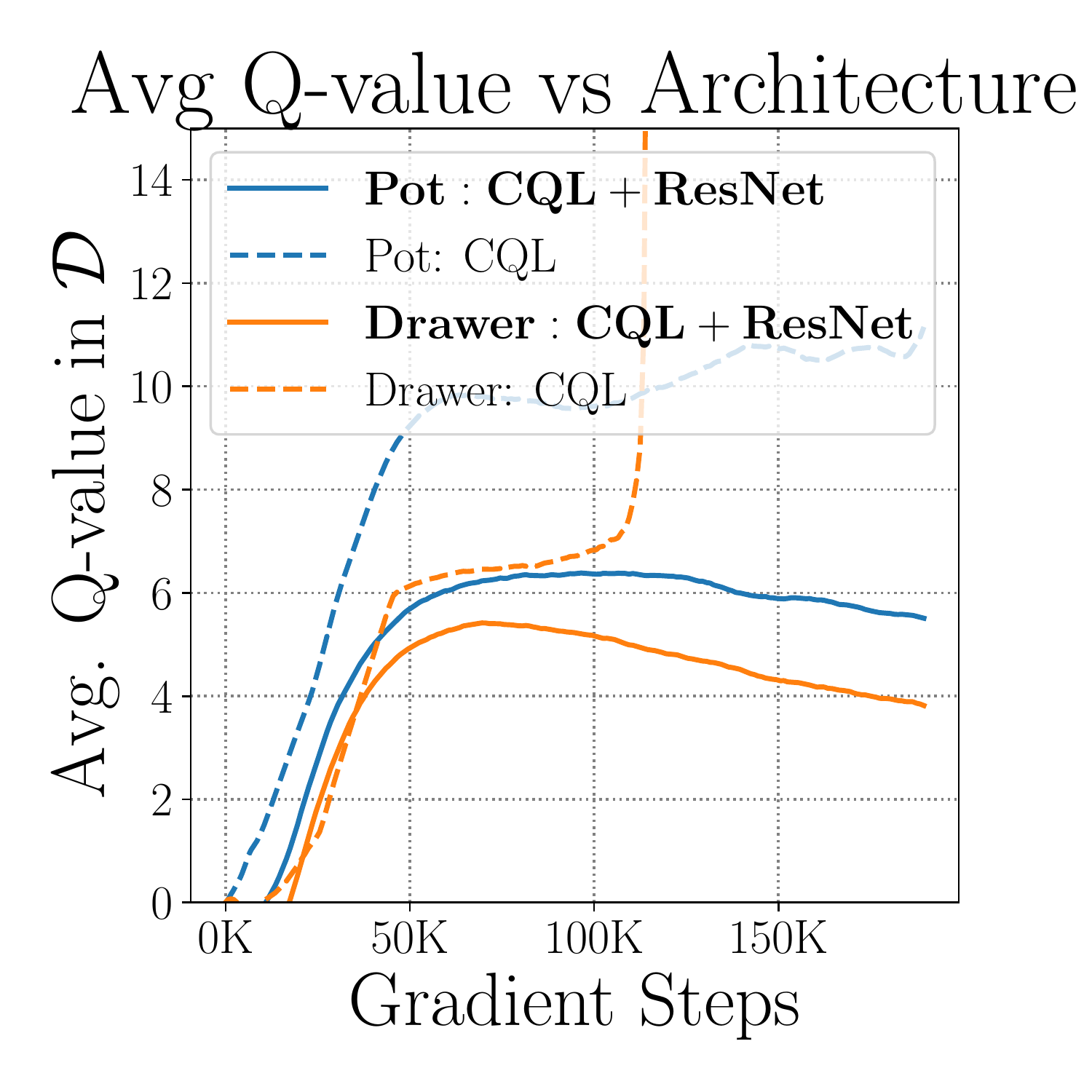}
\includegraphics[width=0.49\linewidth]{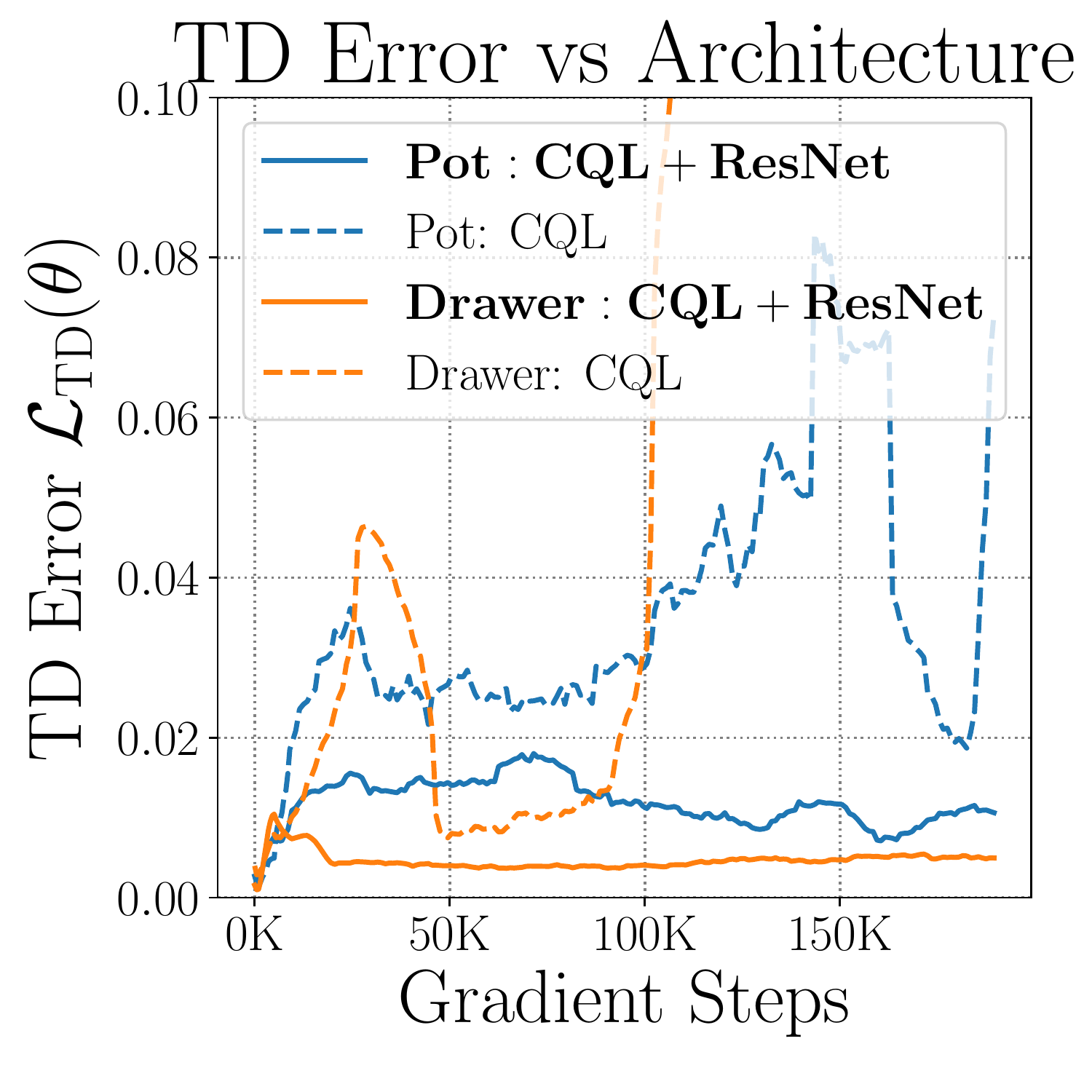}
\vspace{-0.8cm}
\end{center}
\caption{\footnotesize{\label{fig:sawyer_training} Average Q-value and TD error on Sawyer tasks as model capacity increases. Q-values increase over training with lower capacity ruling out overfitting and increasing model capacity leads to a reduction in TD error indicating the presence of underfitting.}}
\vspace{-0.55cm} 
\end{wrapfigure}
We run default CQL on these tasks and track the average Q-value, TD error, and CQL regularizer value. As shown in Figure~\ref{fig:sawyer_training}, the average Q-value does not decrease over training, and the TD error (and CQL regularizer shown in Appendix~\ref{app:real_studies}) is large. Per our discussion in Section~\ref{sec:workflow_metrics}, this indicates underfitting. Following our guidelines from Section~\ref{sec:addressing_workflow}, we utilize a more expressive ResNet
policy (Figure~\ref{fig:architectures}), which increases the number of total convolutional layers from 3 to 9. We observe that this reduces the values of both the TD error Figure~\ref{fig:sawyer_training} and CQL regularizer (Appendix~\ref{app:real_studies}) on both tasks. We then evaluate  the learned policy over 12 trials conducted with different sets of distractor objects, including ones that are unseen during training. While the policy trained using base CQL
is unable to successfully complete either task even once attaining a score of 0/12 on both tasks, the run that uses ResNet
attains a significantly better success rate of \textbf{9/12} on the put lid on pot task and \textbf{8/12} on the drawer opening task, equal to \textbf{70.8\%} success rate on average.

\textbf{WidowX pick and place task.} In our second setting, we tune CQL on a pick and place task with a WidowX 250 robotic arm, shown in Figure~\ref{fig:real_tasks}. The dataset consists of 200 trajectories collected by running a noisy scripted policy (Appendix~\ref{app:setup_details}) with 35\% success. We run CQL on this task and track the average Q-values, which we find initially increase and then decrease (Figure~\ref{fig:bottleneck_plots_real} (left; labeled as ``Q-values'')), indicating overfitting. We then evaluate our policy selection scheme, which in this case suggests deploying checkpoint 50, the immediate checkpoint after the peak in Q-values. 
To see if this checkpoint is effective, we evaluate the performance of a few other policy checkpoints (for analysis only) and plot this performance trend in Figure~\ref{fig:bottleneck_plots_real} (right) as a dashed line.
Observe that indeed the checkpoint found by our approach attains the highest success rate (\textbf{7/9}) compared to other checkpoints, which only succeed $\leq$~4/9 times. This indicates the efficacy of our proposed metrics in identifying overfitting and our policy selection guideline.

\begin{figure}[t]
\centering
\includegraphics[width=0.32\linewidth]{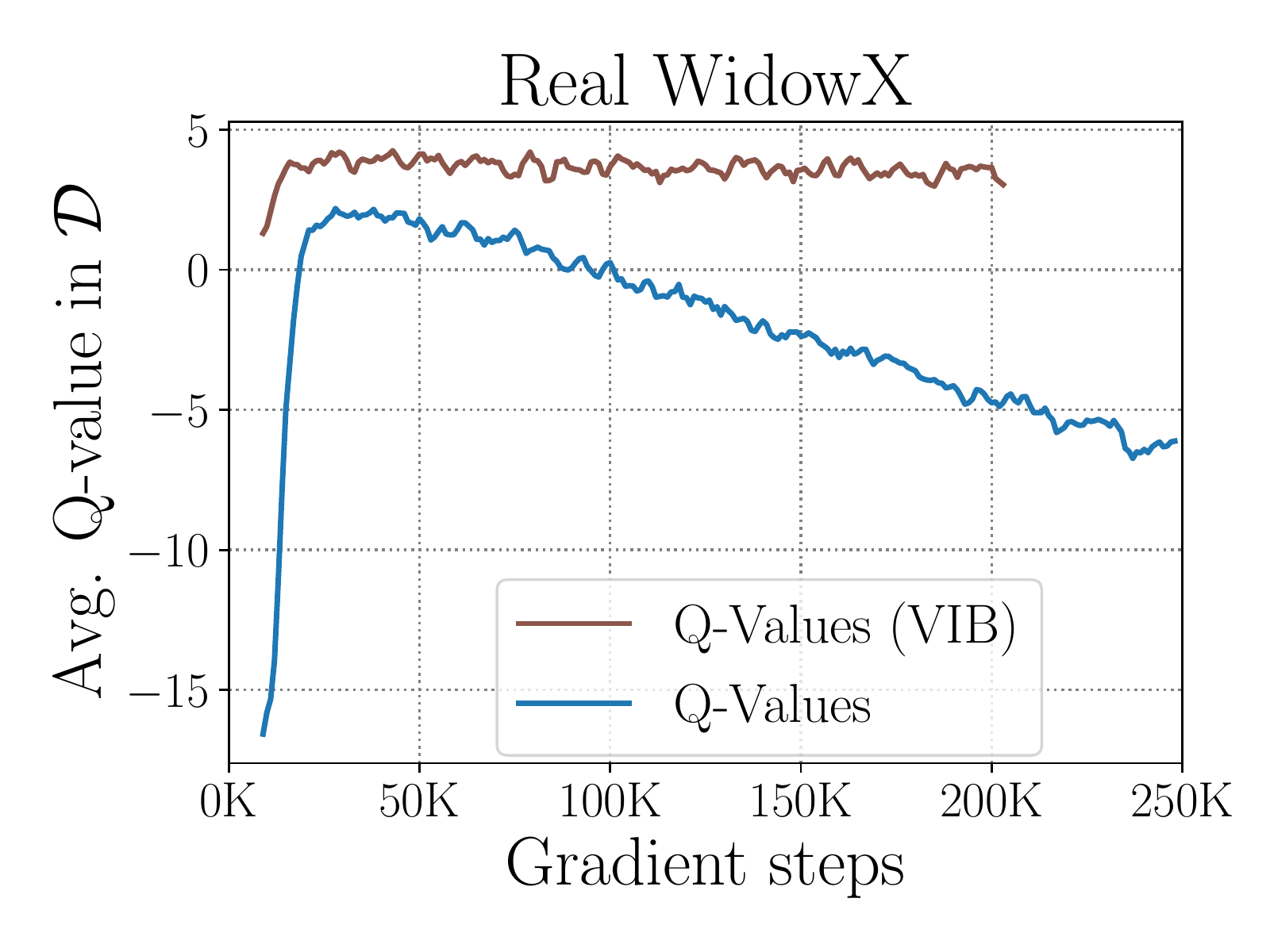}
\includegraphics[width=0.32\linewidth]{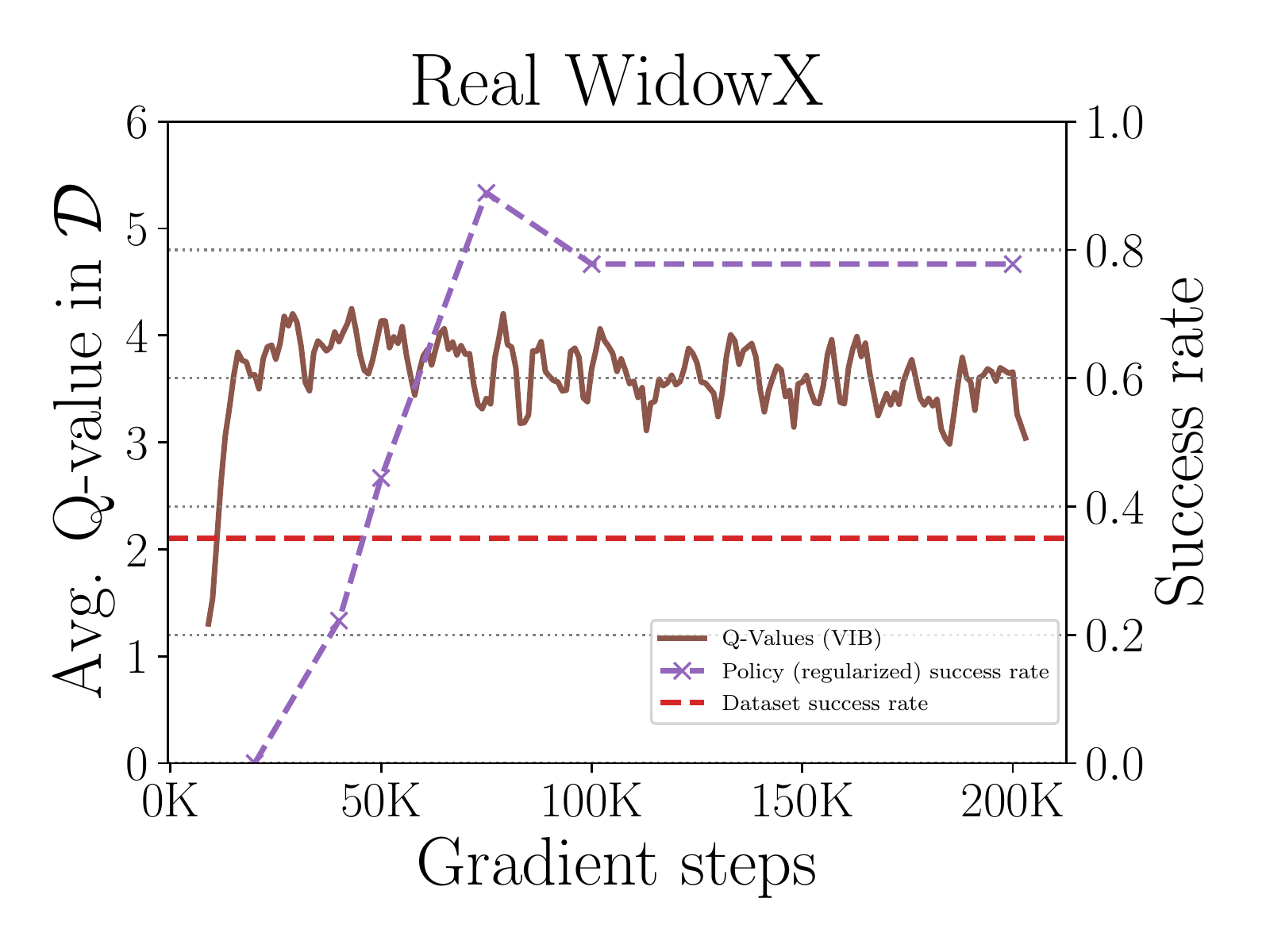} 
\includegraphics[width=0.32\linewidth]{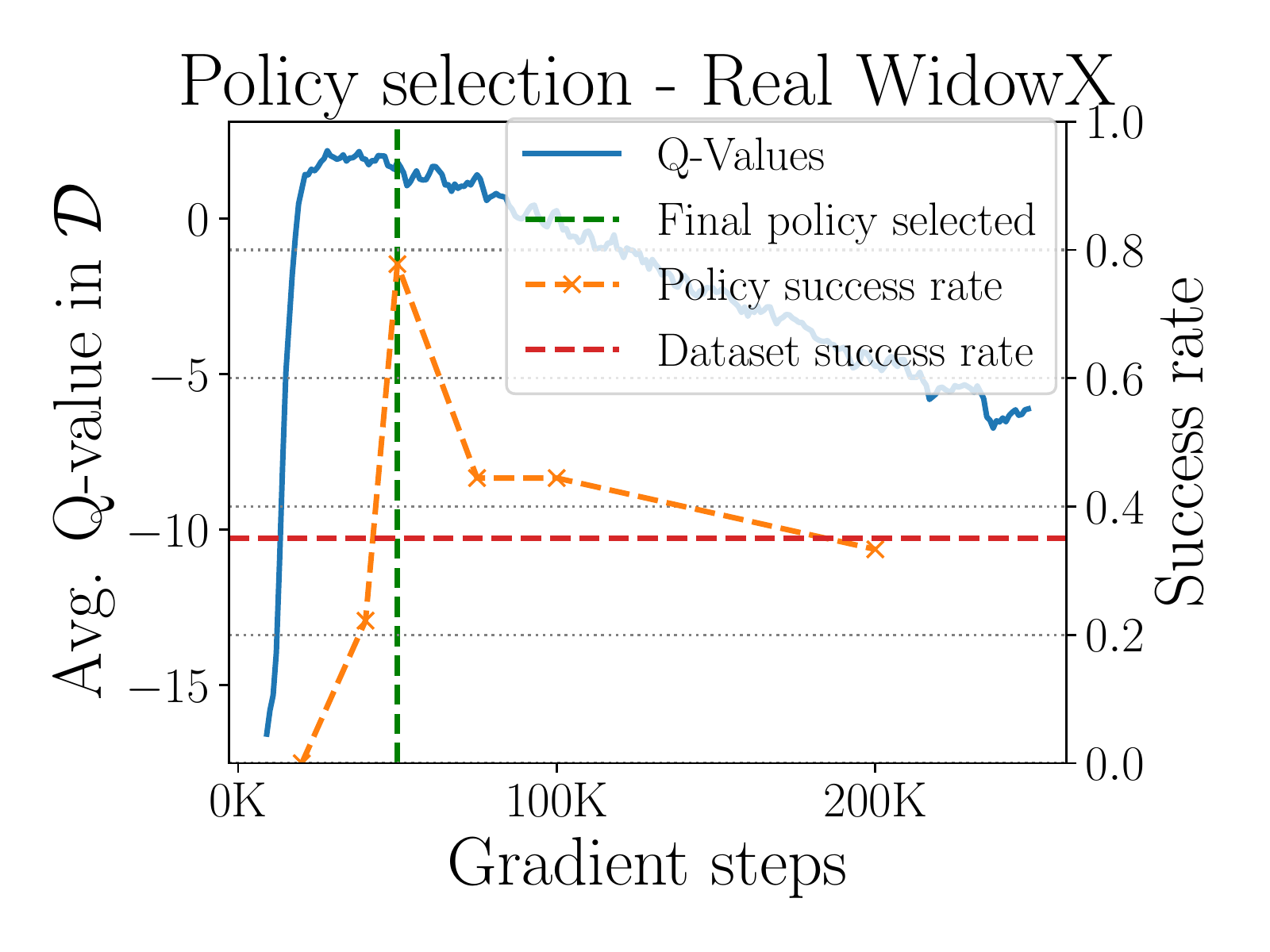}
\vspace{-5pt}
\caption{\label{fig:bottleneck_plots_real} \small{Q-values \textbf{(left)} and performance of CQL with \textbf{(middle)} and without \textbf{(right)} the variational information bottleneck correction for overfitting on the real-world widowX pick and place task. Since the Q-values start to decrease with more training, our workflow detects that CQL is overfitting. Using our policy selection guideline (Guideline~\ref{guideline:policy checkpoint}) enables us to choose checkpoint 50 marked with the green vertical dashed line (right) which performs well. Further, addressing overfitting by applying  the VIB regularizer stabilizes the Q-values (brown) which do not decrease unlike base CQL (blue) (left). Finally, applying the VIB regularizer improves performance and reduces sensitivity to policy selection (middle).}}
\vspace{-15pt}
\end{figure}
Since overfitting is detected, we now turn to addressing overfitting by adding the variational information bottleneck regularizer (Equation~\ref{eqn:bottleneck}) during training. As shown in Figure~\ref{fig:bottleneck_plots_real} (left),the Q-values obtained after the addition of this regularizer (shown in brown; labeled ``Q-values (VIB)'') are now stable -- they increase and stabilize, and do not decrease over the course of training. In this case, our proposed workflow would suggest evaluating any policy checkpoint that attains this stable value. We evaluate four of these policies for visualization purposes in Figure~\ref{fig:bottleneck_plots_real} (middle) and observe that all of these policies attain either a \textbf{7/9} or \textbf{8/9} success rate, comparable or better than the base CQL algorithm shown in Figure~\ref{fig:bottleneck_plots_real} (right). This indicates that addressing overfitting not only leads to some gains in performance but also greatly simplifies policy selection as all checkpoints perform similarly and attain good performance. We summarize these results in Table~\ref{table:real_widowx_results} below, where the bold entries denote the checkpoints found by our policy selection rule.

{\footnotesize
\begin{table}[h]
\begin{center}
    \begin{tabular}{ |p{4cm}|p{1.5cm}|p{1.5cm}|p{1.5cm}|p{1.5cm}|  }
         \hline
         \multicolumn{5}{|c|}{\textbf{Real-World WidowX Pick and Place: Correcting Overfitting}} \\
         \hline
         \textbf{Method} & Epoch 50 & Epoch 75 & Epoch 100 & Epoch 200\\
         \hline
         CQL & \textbf{7/9} & 4/9 & 4/9 & 2/9 \\
         \hline
         CQL + VIB & 3/9 & \textbf{8/9} & \textbf{7/9} & \textbf{7/9} \\
         \hline
    \end{tabular}
    \vspace{0.1cm}
    \caption{\label{table:real_widowx_results} \small{Performance of various policy checkpoints of CQL and CQL + VIB on the real WidowX pick and place task. Note that when overfitting is corrected via VIB, multiple policy checkpoints perform well.}}
    \vspace{-0.7cm}
\end{center}
\end{table}
}

These results indicate the effectiveness of our workflow in tuning CQL by addressing overfitting and underfitting on multiple real robot platforms.

%% file: appendix.tex
\part*{Appendices}

\begin{figure}
    \centering
    \begin{subfigure}[m]{0.47\textwidth}
    \includegraphics[width=\linewidth]{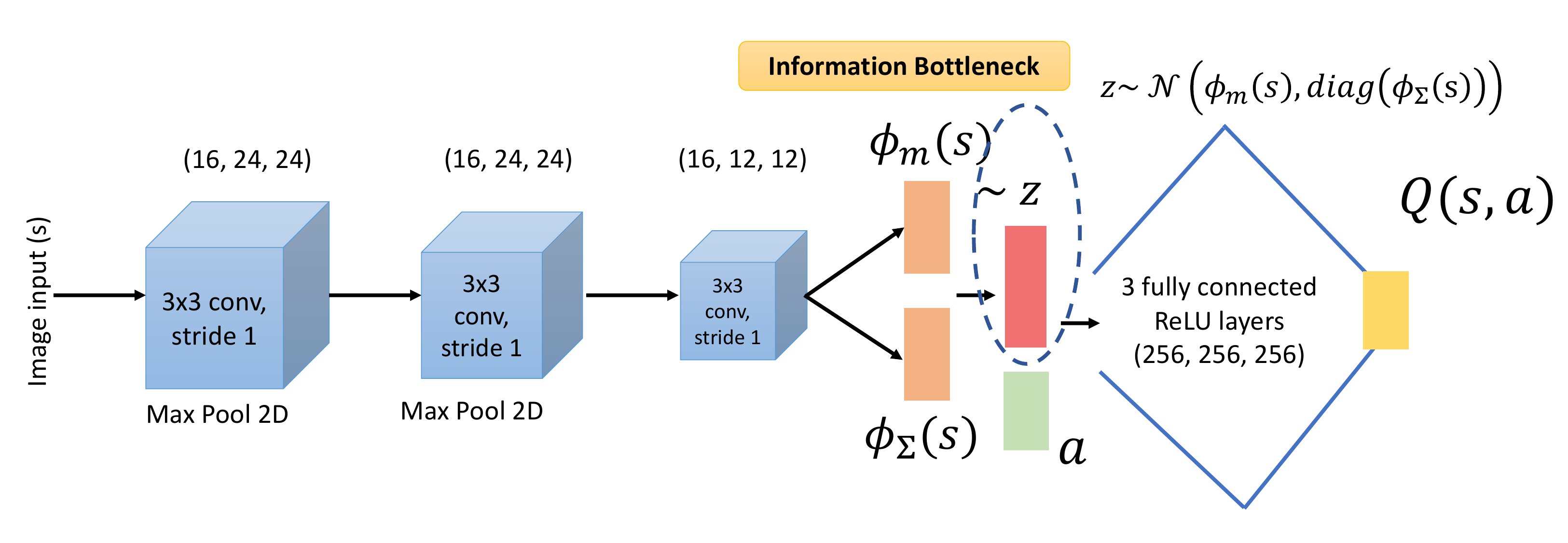}
    \end{subfigure}
    \hfill
    \begin{subfigure}[m]{0.47\textwidth}
        \includegraphics[width=\textwidth]{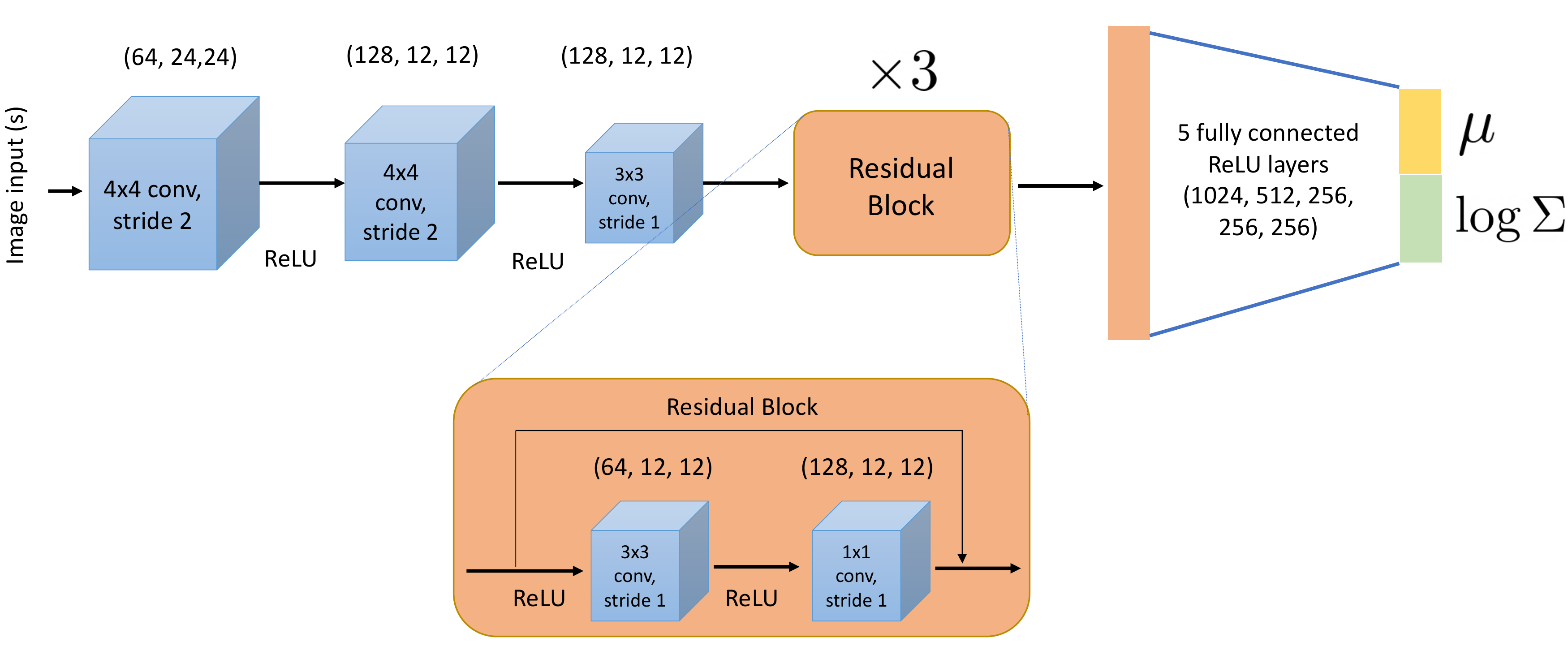}
     \end{subfigure}
    \caption{Architectures for addressing \figleft \; underfitting, using an information bottleneck regularizer, and \figright \; overfitting, using a residual network. 
        \label{fig:architectures}}
\end{figure}

\begin{figure}
\centering
\includegraphics[width=0.8\linewidth]{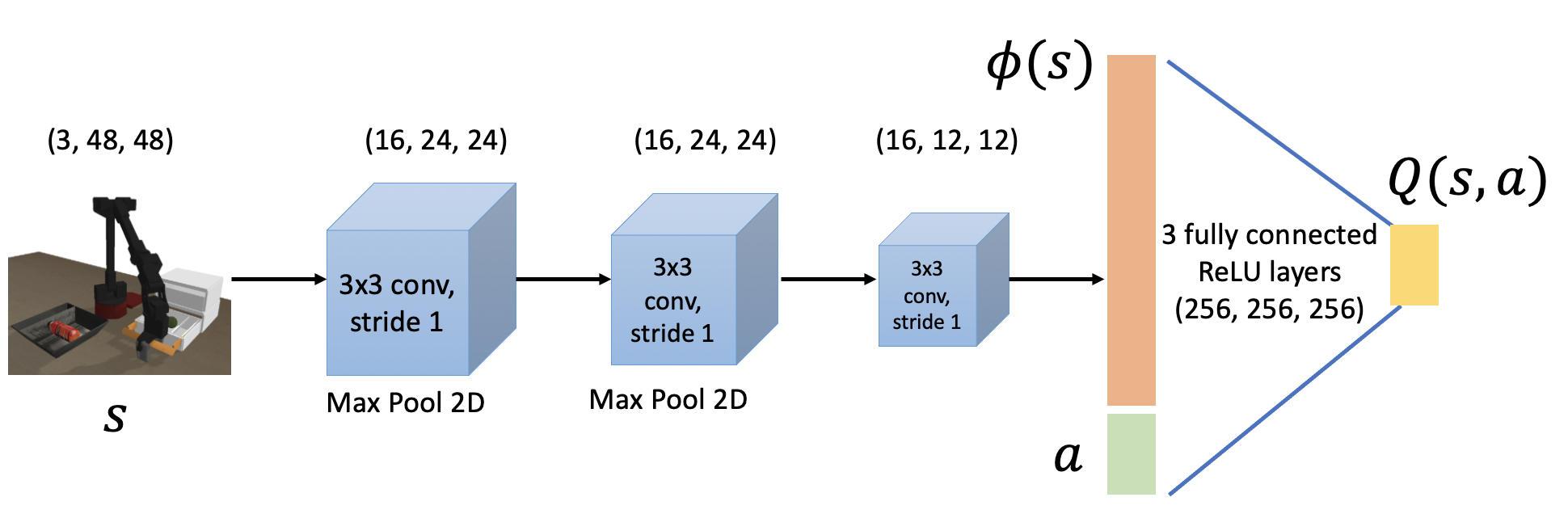}
\caption{\label{fig:standard} Standard architecture for the Q-function used in our experiments. We build on the code provided in \citet{singh2020cog} and utilize their default architecture.}
\end{figure}

\section{Additional Discussion of Overfitting and Underfitting}
\label{app:overfitting_and_underfitting}

In this section, we shall discuss additional details pertaining to various metrics and protocols for detecting overfitting and underfitting discussed in Section~\ref{sec:workflow_metrics}. We first formalize our insight as to why decreasing Q-values as a result of more gradient steps are indicative of overfitting in offline RL and then provide additional discussion about underfitting.

\subsection{Overfitting in CQL}
\label{sec:overfitting_more}
Our proposed workflow in Section~\ref{sec:workflow_metrics} characterizes overfitting in CQL as a non-monotonic, first increasing and then decreasing trend in the average dataset Q-value. To understand why this trend can be used to characterize overfitting, i.e., a reduction in the test objective $J(\pi)$ (the actual return of the learned policy) as per our definition in Section~\ref{sec:background}, Table~\ref{tab:summary}, we first characterize conditions under which CQL (Equation~\ref{eqn:cql_training}) Q-values averaged over the samples in the training dataset cannot exhibit a decreasing trend with more iterations of of training. To derive these conditions, we operate in the regime where the policy $\pi_\phi$ is trained to exactly maximize the Q-value, $\E_{\bs \sim \data, \ba \sim \pi_\phi(\cdot|\bs)}[Q_\theta(\bs, \ba)]$.

\textbf{Notation and Assumptions.} In order to understand the trend in the average dataset Q-value observed in our experiments, we consider a slightly modified variant of Equation~\ref{eqn:cql_training} marked with indices that denote the iteration of learning $k=1, 2, \cdots$:

\begin{equation}
\label{eqn:modified_cql}
    Q^{k+1} \leftarrow \arg\min_Q~ \alpha \left(\E_{\bs, \ba \sim \data, \mu}[Q(\bs, \ba)] - \E_{\bs, \ba \sim \data}[Q(\bs, \ba)]\right) + \frac{1}{2} \E_{\bs, \ba, \bs' \sim \data}\left[ (Q - \bellman^{\pi^k} Q^k)^2 \right],
\end{equation}
where $\pi^k$ is the policy that maximizes the current Q-function, $Q^k$. Thus, variant of CQL shown in Equation~\ref{eqn:modified_cql} implements \textit{exact} policy optimization at each step of training $k$: $\forall \bs, \ba, ~~\pi^k(\ba|\bs) = \mathbb{\delta}[\ba = \arg\max_{\ba'} Q^k(\bs, \ba')]$. Arguably, this is closer to how CQL (and other actor-critic algorithms) are performed in practice -- rather than performing a complete evaluation of the learned policy and only then performing policy improvement, these practical approaches perform alternating iterations of policy evaluation and improvement. The Q-learning variant of CQL~\citep{kumar2020conservative} actually performs exact policy improvement for each step, which is exactly what is shown in Equation~\ref{eqn:modified_cql}. As a result, we analyze Equation~\ref{eqn:modified_cql}. Our goal will be to understand the conditions under which the learned Q-values, averaged over the dataset, can exhibit a decreasing trend with more training.    

\begin{theorem}[Characterizing a decreasing trend in Q-values]
When running CQL using updates in Equation~\ref{eqn:modified_cql} in a tabular setting, using the policy $\pi^k(\ba|\bs) = \mathbb{\delta}[\ba = \arg\max_{\ba'} Q^k(\bs, \ba')]$, the expected Q-value on the dataset, i.e., $f(k) := \E_{\bs, \ba \sim \mathcal{D}}[Q^k(\bs, \ba)]$, is a non-decreasing function of iteration $k$, i.e., $f(k+1) \geq f(k)$, whenever either of the two conditions hold:
\begin{enumerate}
    \item The learned average dataset Q-value is smaller than the Monte-Carlo return of the dataset: $f(k) \leq \frac{1}{1 - \gamma} \E_{\bs, \ba \sim \data}[r(s, a)]$ ~~~~ (expected dataset return), or,
    \item The gap between the maximal value of the learned Q-function $\max_{\ba} Q^k(\bs, \ba)$ and the Q-function value at a different action $\ba'$ at a given state $\bs$ in expectation over all dataset states is large enough, i.e., $\E_{\bs \sim \data}[\max_{\ba} Q^k(\bs, \ba) - Q^k(\bs, \ba')] \geq \zeta$, where $\zeta$ depends inversely on the density of the action $\arg \max_{\ba} Q^k(\bs, \ba)$ under the behavior policy $\pi_\beta(\cdot|\bs)$.
\end{enumerate}
\label{thm:cql_no_reduce}
\end{theorem}
\begin{proof}
To prove this result, we build on the analysis in \citet{kumar2020conservative} and find that the Q-function at iteration $k+1$ can be written as follows:
\begin{equation*}
    Q^{k+1}(\bs, \ba) := \left(\bellman^{\pi_k} Q^k\right)(\bs, \ba) - \alpha \left(\frac{\mu(\ba|\bs)}{\pi_\beta(\ba|\bs)} - 1 \right),
\end{equation*}
where $\pi_\beta(\ba|\bs)$ denotes the behavior policy action-conditioned on state marginal in the dataset $\mathcal{D}$. The average Q-value in the dataset is thus given by:
\begin{align*}
    f(k+1) &= \E_{\bs, \ba \sim \data}[Q^{k+1}(\bs, \ba)]\\
    &= \E_{\bs, \ba \sim \data}\left[\left(\bellman^{\pi^k} Q^k\right)(\bs, \ba)\right] - \alpha \E_{\bs \sim \data, \ba \sim \pi_\beta(\ba|\bs)}\left[ \frac{\mu(\ba|\bs)}{\pi_\beta(\ba|\bs)} - 1 \right]\\
    &= \E_{\bs, \ba \sim \data} \left[r(\bs, \ba) + \gamma \E_{\bs' \sim P(\bs'|\bs, \ba)}[\max_{\ba'} Q^k(\bs', \ba')] \right] - 0.
\end{align*}
Now, we consider the behavior of the above quantity $f(k)$, when the reward function $r(\bs, \ba) \geq 0$, and in particular, for our domains of interest, $\forall \bs, \ba, ~r(\bs, \ba) = 0$ or $r(\bs, \ba) = 1$. When the initial value function $\forall \bs, \ba,~Q^0(\bs, \ba) = 0$, we now wish to characterize conditions under which the Q-function iterates $Q^1, \cdots, Q^k, \cdots$ are monotonically increasing in expectation, i.e.,
\begin{equation*}
    \forall \bs, \ba, ~~\E_{\bs, \ba \sim \data}[Q^1(\bs, \ba)] \leq \E_{\bs, \ba \sim \data}[Q^2(\bs, \ba)] \leq \cdots \leq \E_{\bs, \ba \sim \data}[Q^k(\bs, \ba)] \leq \cdots.
\end{equation*}
We can analyze this progression using mathematical induction. To first prove the base case, note that since $Q^0(\bs, \ba) = 0$ (initialization), $f(1) = \E_{\bs, \ba \sim \data}[r(\bs, \ba)]$, and 
\begin{equation*}
f(2) = \E_{\bs, \ba \sim \data}[r(\bs, \ba) + \gamma \max_{\ba'} Q^1(\bs', \ba')] \geq \underbrace{\E_{\bs,\ba \sim \data}[r(\bs, \ba)]}_{=f(1)} + \gamma \underbrace{\E_{\bs', \ba' \sim \data P^{\pi_\beta}} [Q^1(\bs', \ba')]}_{\geq 0},    
\end{equation*}
since in expectation, $\E_{\bs', \ba' \sim \data P^{\pi_\beta}}[Q^1(\bs', \ba')] = \E_{\bs, \ba \sim \data}[r(\bs, \ba)]$, where $(\bs', \ba') \sim \data P^{\pi_\beta} = \mathcal{D}$, since the dataset distribution is the stationary state-action visitation distribution of the behavior policy $\pi_\beta$. Thus, we find that $f(2) \geq f(1)$, proving the base case for induction,

Now we assume that upto a given k, $\forall j \in [k], f(j) \geq f(j-1)$. Then, our aim is to derive the condition that $f(k+1) \geq f(k)$. To show this, we write out the expressions:
\begin{align}
    \nonumber
    f(k+1) &= \E_{\bs, \ba \sim \data}\left[r(\bs, \ba) + \gamma \E_{\bs'} \left[ \max_{\ba'} Q^k(\bs', \ba')\right] \right]\\
    \label{eqn:expressions}
    f(k) &= \E_{\bs, \ba \sim \data}\left[r(\bs, \ba) + \gamma \E_{\bs'} \left[ \max_{\ba'} Q^{k-1}(\bs', \ba') \right] \right],
\end{align}
and then expand $f(k+1) - f(k)$:
\begin{align*}
    f(k+1) - f(k) &= \E_{\bs, \ba \sim \data}[r(\bs, \ba)] + \gamma \E_{\bs'}[\max_{\ba'} Q^k(\bs, \ba)] - \E_{\bs, \ba \sim \data}[Q^k(\bs, \ba)]\\
    &= \underbrace{\E_{\bs, \ba \sim \data}[r(\bs, \ba)] - (1 - \gamma) f(k)}_{\text{(a)}} + \gamma \underbrace{\E_{\bs \sim \data}\left[ \max_{\ba} Q^k(\bs, \ba) - \E_{\pi_\beta}[Q^k(\bs, \ba)] \right]}_{\text{(b)}}.
\end{align*}

First, by definition note that $(b) \geq 0$. And term $(a) \geq 0$ for iterations $k$ where $f(k) \leq \frac{\E_{\bs, \ba \sim \data}[r(\bs, \ba)]}{1 - \gamma}$, which occurs whenever the average dataset Q-value, $f(k)$ is smaller than the dataset discounted cumulative reward. Thus, whenever the dataset Q-value is smaller than the average dataset discounted cumulative reward, $(a) \geq 0$, $(b) \geq 0$ implying that $f(k+1) \geq f(k)$.

Now, let's consider the case when the average dataset Q-value is smaller than the expected cumulative reward in the dataset and characterize the conditions under which the Q-values will exhibit a non-decreasing trend in this case. To characterize this condition, we begin with a direct difference of the expressions for $f(k)$ and $f(k-1)$ in Equation~\ref{eqn:expressions} and analyze the difference in Q-values from consecutive Q-function iterates at arg-max actions $\ba'$ at the next state $\bs'$. For all iterations $k$, where $(a) \leq 0$, i.e., the average dataset Q-value is higher than the dataset discounted cumulative reward, consider the expressions for $f(k+1)$ and $f(k)$ from Equation~\ref{eqn:expressions} again, and note that there are two cases for each state $\bs'$ appearing in the RHS of the expressions:

\textbf{Case 1: $\arg \max_{\ba'} Q^k(\bs', \ba') = \arg \max_{\ba'} Q^{k-1}(\bs', \ba')$:} In this case, using the expression for the Q-function obtained in CQL, we can express $Q^k$ as:
\begin{equation*}
Q^{k}(\bs', \ba') - Q^{k-1}(\bs', \ba') = \gamma \E_{\bs''}\left[ \max_{\ba''} Q^{k-1}(\bs'', \ba'') - \max_{\ba''} Q^{k-2}(\bs'', \ba'') \right],    
\end{equation*}
which is a similar expression to what already exists in an expansion of $f(k) - f(k-1)$ analogous to Equation~\ref{eqn:expressions}.

\textbf{Case 2: $\arg \max_{\ba'} Q^k(\bs', \ba') \neq \arg \max_{\ba'} Q^{k-1}(\bs', \ba')$:} Let $\ba_1 = \arg \max_{\ba'} Q^k(\bs', \ba')$ and let $\ba_2 = \arg \max_{\ba'} Q^{k-1}(\bs', \ba')$. Then, we can split their difference as:
\begin{equation*}
    Q^{k}(\bs', \ba_1) - Q^{k-1}(\bs', \ba_2) = \underbrace{Q^k(\bs', \ba_2) - Q^{k-1}(\bs', \ba_2)}_{(i)} + \underbrace{Q^k(\bs', \ba_1) - Q^k(\bs', \ba_2)}_{(ii)}.
\end{equation*}
Term $(ii)$ in the above expression is non-negative, since $\ba_1$ is the action with the highest Q-value $Q^k$ at state $\bs'$. Term $(i)$ in the above expression can be split further:
\begin{equation*}
    Q^k(\bs', \ba_2) - Q^{k-1}(\bs', \ba_2) := - \frac{\alpha}{\pi_\beta(\ba_2|\bs')} + \gamma \E_{\bs''}\left[ \max_{\ba''} Q^{k-1}(\bs'', \ba'') - \max_{\ba''} Q^{k-2}(\bs'', \ba'') \right].
\end{equation*}

The second term in the above expression is similar to the term in $f(k) - f(k-1)$, and thus if the offset $-\frac{\alpha}{\pi_\beta(\ba_2|\bs')}$ does not fully compensate 
for the increase due to term $(ii)$, by induction we can claim that $f(k+1) \geq f(k)$ if the inequality holds for all $j \leq k$.

\textbf{To summarize,} we can group the two cases to list down conditions under which the learned average dataset Q-value \emph{can} decrease in a given iteration $k$ of CQL. This means that it is not necessary that the Q-values would decrease when these conditions are met, but if these conditions are not met, then the Q-values cannot necessarily decrease with more training. For a given iteration $k$ of CQL, the average Q-value under the dataset can decrease when:

\begin{equation}
    \label{eqn:when_q_vals_decrease}
    f(k) \geq \frac{\E_{\bs, \ba \sim \data}[r(\bs, \ba)]}{1 - \gamma} ~~\text{and}~~~ \E_{\bs' \sim \data}\left[\max_{\ba} Q^k(\bs', \ba) - Q^k(\bs', \ba_2) \right] \leq \underbrace{\E_{\bs'}\left[\frac{\alpha}{\pi_\beta(\ba_2|\bs')}\right]}_{\zeta}.
\end{equation}
Thus, if the difference of Q-values across different actions in expectation over all states in the dataset is large enough, the condition in Equation~\ref{eqn:when_q_vals_decrease} is not met and we would expect Q-values to increase, and not decrease. Similarly, in the phase of learning where the Q-value is smaller than the average dataset return, we would expect the Q-values to continue increasing. Thus, the average dataset Q-value should be non-decreasing if either of the two conditions in Equation~\ref{eqn:when_q_vals_decrease} are not satisfied, which corresponds to conditions (1) and (2) in the theorem statement.
\end{proof}

\textbf{Interpretation of Theorem~\ref{thm:cql_no_reduce}: Early stopping and the peak in Q-values.} Now we shall deduce the conclusion of overfitting from Theorem~\ref{thm:cql_no_reduce}. The Q-values decrease only if the gap between Q-values at actions taken by two consecutive policy iterates is smaller than a quantity $\zeta$ that depends inversely on the likelihood of the action under the behavior policy. This means that if and once the Q-function finds a good policy $\pi$, better than the behavior policy $\pi_\beta$, the average dataset Q-values can start to decrease if $\pi$ is not close enough to $\pi_\beta$, since the actions from the learned policy $\pi(\ba'|\bs')$ will not have a high likelihood under the behavior policy $\pi_\beta(\cdot|\bs')$, and thus the $\zeta$ term in Equation~\ref{eqn:when_q_vals_decrease} can easily become larger than the gap between Q-values. Thus, we would expect that the peak in the Q-values would correspond to this a performing policy $\pi$, that is potentially different from the behavior policy. One would also expect that a decrease in the Q-function would cause the learned policy $\pi$ to gradually move towards the behavior policy as this would increase $\pi_\beta(\ba_2|\bs')$ by selecting action $\ba_2$ highly likely under the behavior policy and would thus reduce $\zeta$. On the other hand, if the Q-values continuously increase, the learned Q-values are either smaller than the dataset Monte-Carlo return or exhibit high gaps between Q-values. In such scenarios, we would expect more gradient steps of policy evaluation and improvement to actually improve the policy, and more training would lead to improved performance. Thus, this discussion implies that a non-monotonic trend in Q-values is indicative of overfitting towards the behavior policy (Metric~\ref{guideline:overfitting}) and that policy selection can be performed near the peak of the Q-values (Guideline~\ref{guideline:policy checkpoint}). 

{ \color{blue}

}

\subsection{Underfitting in CQL}
\label{app:underfitting}
The metric used to characterize underfitting in Section~\ref{sec:workflow_metrics} is to compute the value of TD error, $\mathcal{L}_\text{TD}(\theta)$ and the CQL regularizer, $\mathcal{R}(\theta)$ and inspect if these values are large either relative to a model with an increased capacity or on an absolute scale. To understand why this corresponds to underfitting, note that a large value of TD error corresponds to a Q-function that does not respect Bellman consistency conditions and hence may be arbitrarily worse, whereas a large positive value of the CQL regularizer corresponds to a Q-function that is not close to the behavior policy and hence may be choosing out-of-distribution actions. In either case, we would aim to learn a Q-function that minimizes both the TD-error and the CQL regularizer. 

Minimizing only one of the two objectives is not sufficient in this setting: \textbf{(1)} a Q-function that minimizes training TD error to a small enough value but attains a large value of the CQL regularizer is not sufficient since this Q-function may take erroneously high values on out-of-distribution actions, leading to a worse policy, and, \textbf{(2)} a Q-function that minimizes the CQL regularizer to a small value and attains a high value of the training TD error may not correspond to a valid Q-function which may lead to a worse policy, potentially close to the behavior policy. As a result, our Metric~\ref{guideline:underfitting} suggests tracking both of these values independently and utilizing a correction for underfitting if either of the two objectives (TD error and CQL regularizer) are not minimized to low-enough values.

\textbf{Utilizing a fix for underfitting by default in CQL.} Similar to supervised learning, precisely quantifying the amount of underfitting is hard in offline RL as well. It is an additional challenge in offline RL that the two objectives (TD error and CQL regularizer) may impose conflicting gradients, making it hard to identify the optimal value of these loss values. As a result, we would suggest that some of the proposed solutions for underfitting discussed in Section~\ref{sec:addressing_workflow} such as utilizing more expressive architectures be used even in cases where it is ambiguous as to whether the loss values are large or not, provided that there are no clear signs of overfitting (per Metric~\ref{guideline:overfitting}).  

\section{Additional Background}
\label{app:additional_background}

In this section, we provide additional background for the conservative Q-learning (CQL)~\citep{kumar2020conservative} algorithm that we use as the base algorithm for devising our workflow. We utilize the actor-critic instantiation of CQL that trains a conservative Q-function $Q_\theta(\bs, \ba)$ and a policy $\pi_\phi(\ba|\bs)$ that maximizes the Q-function. This algorithm proceeds in alternating steps of policy evaluation and policy improvement and our practical instantiation of this algorithm operates as per the following (policy evaluation and policy improvement) updates:
\begin{small}
\begin{align*}
    \theta^{k+1} \leftarrow& \small{\arg\min_{\theta}~ \alpha \mathbb{E}_{\bs \sim \data}\left[\log \sum_{\ba} \exp(Q_\theta(\bs, \ba))-\E_{\ba \sim \data}\left[Q_\theta(\bs, \ba)\right]\right] + \frac{1}{2} \E_{\bs, \ba, \bs' \sim \data}\left[\left(Q - {\bellman}^{\policy_k} \bar{Q} \right)^2 \right]\!}\\ 
    \phi^{k+1} \leftarrow& \arg\max_{\phi} \E_{\bs \sim \data, \ba \sim \policy^{k}_{\phi}(\ba|\bs)}\left[\hat{Q}_\theta^{k+1}(\bs, \ba)\right]~~~ \text{(policy improvement)} 
\end{align*}
\end{small}
In practice, these updates are performed via alternating gradient descent on the actor ($\pi_\phi(\ba|\bs)$) and the critic ($Q_\theta(\bs, \ba)$). While the hyperparameter $\alpha$ in the update above also needs to be chosen offline, we utilize the value of $\alpha=1.0$ from prior work, fixed across all our experiments in both simulated and real-world domains, and focus on tuning other decisions such as network size, regularization and policy selection.

\section{Additional Experimental Details}
\label{app:setup_details}

\subsection{Simulated Domains}
\label{sec:simulation}

In this section, we provide a detailed discussion of the domains used in our simulated experiments in Section~\ref{sec:workflow_exps}.  

\textbf{Pick and place task.}
As detailed in Section~\ref{sec:workflow_exps}, our first simulated domain consists of a 6-DoF WidowX robot in front of a tray containing a small object and a tray. The objective is to put the object inside the tray. The reward is +1 when the object has been placed in the tray, and zero otherwise. The offline dataset consists of trajectories that grasp the object with a 35\% success rate and place it with a success rate of 40\%. We collected the dataset using scripted policies that we briefly discuss below. For more detail, please refer to Appendix A.1 in \citet{singh2020cog}.

\textit{Scripted grasping policy.} Our scripted policy is identical to the policy in \citet{singh2020cog}. This policy is supplied with the object's (approximate) coordinates and can localize the object using background subtraction. Once the policy localizes the objects, it goes to the objects by executing actions with added noise and then closes the gripper when it is within some pre-specified distance of the object. This distance threshold is randomized similar to \citet{singh2020cog} and the grasp can fail or succeed with about a 35\% chances of success. 

\textit{Scripted pick and place policy.} As previously used in \citet{singh2020cog}, our scripted pick and place policy attempts a grasp as described above, and then tries to place the object randomly at some location in the workspace. Only if it places the object on the tray does it get a +1 reward, and after placing the object, it moves up and tries to hover around by executing small magnitude random actions until the episode terminates.

\textbf{Grasping from a blocked drawer.} 
The scripted policies we use for this task are borrowed from \citet{singh2020cog}. These policies can open and close both the drawers with 40-50\% success rates, can grasp objects with about a 70\% success rate, and place those objects at random locations in the workspace.
Since we use the datasets from \citet{singh2020cog} directly, the prior data does \emph{not} contain any interactions with the object inside the drawer and contains data such as behavior that blocks the drawer by placing objects in front of it.

\textit{Scripted drawer opening and closing.} Our scripted policy for drawer opening and closing moves the gripper to the handle, then pulls or pushes it to open/close the drawer. At each step, Gaussian noise is added to the data collection and it does not succeed 70\% of the times.

Pesudocode and more details of these policies, which are directly used from prior work~\citep{singh2020cog} is provided in Algorithms 1-3 of \citet{singh2020cog}. 

\subsection{Real-World Domains}
\label{sec:real_world}
\textbf{Sawyer tasks.} As detailed in Section~\ref{sec:real_world}, the dataset used for our Sawyer tasks is the same as \citet{khazatsky2021can}. We emphasize that we directly utilize the previously collected datasets from \citet{khazatsky2021can} to mimic the real-world use case of offline RL, where we are supposed to learn effective policies from a previously collected dataset. The datasets for each of two tasks (put lid on pot, open a drawer) consist of 100 trajectories which were collected using a 3Dconnexion SpaceMouse device. Each trajectory in both the datasets is of length 80, which is also the number of time steps provided to the learned policy for solving the task. We then label the trajectories using 0-1 reward indicating a success when the task is complete (i.e., the lid is on the pot, and the drawer is sufficiently open). We present some examples of trajectories in the dataset on the associated supplementary website \url{https://sites.google.com/view/offline-rl-workflow}.

\textbf{Real WidowX Pick and Place task.} We collect data for this task by utilizing a scripted policy that first localizes the object, then reaches for this object using noisy actions and then attempts a grasp (with added noise) and places the object on the tray imperfectly. The success rate of the policy is 35\% in both the grasping and the placing of the object on the the tray. A reward of +1 was provided when the policy was able to place the object in the tray. Each trajectory in this dataset is of length 15, which is also the time-limit provided to the learned policy for solving the task at evaluation time. We provide videos of sample trajectories in the dataset in the associated supplementary website \url{https://sites.google.com/view/offline-rl-workflow}.

\section{Detailed Empirical Results}
\label{app:more_results}

In this section, we provide additional empirical results for various components of our workflow, including missing evidence from the main paper. 

\subsection{Simulated Domains}
\label{app:sim_studies}

\textbf{Addressing overfitting in the grasping from blocked drawer task in Scenario \#1.} We first discuss the efficacy of the proposed correction for overfitting via the variational information bottleneck regularizer (Equation~\ref{eqn:bottleneck}) on the grasping from blocked drawer task. As shown in Figure~\ref{fig:bottleneck_on_drawer}, utilizing the bottleneck regularizer gives rise to a stable trend in Q-values (Q-values no more decrease with more training) as shown in the blue curve compared to the orange curve for base CQL, and as is evident from the policy performance plot, utilizing our fix for overfitting also leads to higher and stable performance.  

\begin{figure}
\centering
\includegraphics[width=0.3\linewidth]{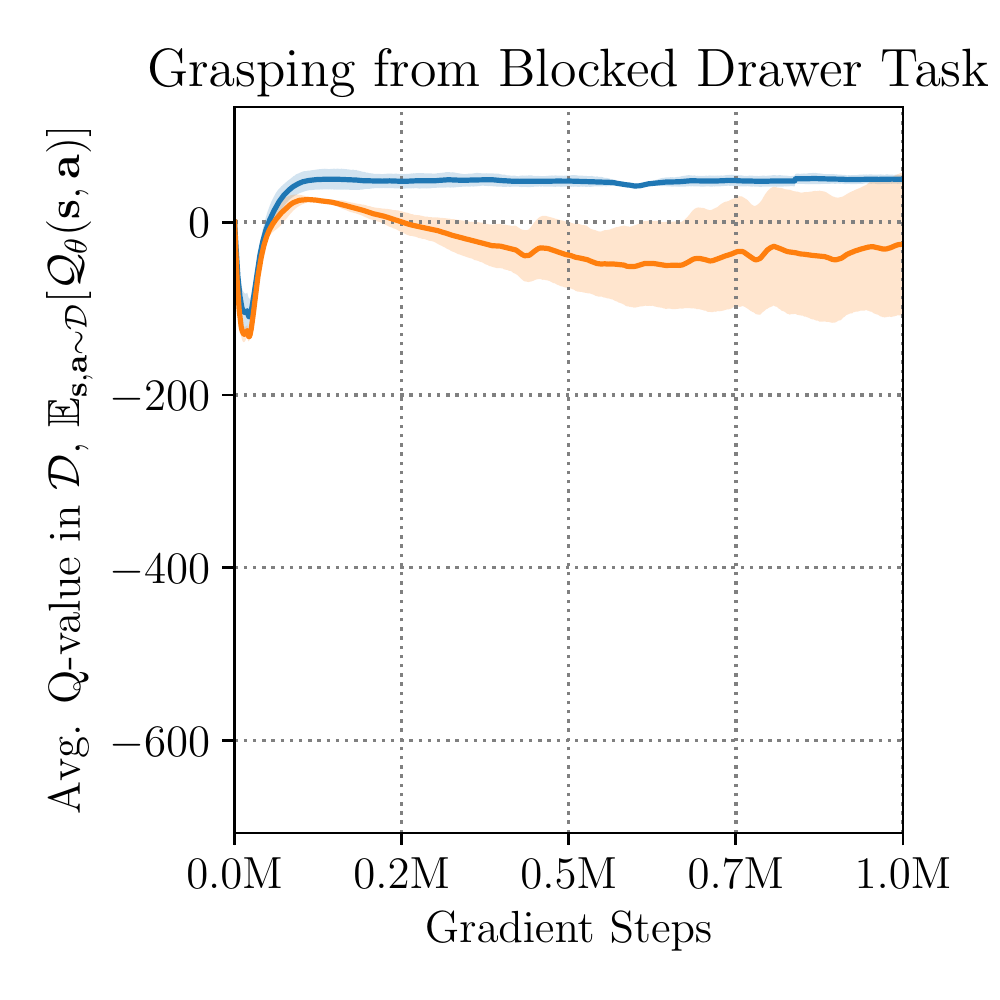}
\includegraphics[width=0.3\linewidth]{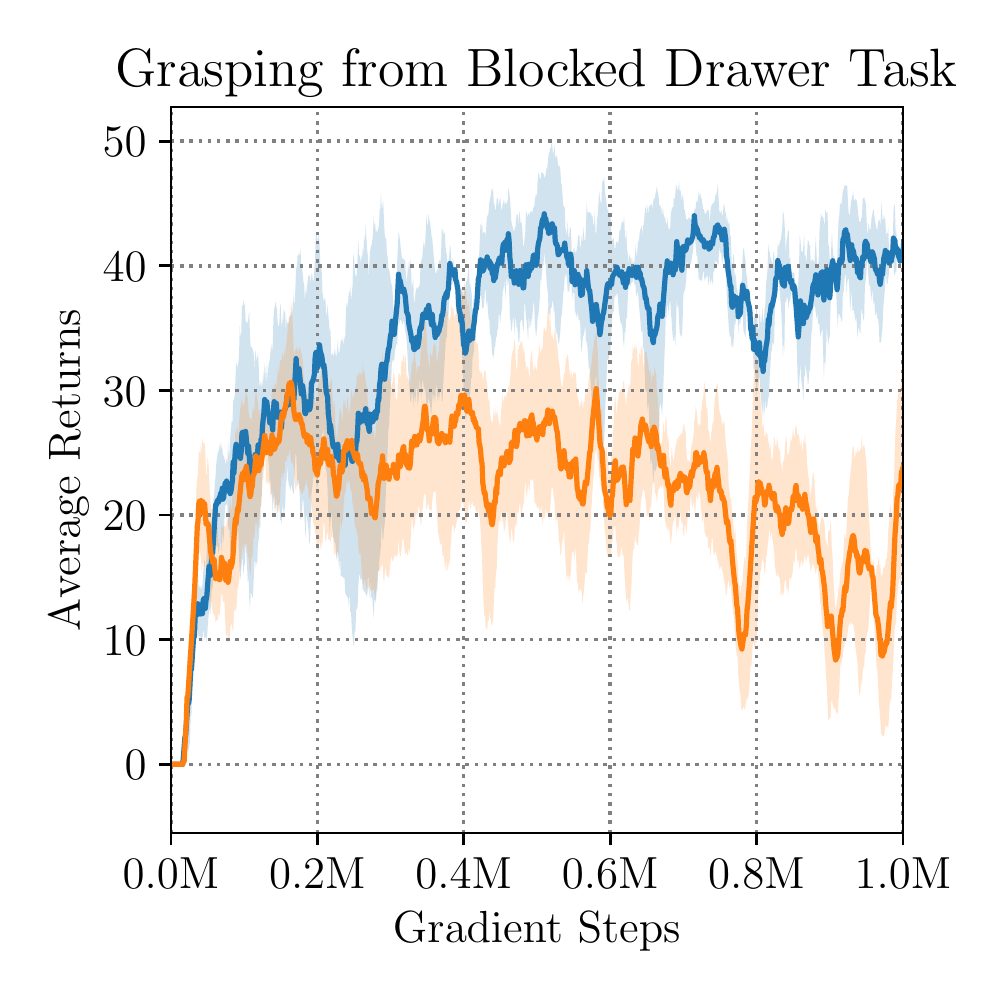}
\caption{\label{fig:bottleneck_on_drawer} \small{Trend in average dataset Q-value (left) and the performance of the policy (right) for base CQL (\textbf{orange}) and base CQL + overfitting correction using VIB (Equation~\ref{eqn:bottleneck}) (\textbf{blue}). Note that using the VIB regularizer addresses overfitting in that the Q-values increase and then stabilize and this stabilization effect is also observed in the performance of the policy, which also increases around two fold.}}
\vspace{-5pt}
\end{figure}

\textbf{Scenario \#2, multiple object pick and place task.} We provide the details (loss curves and Q-value trends) for this task on our anonymous project website linked here: \url{https://sites.google.com/view/offline-rl-workflow}.

\subsection{Real-World Experiments}
\label{app:real_studies}

\textbf{Sawyer tasks.} We present the missing CQL regularizer ($\mathcal{R}(\theta)$) plot for this task from the main text (Section~\ref{sec:real_world_case}) below. Note that even the CQL regularizer eventually increases (dashed lines in the figure below) in the case of the base CQL algorithm that does not utilize a large ResNet architecture. On the other hand, utilizing the ResNet architecture leads to a clearly decreasing trend in the value of the CQL regularizer as is evident below. Thus, utilizing a larger network addresses underfitting. 
\begin{center}
\includegraphics[width=0.3\linewidth]{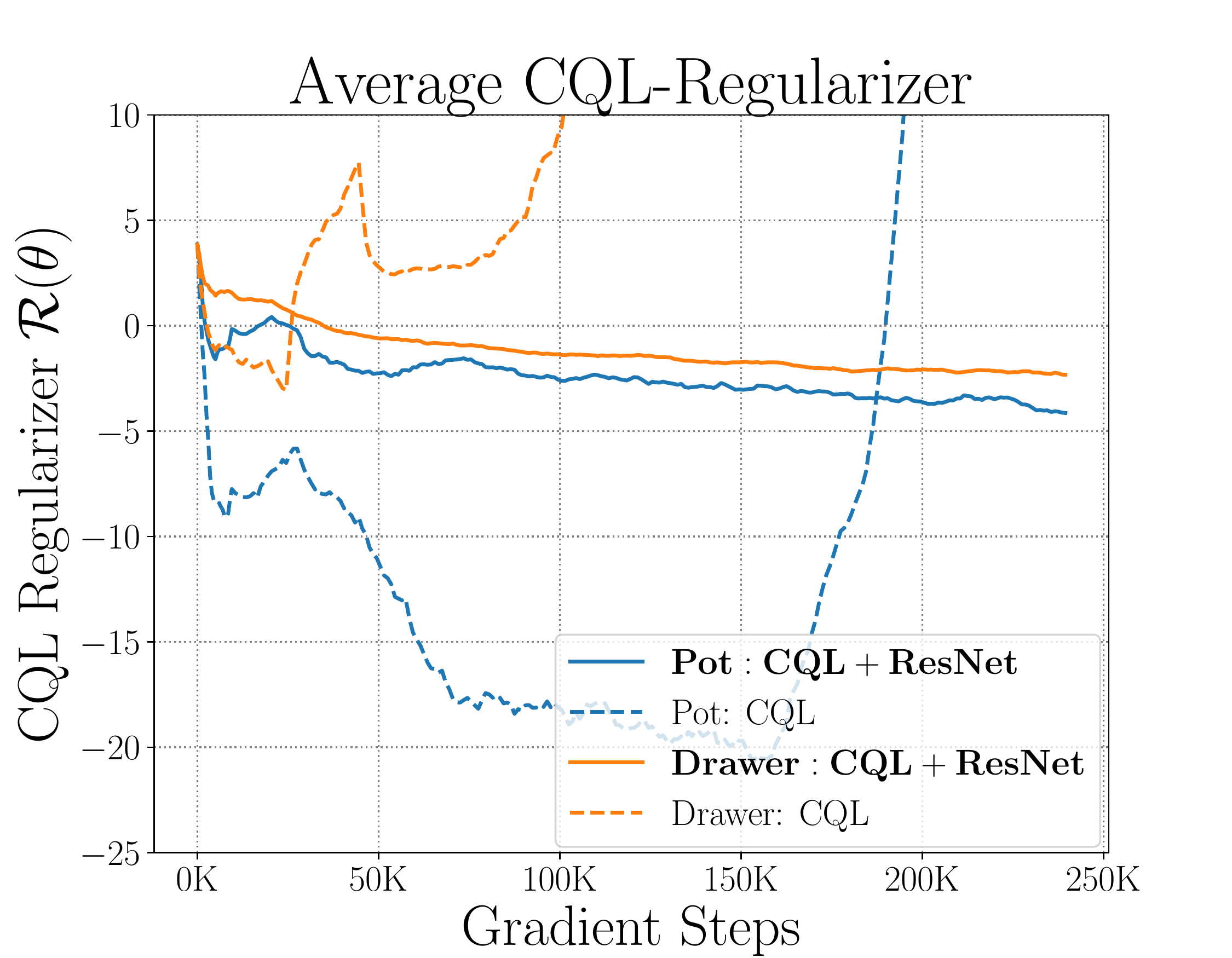}
\end{center}

\textbf{More results and videos for each task can be found on our anonymous website located here: \url{https://sites.google.com/view/offline-rl-workflow}}.

\vspace{-0.1cm}
{ 
\section{Applying Our Workflow to Other Offline RL Algorithms}
\vspace{-0.1cm}
In this section, we discuss how to apply our workflow to other offline RL algorithms beyond CQL. Our workflow is applicable to conservative offline RL algorithms that can be interpreted as abstractly optimizing the objective in Equation~\ref{eqn:generic_offline_rl} in some form. We elaborate on this class of algorithms in the next section (Appendix~\ref{app:which_algos}) and then present, in Appendix~\ref{app:brac_example}, an application of our workflow for detecting and correcting overfitting with BRAC~\citep{wu2019behavior}, a policy-constraint conservative offline RL method. 

\subsection{Which Algorithms Does Our Workflow Apply To?}
\label{app:which_algos}
Our workflow guidelines are intended to be applicable to ``conservative offline RL algorithms'' that can be abstractly represented using the policy optimization objective shown in Equation~\ref{eqn:generic_offline_rl}, which is restated below for convenience of the reader:

\begin{equation}
\label{eqn:generic_offline_rl_repeated}
    \pi^* := \arg \max_{\pi}~~ J_{\mathcal{D}}(\pi) - \alpha D(\pi, \pi_\beta)~~~~~~~~~~~ \text{(Conservative offline RL)}.
\end{equation}
$D(\pi, \pi_\beta)$ in Equation~\ref{eqn:generic_offline_rl_repeated} represents the divergence between the learned policy $\pi(\cdot|\bs)$ and the behavior policy $\pi_\beta(\cdot|\bs)$ averaged over the state-visitation distribution of the learned policy $\pi$. This is given by $D(\pi, \pi_\beta) = \E_{\bs, \ba \sim d^\pi_{\mathcal{D}}(\bs) \pi(\ba|\bs)}\left[D(\pi(\cdot|\bs), \pi_\beta(\cdot|\bs))\right]$. Thus, Equation~\ref{eqn:generic_offline_rl_repeated} can be expressed as:
\begin{equation*}
    J_\mathcal{D}(\pi) - \alpha D(\pi, \pi_\beta) = \E_{\bs, \ba \sim d^\pi_{\mathcal{D}}}\left[\underbrace{r(\bs, \ba) - \alpha D\left(\pi(\cdot|\bs), \pi_\beta(\cdot|\bs)\right)}_{\text{effective new reward function}}\right]
\end{equation*}
This can be viewed as solving the RL problem with a modified reward function that penalizes deviation between the learned policy $\pi$ and the behavior policy $\pi_\beta$. Thus, optimizing the policy against Equation~\ref{eqn:generic_offline_rl_repeated} requires utilizing the long-term, cumulative estimate of divergence $D$. 

\textbf{Which algorithms are covered by our definition of conservative offline RL from Equation~\ref{eqn:generic_offline_rl_repeated}?} Two algorithms covered under this definition are BRAC-v~\citep{wu2019behavior} and CQL~\citep{kumar2020conservative}. While CQL applies a Q-function regularizer (Equation~\ref{eqn:cql_training}) to learn a conservative Q-function that directly models the combined effect of environment reward $r(\bs, \ba)$ and divergence from the behavior policy $D(\pi(\cdot|\bs), \pi_\beta(\cdot|\bs))$ in the learned Q-function, BRAC-v instead exploits an explicit policy constraint. We discuss BRAC in detail below and demonstrate how to effectively apply our workflow to tune overfiting in BRAC in the next section. 

\textbf{Details and background on BRAC.} Unlike CQL, BRAC-v explicitly subtracts the divergence $D(\pi(\cdot|\bs'), \pi_\beta(\cdot|\bs'))$ from the target value while performing the Bellman update.  Additionally, since the divergence between the learned policy and the behavior policy \emph{at the current state} is not a part of the Q-function, BRAC-v also explicitly adds the divergence value at the current state to the policy update. We instantiate the version of BRAC that uses the KL-divergence:
\begin{align*}
D(\pi(\cdot|\bs), \pi_\beta(\cdot|\bs)) = D_{\mathrm{KL}}(\pi(\cdot|\bs), \pi_\beta(\cdot|\bs)) = \E_{\ba \sim \pi(\cdot|\bs)}\left[ \log \pi(\ba|\bs) - \log \pi_\beta(\ba|\bs) \right].
\end{align*}
The first term in this divergence $D_\mathrm{KL}$ corresponds to an entropy regularizer on the policy $\pi(\cdot|\bs)$ that standard MaxEnt RL algorithms like Soft Actor-Critic (SAC)~\citep{haarnoja} already apply. To estimate the second term, BRAC estimates a model of the behavior policy, that we denote as $\hat{\pi}_\beta$, and uses it to explicitly compute this divergence. Denoting the policy and the Q-function as $\pi_\phi$ and $Q_\theta$, the BRAC-v training objectives are (practical implementations use different values for $\alpha$ and $\beta$):
\begin{align}
    \text{Q-function:} ~~&~~ \min_\theta~~ \E_{\bs, \ba \sim \mathcal{D}}\left[\left(r(\bs, \ba) + \gamma \E_{\ba' \sim \pi_\phi(\cdot|\bs')}[\bar{Q}_\theta(\bs', \ba') + \textcolor{red}{\beta \log \hat{\pi}_\beta(\ba'|\bs')}] - Q_\theta(\bs, \ba) \right)^2 \right]. \nonumber\\
    \text{Policy:} ~~&~~ \max_\phi~~ \E_{\bs \sim \mathcal{D}, \ba \sim \pi_\phi(\cdot|\bs)}\left[ \underbrace{Q_\theta(\bs, \ba) + \textcolor{red}{\beta \log \hat{\pi}_\beta(\ba|\bs)}}_{\text{conservative Q-value}; Q_c(\bs, \ba)} - \underbrace{\alpha \log \pi_\phi(\ba|\bs)}_{\text{policy entropy; standard MaxEnt RL}} \right].  
    \label{eqn:brac_eqns}
\end{align}
We will refer to the estimate $Q_c(\bs, \ba) := Q_\theta(\bs, \ba) + \beta \log \hat{\pi}_\beta(\ba|\bs)$ as the \emph{conservative Q-value}, that estimates the combined effect of both the reward and the divergence from the behavior policy. $Q_c(\bs, \ba)$ is analogous to the Q-function trained via CQL which directly estimates this combined effect. To note further similarities, observe that CQL optimizes the policy against the conservative Q-value estimate, predicted directly by the Q-network, along with an added entropy regularizer, whereas BRAC uses $Q_c$ (Equation~\ref{eqn:brac_eqns}) in its place. $Q_c$ will play a crucial role in adapting our workflow for overfitting to BRAC which we discuss in the next section. 

\textbf{Which offline RL methods is our workflow not applicable to?} The formulation of conservative offline RL in Equations~\ref{eqn:generic_offline_rl} and \ref{eqn:generic_offline_rl_repeated} does not encompass offline RL algorithms that only utilize a ``myopic'' behavior regularization, such as BCQ~\citep{fujimoto2018off}, BEAR~\citep{kumar2019stabilizing}, AWR~\citep{peng2019advantage}, TD3+BC~\citep{fujimoto2021minimalist}. These methods only apply the behavior constraint locally at the current state and do not propagate its effect through the Bellman backup. The Q-functions for such myopic behavior-regularized algorithms are trained in a similar fashion as standard online actor-critic algorithms, and so we would not expect the Q-values of these algorithms to exhibit similar trends as conservative Q-functions. Our proposed workflow is not designed to handle such methods, though extending it to address them is an interesting direction for future work.          

\begin{table}[t]
  \centering
  \footnotesize
  \def\arraystretch{0.9}
  \setlength{\tabcolsep}{0.42em}
\begin{tabularx}{0.95\linewidth}{c X X}
\toprule
\textbf{Metric/Guideline} & \textbf{CQL (Main paper)} & \textbf{BRAC-v (Appendix~\ref{app:brac_example})}\\
\midrule
Metric~\ref{guideline:overfitting} (Detecting overfitting) & Low average dataset Q-value, $\E_{\bs, \ba \sim \mathcal{D}}[Q_\theta(\bs, \ba)]$, decreasing with more gradient steps & Low average \textcolor{red}{conservative} Q-value, $\E_{\bs, \ba \sim \mathcal{D}}[Q_c(\bs, \ba)]$ on the dataset, that is decreasing with more gradient steps\\
\midrule
Guideline~\ref{guideline:policy checkpoint} (Policy selection) & If overfitting is detected, select the checkpoint with highest average dataset Q-value before overfitting & If overfitting is detected, select the checkpoint with highest average dataset \textcolor{red}{conservative} Q-value before overfitting\\
\midrule
Guideline~\ref{guideline:addressing_overfitting} (Addressing overfitting) & Use some form of capacity-decreasing regularizer on the Q-function, e.g., VIB regularizer, Dropout, etc & Use some form of capacity-decreasing regularizer on \textcolor{red}{both the estimated behavior policy $\hat{\pi}_\beta$ and Q-function $Q_\theta(\bs, \ba)$}, since both combine to form the conservative estimate $Q_c$\\
\bottomrule
    \end{tabularx}
    \vspace{0.1cm}
         \caption{\label{tab:from_cql_to_brac} \footnotesize{Summary of how our proposed overfitting workflow for CQL in the main paper can be adapted to BRAC, with main modifications from CQL to BRAC highlighted in \textcolor{red}{red}. The primary modification is to utilize conservative Q-value estimates, $Q_c(\bs, \ba)$ for BRAC (Equation~\ref{eqn:brac_eqns}), instead of the outputs of the Q-network.
     }}
\vspace{-0.4cm}
\end{table}

\subsection{Empirical Demonstration: Applying Our Overfitting Workflow to BRAC}
\label{app:brac_example}
In this section, we adapt our proposed workflow (Metrics~\ref{guideline:overfitting} and Guidelines \ref{guideline:policy checkpoint} and \ref{guideline:addressing_overfitting}) for detecting and addressing overfitting to the behavior-regularized actor-critic (BRAC) algorithm and empirically verify the efficacy of these metrics and guidelines. Per the discussion above, the main modification needed to apply our workflow from CQL to BRAC is to utilize the conservative Q-value estimate $Q_c(\bs, \ba)$ for BRAC, instead of the Q-values estimated by the Q-network which worked in the case of CQL. Barring this modification, the key principles of our workflow remain the same for BRAC. We detail these below, and present a comparison against our workflow for CQL in Table~\ref{tab:from_cql_to_brac}.

\textbf{Detecting overfitting in BRAC.} Unlike CQL, where the learned Q-values represent a conservative Q-estimate that accounts for both the reward and the divergence from the behavior policy, BRAC estimates these quantities separately as shown in Equation~\ref{eqn:brac_eqns}, with the Q-value not accounting for the divergence against the behavior policy at the current state. Therefore, to apply our workflow guidelines (Metric~\ref{guideline:overfitting}, Guideline~\ref{guideline:policy checkpoint}) to BRAC, we track the ``conservative Q-value estimate’’ discussed in the previous section ($Q_c(s, a) := Q_\theta(s, a) + \beta \log \hat{\pi}_\beta(a|s)$), which is BRAC's analogue of the Q-value learned by CQL. Similar to CQL, overfitting in BRAC-v can be detected via a non-monotonic trend in average dataset conservative Q-value: if the average dataset conservative Q-value first increases and then decreases with more training, this indicates the presence of overfitting. We summarize this in Table~\ref{tab:from_cql_to_brac}, first row.  

\textbf{Policy selection for BRAC.} When overfitting is detected, i.e., the conservative Q-value estimates first increase and then decrease with more gradient steps, we utilize early stopping to find a good policy checkpoint within this run for deployment. Analogously to CQL, our policy checkpoint selection guideline (Table~\ref{tab:from_cql_to_brac}, second row) suggests that a good checkpoint can be found by picking the one that attains the highest average conservative Q-value on the dataset before overfitting begins.

\begin{wrapfigure}{r}{0.64\textwidth}
\vspace{-0.85cm}
\begin{center}
\includegraphics[width=\linewidth]{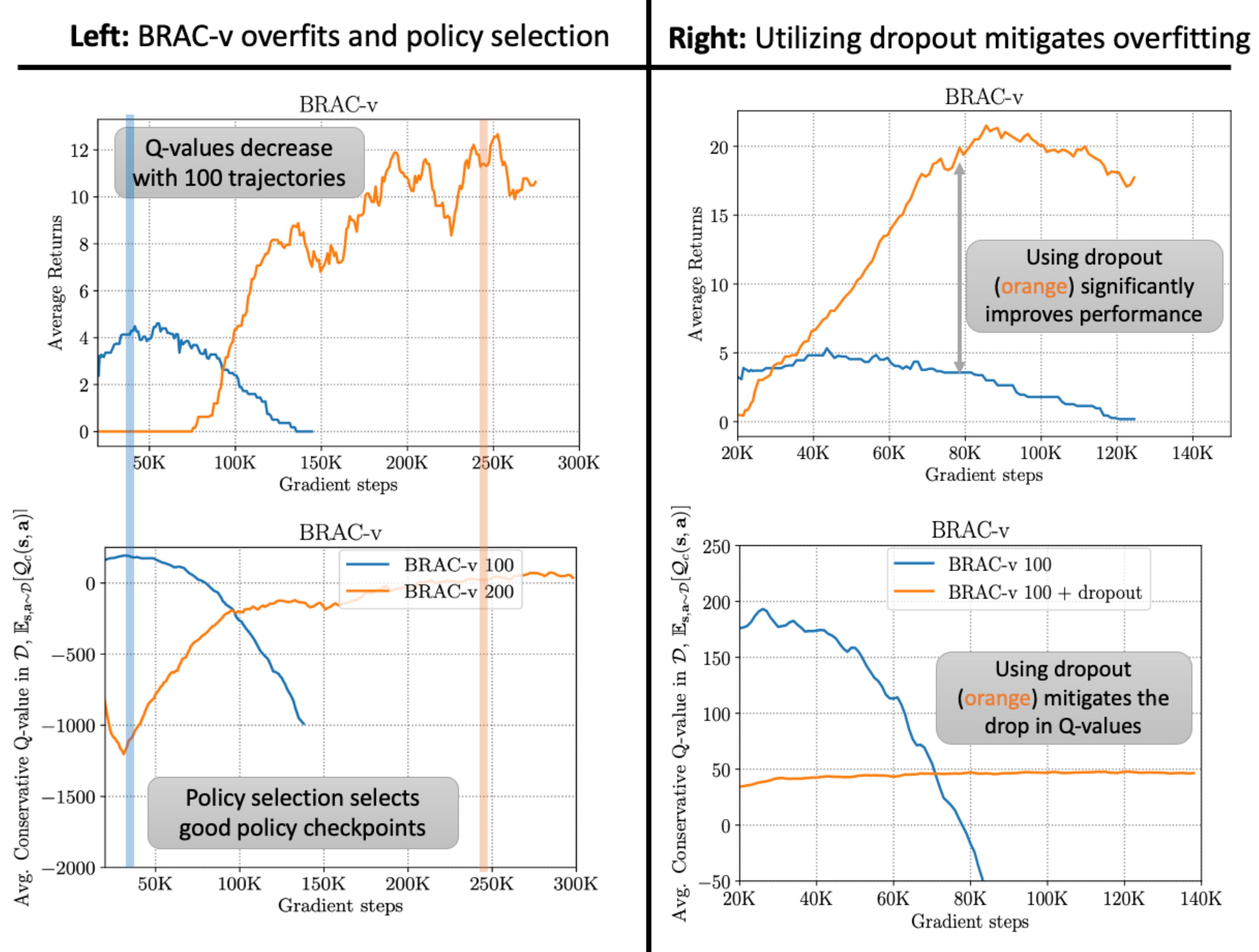}
\vspace{-0.7cm}
\end{center}
\caption{\footnotesize{\label{fig:brac_overfitting_runs} \textbf{Left:} \textbf{Overfiting and policy selection for BRAC-v}: Policy performance \textbf{(top)} and average dataset conservative Q-value \textbf{(bottom)} with varying number of trajectories (100 and 200). The conservative Q-value for the run with 100 trajectories (blue) eventually decreases, while it is relatively stable for 200 trajectories (orange). Vertical bands indicate regions around the peak Q-value and observe that these regions correspond to policies with good actual performance. \textbf{Right: Addressing overfitting in BRAC-v} by using the capacity-decreasing dropout regularizer leads to stable and non-decreasing conservative Q-values and improved policy performance.}}
\vspace{-0.3cm}
\end{wrapfigure}

\textbf{To empirically verify if the adaptation of our workflow is effective for BRAC-v}, we ran experiments on the simulated grasping from drawer task from Scenario \#1, with offline datasets containing 100 and 200 trajectories. Observe in Figure~\ref{fig:brac_overfitting_runs} (left), that with 100 trajectories, the average dataset conservative Q-values $\E_{\bs, \ba \sim \mathcal{D}}[Q_c(\bs, \ba)]$ first increases and then drops with more gradient steps. This observation is consistent with what we expect to happen if the run of BRAC-v overfits per the guideline in Table~\ref{tab:from_cql_to_brac}. In the figure, the vertical dashed lines indicate the policy checkpoints that will be selected by our policy selection guideline (Table~\ref{tab:from_cql_to_brac}). We also visualize the performance of the chosen checkpoints against the actual policy return in the top row for analysis purposes. Note that the selected policy checkpoint indeed attains close to the peak performance achieved over the entire training run of BRAC-v. This indicates the efficacy of our workflow for detecting overfitting and performing policy selection for the BRAC-v algorithm.

\textbf{Addressing overfitting in BRAC-v.} Once overfitting is detected, we need to find an method to alleviate it. As in our workflow for CQL, we can add any capacity-decreasing regularizer such as dropout~\citep{srivastava2014dropout}, variational information bottleneck (VIB), etc to mitigate overfitting. Technically, we want to apply this regularization on the conservative Q-function estimator, $Q_c(\bs, \ba)$, but in the case of BRAC-v, this quantity is not estimated using a single neural network. Thus, we recommend applying the capacity-decreasing regularization to both the critic ($Q_\theta(\bs, \ba)$) and the behavior policy estimate $\hat{\pi}_\beta(\cdot|\bs)$ separately. This guideline is summarized in the third row of Table~\ref{tab:from_cql_to_brac}.

\textbf{To empirically validate our guideline for addressing overfitting in BRAC-v}, we applied the capacity-decreasing dropout regularizer on the run of BRAC on the grasping from drawer task with 100 trajectories. {We chose the dropout regularizer since it worked for CQL (Figure~\ref{fig:other_regularizers}, Appendix~\ref{app:other_overfitting_corrections}), and because it is easier to apply than two separate information bottlenecks on the Q-function and the estimated behavior policy.} As shown in Figure~\ref{fig:brac_overfitting_runs} (right), applying dropout not only alleviates the drop in conservative Q-value estimates after many gradient steps, but it also allows us to pick later checkpoints in training, all of which perform equally well, and much better than the base BRAC-v algorithm. Crucially note that while the policy performance of BRAC-v degrades to zero with more training, utilizing dropout improves the policy performance and increases stability. This validates that overfitting in BRAC-v, as detected via our workflow, can be effectively mitigated by decreasing the capacity of the conservative Q-function in BRAC, in this case by applying dropout to the Q-network and the estimated behavior policy.

\section{Experiments Tuning The CQL $\alpha$ Hyperparameter}
\label{app:alpha_hparam}
In our experiments on both simulated domains and real robots, we utilized a default value of $\alpha = 1.0$ as the multiplier on the CQL term. This directly follows from the choice made in prior work \citep{singh2020cog}, without any modification or tuning. However, to understand the effect of $\alpha$, we now evaluate our workflow on runs with various $\alpha$ values, $\alpha \in \{0.0, 0.01, 0.1, 2, 10, 50\}$, using the two tasks (pick-and-place task and grasping from drawer task) from Scenario \#1 with 50 and 100 trajectories.  Generally, we find that our workflow for detecting overfitting and performing policy selection is reasonably effective for a range of values with $10 \geq \alpha \geq 0.1$, but fails to learn a good policy with very low $\alpha$ values ($\leq 0.01$), for which CQL does not prevent catastrophic overestimation. Similarly our workflow is unable to improve the policy performance in CQL runs with very high $\alpha$ values, which lead to Q-functions that overwhelmingly prioritize giving high value to dataset actions and lead to imitation-like behavior. It is therefore necessary to avoid such extreme $\alpha$ values. In this section, we apply our proposed guidelines for detecting if $\alpha$ is too small or too large (Guidelines~\ref{guideline:decreasing_alpha_main} and \ref{guideline:increasing_alpha_main}) and first adjust $\alpha$. We first discuss $\alpha \geq 0.1$, and then the lower values.

\begin{figure}[t]
\begin{center}
\includegraphics[width=0.95\linewidth]{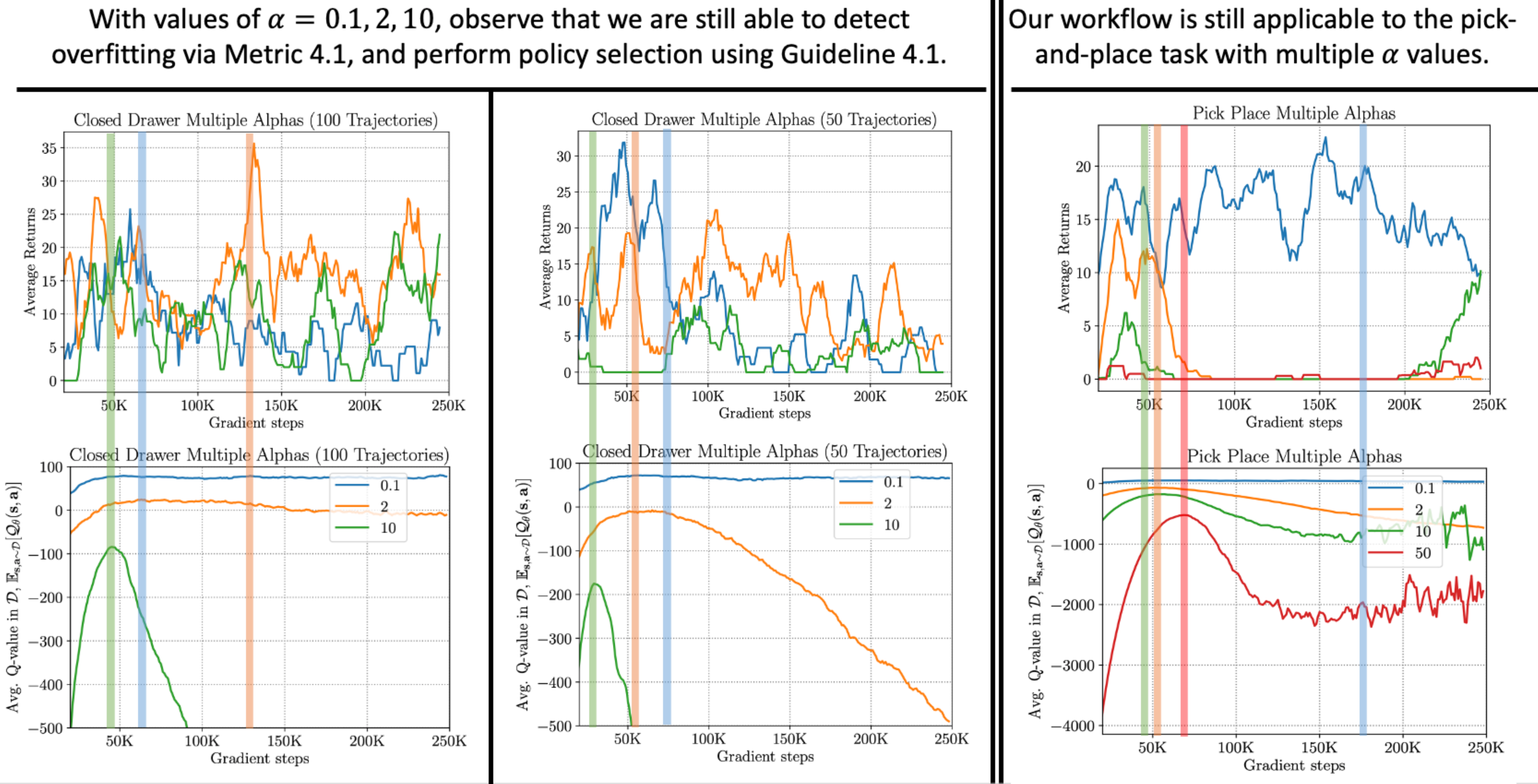}
\vspace{-0.3cm}
\end{center}
\caption{\footnotesize{\label{fig:alpha_values_sweep} \textbf{Evaluating our overfitting workflow with multiple values of the CQL hyperparameter $\alpha \in \{0.1, 2, 10, 50\}$} on three tasks from Scenario \#1: grasping from drawer task with 50 and 100 trajectories and the pick-and-place task with 100 trajectories. Observe that with all of these values, the average dataset Q-value first increases and then decreases, which indicates the presence of overfitting. Also note that policy checkpoints prescribed by our policy selection guideline perform well, especially when compared to other checkpoints within the run. The performance of both CQL and our workflow is generally poor in runs with large $\alpha=10$, because large $\alpha$ values constrain the learned policy to be close to the behavior policy and our workflow does not improve the policy performance in this case. We additionally evaluate $\alpha=50.0$ for the pick-and-place task and observe that the run is overfitting, however, the policy performance is bad for all the checkpoints because $\alpha$ is too large, making CQL similar to behavior cloning.}}
\vspace{-0.3cm}
\end{figure}

\subsection{Values of $\alpha$ That Are Not Too Small} 
\label{app:larger_alpha_values}
We present the trend in average dataset Q-values in Figure~\ref{fig:alpha_values_sweep} for $\alpha \geq 0.1$. Observe that for $\alpha \in \{0.1, 2, 10\}$, the average dataset Q-value first increases and then decreases with more gradient steps, indicating the presence of overfitting per Metric~\ref{guideline:overfitting}. Since overfitting is detected, we can perform policy selection using Guideline~\ref{guideline:policy checkpoint} by choosing the policy checkpoint that appears near the peak in the average dataset Q-value for deployment. For each $\alpha$, these checkpoints are marked with a vertical dashed line. Observe that the selected policy checkpoint indeed performs well compared to all other checkpoints within each training run. This indicates that Metric~\ref{guideline:overfitting} and our policy selection rule in Guideline~\ref{guideline:policy checkpoint} work well across these $\alpha$ values. However, the performance of CQL with large $\alpha$ values is worse compared to smaller $\alpha$ values, likely because CQL finds a policy close to the behavior policy
when the $\alpha$ values are too large (e.g., $\alpha=50$ for the pick-place task in Figure~\ref{fig:alpha_values_sweep}, or $\alpha=10$ for the drawer task with 50 trajectories in Figure~\ref{fig:alpha_values_sweep}). Thus, no matter how we select the policy checkpoint, the performance would be bad, since no checkpoint in the run actually attains good performance. We will shortly discuss how we can detect if $\alpha$ is large and decrease it, but we first present results of applying the VIB overfitting correction to runs with various $\alpha$ values. 

Since overfitting is detected, we would attempt alleviate overfitting by utilizing the VIB regularizer (Equation~\ref{eqn:bottleneck}) following Guideline~\ref{guideline:addressing_overfitting}. As shown in Figure~\ref{fig:alpha_vib_plot},
the VIB regularizer leads to improved policy performance for $\alpha \in \{0.1, 2\}$. However, the VIB regularizer is ineffective with $\alpha=10.0$, where it does not improve performance. 

\begin{figure}[h]
\vspace{-0.4cm}
\begin{center}
\includegraphics[width=0.95\linewidth]{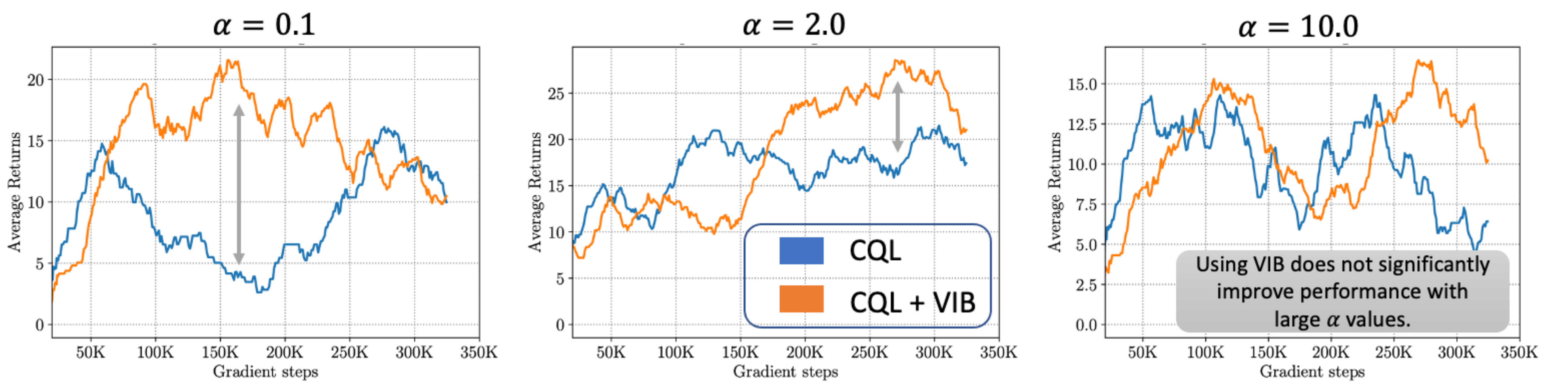}
\vspace{-0.3cm}
\end{center}
\caption{\footnotesize{\label{fig:alpha_vib_plot} \textbf{Utilizing the VIB regularizer from Equation~\ref{eqn:bottleneck} to correct overfitting in CQL runs with $\alpha \in \{0.1, 2, 10\}$} for the grasping from drawer task with 100 trajectories. Applying the VIB regularizer to decrease capacity and correct overfitting improves performance with $\alpha=0.1$ and $\alpha=2.0$, but does not improve performance with a larger value of $\alpha=10.0$.}}
\vspace{-0.3cm}
\end{figure}

The above evidence indicates that our overfitting workflow can fail if the $\alpha$ value is too large ($\alpha \geq 10.0$), but our workflow improves the performance of CQL when $\alpha \in \{0.1, 1.0, 2.0\}$. Hence, for large $\alpha$s we first follow Guidelein~\ref{guideline:decreasing_alpha_main} to decrease $\alpha$ before applying our overfitting workflow. 

To validate the efficacy of Guideline~\ref{guideline:decreasing_alpha_main}, we point the reader to Figure~\ref{fig:alpha_values_sweep}. If we start with $\alpha=10.0$ or $50.0$ on the the drawer task with 50 trajectories or the pick-and-place task, Guideline~\ref{guideline:decreasing_alpha_main} would prescribe reducing $\alpha$, since smaller $\alpha$ values such as $\alpha=0.1, 1.0, 2.0$ also exhibit overfitting per Metric~\ref{guideline:overfitting}. Doing so also does improve the policy performance, especially when starting from $\alpha=50.0$, indicating that this guideline is effective.

\vspace{-0.1cm}
\subsection{Small values of $\alpha$} 
\label{app:small_alphas}
\vspace{-0.2cm}
Next, we evaluate our workflow with the two smallest values of $\alpha=0.0, 0.01$. In both cases, as shown in Figure~\ref{fig:small_alpha_plot}, we find that the value of the CQL regularizer is large (close to 0, which means out-of-distribution actions have similar values as in-distribution actions), and average dataset Q-value does not decrease with more gradient steps. As expected, the corresponding policy performs poorly in each case, since a high CQL regularizer value implies that the out-of-distribution actions have a higher Q-value than in-distribution actions, which in turn means that policy optimization will select out-of-distribution actions. In this case, our underfitting workflow will not improve the performance of CQL, and the value of $\alpha$ would need to be raised. We thus follow Guideline~\ref{guideline:increasing_alpha_main} first, to tune $\alpha$ before applying the rest of our workflow. 

\begin{figure}[h]
\vspace{-0.25cm}
\begin{center}
\includegraphics[width=0.95\linewidth]{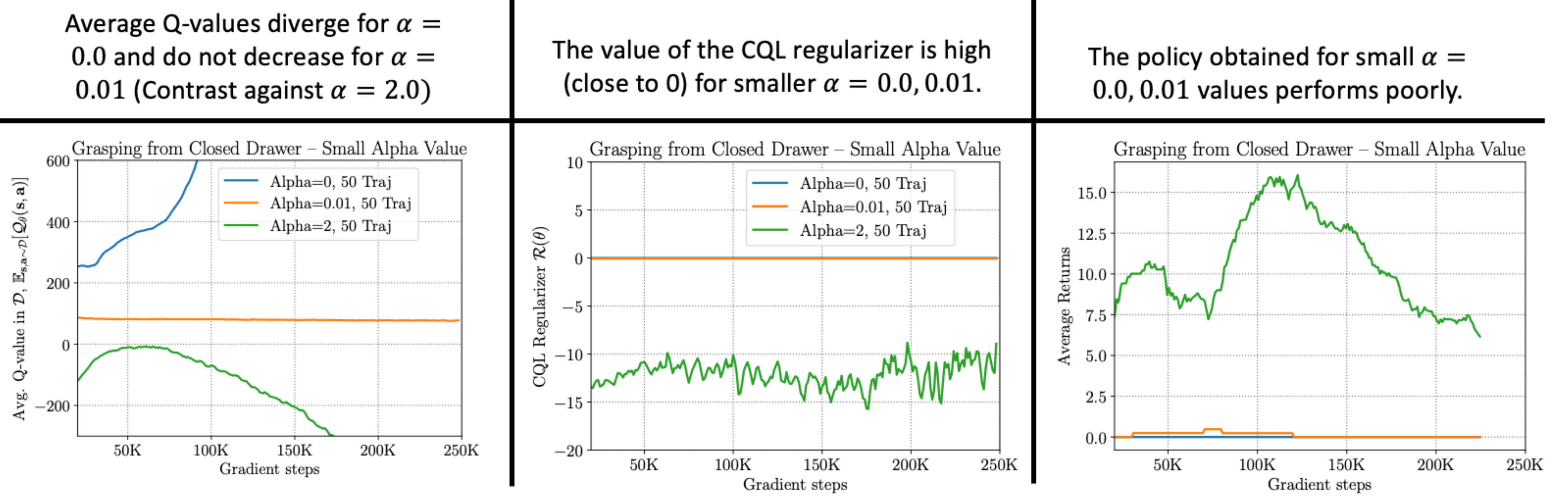}
\vspace{-0.3cm}
\end{center}
\caption{\footnotesize{\label{fig:small_alpha_plot} \textbf{CQL fails to prevent erroneous overestimation in the Q-function with small $\alpha$ values and hence performs poorly.} For $\alpha=0.0$, the Q-function positively diverges. For $\alpha=0.01$, the average Q-value is stable and does not decrease with more gradient steps. Note the vastly different trend in the average Q-value for $\alpha=2.0$ for contrast. Additionally, observe that the value of the CQL regularizer is close to $0$ for $\alpha=0.0$ and $\alpha=0.01$, which means Q-values for out-of-distribution actions are high compared to in-distribution actions, for contrast see the much smaller value of the CQL regularizer with $\alpha=2.0$.}}
\vspace{-0.3cm}
\end{figure}
%
%
To empirically demonstrate the efficacy of Guideline~\ref{guideline:increasing_alpha_main}, we attempt to correct the apparent underfitting in the run of CQL with $\alpha=0.01$ by rerunning it with increased model-capacity and present the results in Figure~\ref{fig:underfitting_attempt}. Observe that the value of the CQL regularizer is still close to $0$, identical to the the na\"ive CQL run without the underfitting correction. Since underfitting correction does not reduce the value of the CQL regularizer, according to Guideline~\ref{guideline:increasing_alpha_main}, we need to increase $\alpha$ to allow for better minimization of the CQL regularizer. As we have already seen in the earlier runs in this section in Figure~\ref{fig:alpha_values_sweep}, if we increase $\alpha$ to $0.1$ or $1.0$, the value of the CQL loss would be small and sufficiently negative attaining values around $-5.0$, and overfitting is detected. Our policy selection guideline would then allow us to find a good policy for deployment. 

\begin{figure}[h]
\vspace{-0.4cm}
\begin{center}
\includegraphics[width=0.95\linewidth]{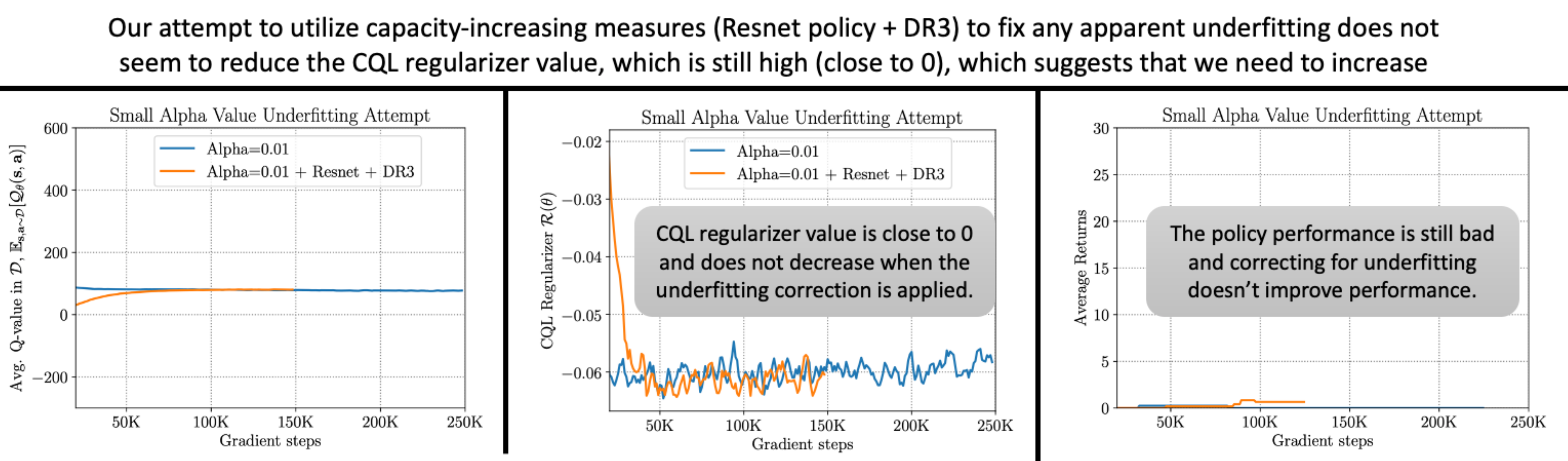}
\vspace{-0.3cm}
\end{center}
\caption{\footnotesize{\label{fig:underfitting_attempt} \textbf{Validating Guideline~\ref{guideline:increasing_alpha_main} by attempting to fix underfitting in CQL with a small $\alpha = 0.01$.} To verify if the run of CQL with $\alpha=0.01$ is underfitting or if it requires increasing $\alpha$, we re-run it with a capacity-increasing measures. However, even in this case, the value of the CQL regularizer is large. The value of the CQL regularizer is close to 0, which means the values of out-of-distribution actions is not small enough compared to in-distribution actions. Since underfitting correction does not help in this case, we conclude that is the case where the value of $\alpha$ needs to be raised to obtain improved performance.}}
\vspace{-0.3cm}
\end{figure}

\textbf{To summarize,} while our workflow for detecting overfitting, performing policy selection, and correcting overfitting works well across several $\alpha$ values, CQL can fail when $\alpha$ is too small, and will reduce to behavior cloning when $\alpha$ is too large. This is expected because smaller $\alpha$ values are insufficient to prevent overestimation and will cause the policy to choose unseen out-of-distribution actions and extremely large $\alpha$ values will strongly update the policy towards the behavior policy. To detect and handle if $\alpha$ is too small or too large, we proposed Guidelines \ref{guideline:decreasing_alpha_main} and \ref{guideline:increasing_alpha_main}, which prescribe \textbf{increasing} $\alpha$ if \textbf{(a)} the value of the CQL regularizer is large, and \textbf{(b)} utilizing capacity-increasing measures does not lead to reduction in the CQL regularizer value and \textbf{decreasing} $\alpha$ if \textbf{(a)} overfitting is observed with the current run, and \textbf{(b)} re-running CQL with a smaller value of $\alpha$ also exhibits an overfitting trend, with average Q-value decreasing with more gradient steps. After modifying the value of $\alpha$, we prescribe following the recommendations of the rest of our workflow.

\section{Underfitting With 10 and 20 Objects in Scenario \#2}
\label{app:other_num_objects}
In this section, we present our results on applying the proposed capacity-increasing measures on the simulated experiments with multiple training objects (10 and 20 objects) from Scenario \#2. The plot for 35 objects is shown in the main paper. In the case of 10 and 20 objects, we also observed a high value of TD error (see Figure~\ref{fig:object_diversity}), with relatively stable Q-values. In this case, our workflow would prescribe correcting for underfitting. In Figure~\ref{fig:dr3_resnet_appendix}, we present results of running CQL with the capacity-increasing DR3 regularizer and a ResNet policy to address underfitting in the case of 10 and 20 objects. We find that in both cases the performance of the policy improves and our capacity-increasing measures also generally lead to a slight decrease in the value of the TD error.

\begin{figure}[t]
\begin{center}
\includegraphics[width=0.6\linewidth]{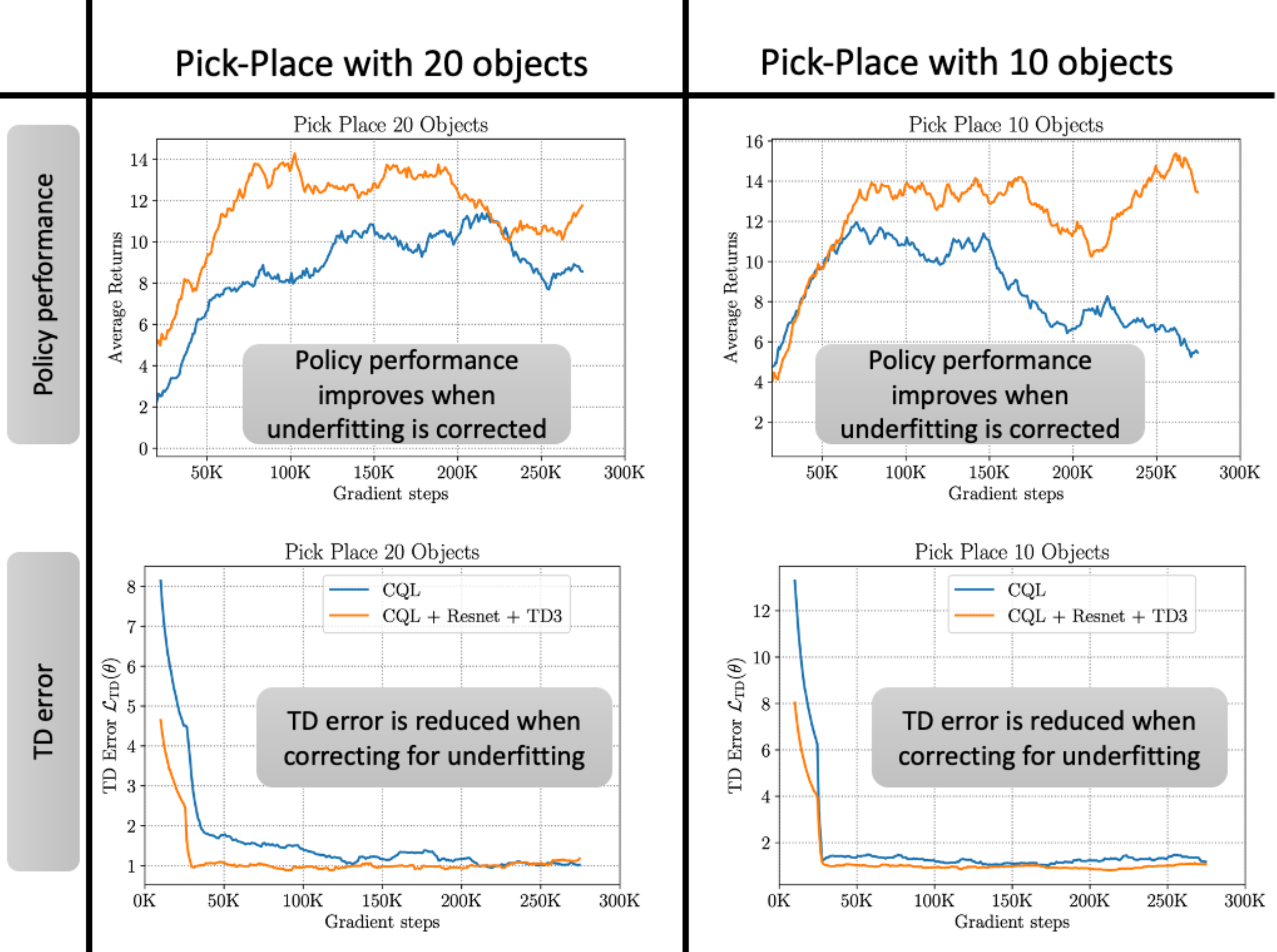}
\vspace{-0.2cm}
\end{center}
\caption{\footnotesize{\label{fig:dr3_resnet_appendix} \textbf{Correcting underfitting by utilizing a ResNet policy + DR3 regularizer for the case of 10 and 20 objects from Scenario \# 2.} Note that the addition of these underfitting corrections improves policy performance, while also reducing the training TD error by some amount.}}
\vspace{-0.6cm}
\end{figure}

\section{Other Capacity-Decreasing Regularizers for Addressing Overfitting}
\label{app:other_overfitting_corrections}

In this section, we present a study that evaluates different choices of capacity-decreasing regularizers to prevent overfitting in CQL. The candidate capacity-decreasing regularizers we evaluate are: 
\begin{itemize}
    \item dropout~\citep{srivastava2014dropout} with masking probability $p$ on the layers of the Q-function $Q_\theta$, 
    \item $\ell_1$ regularization on the parameters $\theta$ of the Q-function (i.e., $\min_\theta \mathcal{L}_\text{CQL}(\theta) + \rho ||\theta||_1$), and 
    \item $\ell_2$ regularization on the parameters $\theta$ (i.e., $\min_\theta \mathcal{L}_\text{CQL}(\theta) + \rho ||\theta||^2_2$).
\end{itemize}
We apply each of these regularizers to the run of CQL on the pick-and-place task from Scenario \#1, with 100 trajectories and report the average dataset Q-value, the corresponding performance of the policy (for analysis purposes) and the value of the training CQL regularizer for each of dropout, $\ell_1$ and $\ell_2$ regularization schemes in Figure~\ref{fig:other_regularizers}. To find a good value of $\rho$ and $p$ completely offline, we run each regularizer with different values of the hyperparameter $\rho \in \{0.1, 0.01, 1.0\}$ (for $\ell_1$/$\ell_2$ regularization) and $p \in \{0.01, 0.1, 0.2, 0.4\}$ (for dropout) and pick the value that stabilizes the trend in the average dataset Q-value, while not inhibiting the minimization of the training CQL regularizer. That is, we require the value of CQL regularizer to be sufficiently negative (ideally $\leq -2$ or $-3$). This is essential since excessive capacity-decreasing regularization can inhibit the minimization of the training CQL objective which would cause the policy to execute bad out-of-distribution actions. Using the scheme described above, we obtained $p=0.2$ for dropout and $\rho=0.01$ for the case of $\ell_2$ regularization.  

Observe in Figure~\ref{fig:other_regularizers} that utilizing dropout (left column) or applying $\ell_2$ regularization (middle column) mitigates the drop in average Q-value that is observed with na\"ive, untuned CQL on this task while also achieving a small CQL regularizer value. Applying dropout and $\ell_2$ regularization leads to improved and much more stable policy performance. This indicates that addressing overfitting by applying capacity-decreasing regularization can lead to improved performance.

We observed that $\ell_1$ regularization did not give rise to improved performance. Out of all three values of $\rho$, all of which are  presented in Figure~\ref{fig:other_regularizers} (right column) we found that $\rho=0.01$ was likely not large enough to mitigate the drop in Q-value, and runs with larger values of $\rho=0.1, 1.0$ failed to decrease the training CQL regularizer. We believe that an intermediate value of $\rho \in [0.01, 0.1]$ can possibly alleviate the overfitting issue, and we will run a finer search over $\rho$ for the final version.   

\begin{figure}
\begin{center}
\includegraphics[width=\linewidth]{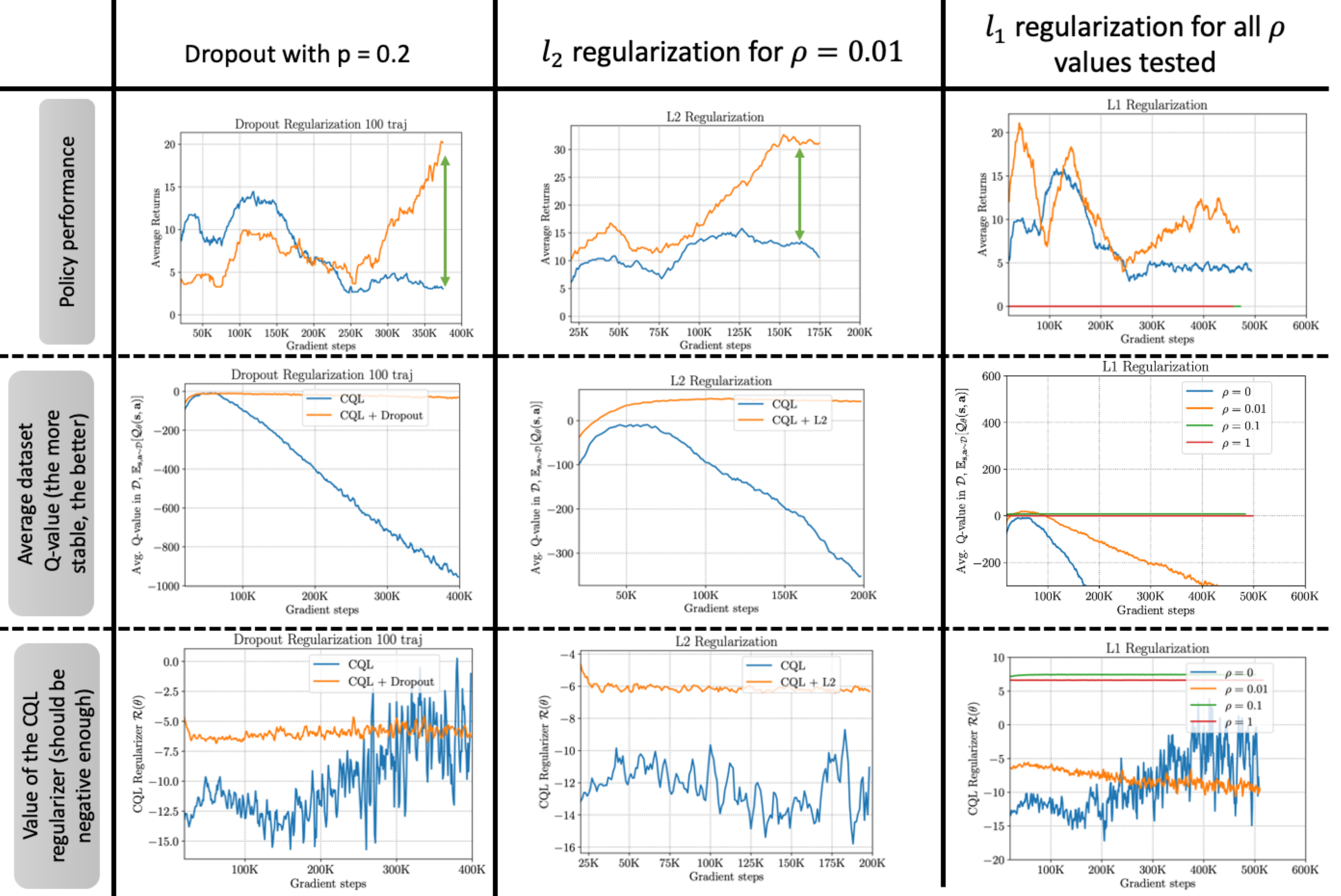}
\vspace{-0.2cm}
\end{center}
\caption{\footnotesize{\label{fig:other_regularizers} \textbf{CQL + different capacity-decreasing regularizers: dropout (left), $\ell_2$ regularization (middle) and $\ell_1$ regularization (right).} Comparison of different capacity-decreasing regularization schemes for the run of CQL on the drawer task with 100 trajectories from Scenario \#1. While na\"ive CQL (shown in blue in the plots) exhibits overfitting, i.e., the average dataset Q-value first increases and then decreases with more gradient steps, the addition of $\ell_2$ regularization or dropout with $\rho=0.01$ and $p=0.2$ respectively alleviates the drop in the Q-value (middle row shows the time series of the average dataset Q-value). Additionally observe that $\ell_2$ regularization and dropout improve policy performance, especially $\ell_2$ regularization. For $\ell_1$ regularization, none of the $\rho$ values we searched over was able to mitigate the drop in the average Q-value while retaining a negative value of the CQL regularizer, and thus did not improve in performance.}}
\vspace{-0.6cm}
\end{figure}

\section{Alternative Metrics for Overfitting}
\label{app:other_metrics}
In addition to Metric~\ref{guideline:overfitting} that prescribes tracking the average dataset Q-value for detecting overfitting and performing policy selection (Guideline~\ref{guideline:policy checkpoint}), we can also, in principle, choose to use an estimate of the policy return estimated using the learned conservative Q-function. Formally, this metric is given by policy value averaged under the initial state distribution $\mu_0(\bs)$: $\E_{\bs \sim \mu_0, \ba \sim \pi(\cdot|\bs)}[Q_\theta(\bs, \ba)]$. We perform a preliminary experimental study comparing metric $\E_{\bs, \ba \sim \mathcal{D}}[Q_\theta(\bs, \ba)]$ (Metric~\ref{guideline:overfitting}) and the policy return at the initial state on the drawer and the pick-place tasks from Scenario \#1, with 50 trajectories. As shown in Figure~\ref{fig:both_metrics}, we find that both of these metrics closely follow each other for most of the training iterations, and applying the policy selection guideline on either of them will choose the same policy checkpoint since these curves heavily overlap near the peak.   

\begin{figure}
\begin{center}
\includegraphics[width=0.7\linewidth]{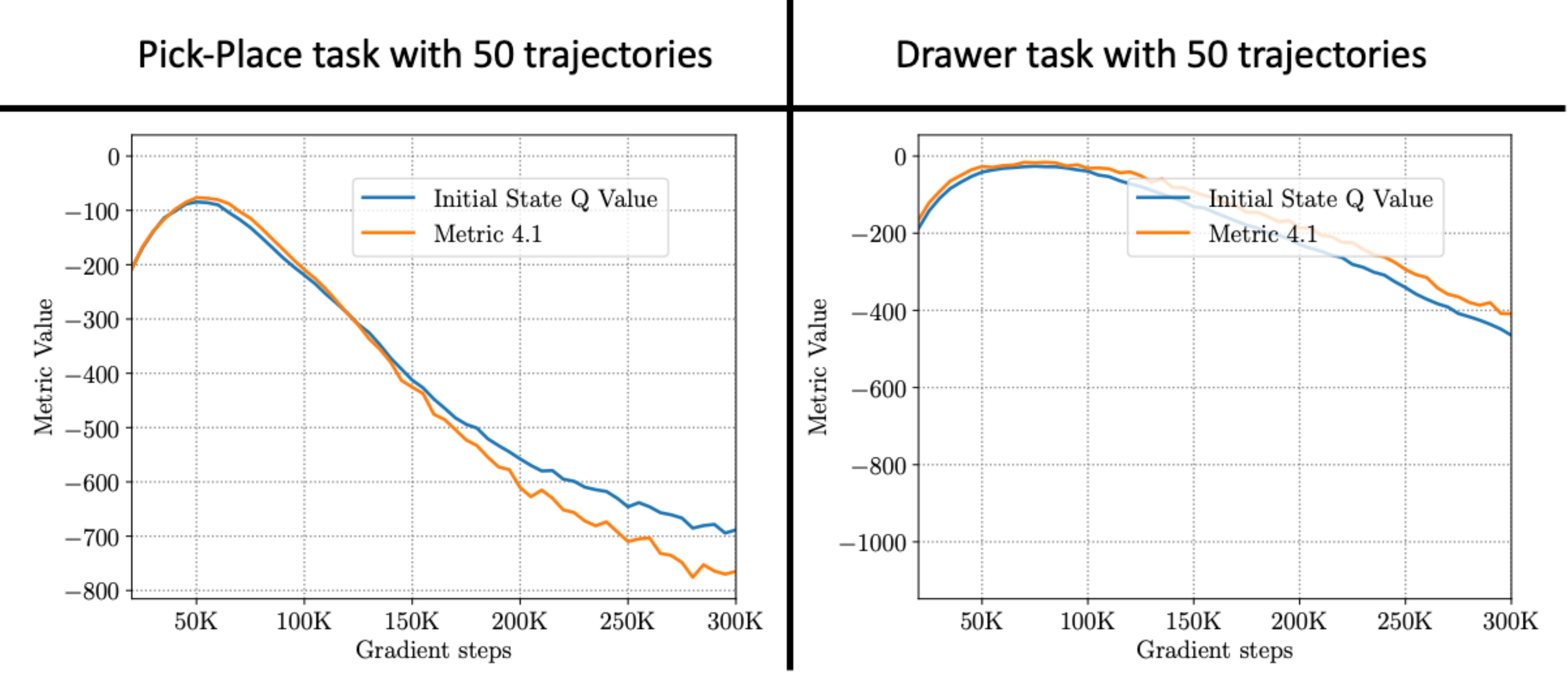}
\vspace{-0.2cm}
\end{center}
\caption{\footnotesize{\label{fig:both_metrics} \textbf{Preliminary experiments comparing the evolution of the average dataset Q-value in Metric~\ref{guideline:overfitting} (orange) and policy value averaged under the initial state distribution (blue).} Observe that both of these metrics follow each other closely for the most part in training, and exhibit a similar behavior, where the metric first increases and then decreases with more training. The peak in both of the curves overlap, indicating that utilizing either of the metrics for policy selection will return the same policy checkpoint.}}
\vspace{-0.6cm}
\end{figure}

}